\documentclass[11pt, letterpaper]{article}
\usepackage[utf8]{inputenc}

\usepackage[a4paper,top=2.54cm,bottom=2.54cm,left=2.54cm,right=2.54cm,marginparwidth=1.75cm]{geometry}

\usepackage[utf8]{inputenc} %
\usepackage[T1]{fontenc}    %
\usepackage{hyperref}       %
\usepackage{url}            %
\usepackage{booktabs}       %
\usepackage{amsfonts}       %
\usepackage{nicefrac}       %
\usepackage{microtype}      %

\usepackage[parfill]{parskip}
\usepackage[square,numbers]{natbib}%
\usepackage{xr-hyper,refcount}
\usepackage{graphicx}
\usepackage[utf8]{inputenc} %
\usepackage{booktabs}       %
\usepackage{amsfonts}       %
\usepackage{nicefrac}       %
\usepackage{microtype}      %
\usepackage[parfill]{parskip}
\usepackage[ruled,vlined, linesnumbered]{algorithm2e}
\usepackage{algorithmic}
\usepackage{pifont}
\usepackage{amsmath,amsthm,amssymb,bbm}
\usepackage{mathtools}
\usepackage{comment}
\usepackage{caption}
\usepackage{subcaption}
\usepackage[page,header]{appendix}
\usepackage{xspace}
\usepackage{enumitem}
\usepackage[english]{babel}%
\usepackage{authblk} %
\usepackage{mathtools}
\usepackage{amsthm}

\usepackage{booktabs}
\usepackage{leqisymbols}
\usepackage{tikz}
\usepackage{bm}
\usepackage{setspace}
\usepackage{mathtools}

\usepackage{selectp}

\usepackage{scrwfile}
\TOCclone[\contentsname~(\appendixname)]{toc}{atoc}
\newcommand\StartAppendixEntries{}
\AfterTOCHead[toc]{%
  \renewcommand\StartAppendixEntries{\value{tocdepth}=-10000\relax}%
}
\AfterTOCHead[atoc]{%
  \edef\maintocdepth{\the\value{tocdepth}}%
  \value{tocdepth}=-10000\relax%
  \renewcommand\StartAppendixEntries{\value{tocdepth}=\maintocdepth\relax}%
}
\newcommand*\appendixwithtoc{%
  \cleardoublepage
  \appendix
  \addtocontents{toc}{\protect\StartAppendixEntries}
  \listofatoc
}

\title{Rebounding Bandits for Modeling Satiation Effects}%

\author[1]{Liu Leqi\footnote{Corresponding author: \texttt{leqil@cs.cmu.edu}.}}
\author[2]{Fatma K{\i}l{\i}n\c{c}-Karzan}
\author[1,2]{Zachary C. Lipton}
\author[1,2]{Alan L. Montgomery}
\affil[1]{Machine Learning Department}
\affil[2]{Tepper School of Business}
\affil[ ]{Carnegie Mellon University}
\date{October 27, 2021}

\begin{document}

\maketitle

\begin{abstract} 
Psychological research 
shows that enjoyment of many goods 
is subject to satiation,  
with short-term satisfaction declining 
after repeated exposures to the same item. 
Nevertheless, proposed algorithms 
for powering recommender systems
seldom model these dynamics,
instead proceeding as though 
user preferences were fixed in time. 
In this work, we introduce \emph{rebounding bandits},
a multi-armed bandit setup,
where satiation dynamics are modeled
as time-invariant linear dynamical systems.
Expected rewards for each arm 
decline monotonically 
with consecutive exposures to it
and rebound towards the initial reward 
whenever that arm is not pulled.
Unlike classical bandit settings,
methods for tackling rebounding bandits
must plan ahead and model-based methods 
rely on estimating the parameters of the
satiation dynamics. 
We characterize the planning problem, 
showing that the greedy policy
is optimal 
when the arms exhibit identical deterministic dynamics.
To address stochastic satiation dynamics with unknown parameters, we propose Explore-Estimate-Plan (EEP),
an algorithm that pulls arms methodically,
estimates the system dynamics, 
and then plans accordingly.

\end{abstract}

\section{Introduction}
Recommender systems suggest
such diverse items as music, 
news, restaurants, 
and even job candidates.  
Practitioners hope that
by leveraging historical interactions,
they might provide services
better aligned with their users' preferences.
However, despite their ubiquity in application,
the dominant learning framework 
suffers several conceptual gaps
that can result in misalignment
between machine behavior and human preferences. 
For example, because human preferences
are seldom directly observed,
these systems are typically trained 
on the available observational data
(e.g., purchases, ratings, or clicks)
with the objective of predicting customer behavior~\citep{bennett2007netflix, mcauley2013amateurs}.
Problematically, such observations 
tend to be confounded 
(reflecting exposure bias 
due to the current recommender system)
and subject to censoring 
(e.g., users with strong opinions 
are more likely to write reviews)~\citep{swaminathan2015counterfactual, joachims2017unbiased}.

Even if we could directly observe 
the utility experienced by each user,
we might expect it to depend, in part,
on the history of past items consumed. 
For example, consider the task 
of automated (music) playlisting.
As a user is made to listen to the same song
over and over again, 
we might expect that the utility 
derived from each consecutive listen
would decline~\citep{ratner1999choosing}. 
However, after listening 
to other music for some time,
we might expect the utility 
associated with that song
to bounce back towards its baseline level.
Similarly, a diner served pizza for lunch
might feel diminished pleasure 
upon eating pizza again for dinner.

The psychology literature on \emph{satiation} 
formalizes the idea that enjoyment depends 
not only on one's intrinsic preference 
for a given product
but also on the sequence of previous exposures 
and the time between them~\citep{baucells2007satiation, caro2012product}. 
Research on satiation dates to the 1960s 
(if not earlier)
with early studies addressing brand loyalty~\citep{tucker1964development, mcconnell1968development}.  
Interestingly, even after controlling
for marketing variables like price, 
product design, promotion, etc., 
researchers still observe 
brand-switching behavior in consumers. 
Such behavior, referred as \emph{variety seeking}, 
has often been explained as a consequence
of utility associated with the change itself
\citep{mcalister1982dynamic, kahn1995consumer}. 
For a comprehensive review on hedonic decline 
caused by repeated exposure to a stimulus,
we refer the readers to \cite{galak2018properties}.

In this paper, we introduce \emph{rebounding bandits}, 
a multi-armed bandits (MABs)~\citep{robbins1952some} framework
that models satiation via linear dynamical systems.
While traditional MABs draw rewards 
from \emph{fixed} but unknown distributions,
rebounding bandits allow 
each arm's rewards to evolve
as a function of both 
the per-arm characteristics
(susceptibility to satiation 
and speed of rebounding)
and the historical pulls 
(e.g., past recommendations).
In rebounding bandits, 
even if the dynamics 
are known and deterministic,
selecting the optimal
sequence of $T$ arms to play
requires planning in a Markov decision process (MDP)
whose state space scales exponentially in the horizon $T$.
When the satiation dynamics 
are known and stochastic, 
the states are only partially observable, 
since the satiation of each arm evolves with (unobserved) stochastic noises between pulls. %
And when the satiation dynamics are unknown,
learning requires that we identify
a stochastic dynamical system.

We propose Explore-Estimate-Plan (EEP)
an algorithm that 
(i) collects data by pulling each arm repeatedly,
(ii) estimates the dynamics using this dataset;
and (iii) plans using the estimated parameters.
We provide guarantees for our estimators
in \S~\ref{sec:system-identification-main-text} 
and bound EEP's regret in \S~\ref{sec:regret-bound}. 

Our main contributions are:
(i) the rebounding bandits problem
(\S\ref{sec:setup}),
(ii) 
analysis showing
that when arms share rewards and (deterministic) dynamics,
the optimal policy pulls arms cyclically,
exhibiting variety-seeking behavior (\S\ref{sec:greedy}); 
(iii) an estimator (for learning the satiation dynamics)  
along with a sample complexity bound  
for identifying an affine dynamical system 
using a single trajectory of data
(\S\ref{sec:system-identification-main-text});
(iv)
EEP, an algorithm
for learning with unknown stochastic dynamics
that achieves sublinear 
$w$-step lookahead regret~\cite{pike2019recovering}
(\S\ref{sec:etc});
and (v) experiments demonstrating EEP's efficacy
(\S\ref{sec:experiment}).

\section{Related Work} 
\label{sec:related-work}

Satiation effects have been addressed
by such diverse disciplines as
psychology, marketing, 
operations research, 
and recommendation systems.
In the psychology and marketing literatures,
satiation 
has been proposed as an explanation for
variety-seeking consumer behavior~\cite{galak2018properties,mcalister1982dynamic,mcalister1982variety}.
In operations research,
addressing continuous consumption decisions, 
\cite{baucells2007satiation} propose a deterministic  
linear dynamical system to model satiation effects. 
In the recommendation systems community, 
researchers have used semi-Markov models
to explicitly model two states: 
(i) \emph{sensitization}---where the user 
is highly interested in the product;
and  (ii) \emph{boredom}---where the user 
is not engaged~\citep{kapoor2015just}. 

The bandits literature 
has proposed a variety 
of extensions
where rewards depend on past exposures,
both to address satiation and other phenomena.
\cite{heidari2016tight,levine2017rotting, seznec2018rotting} 
tackle settings
where each arm's expected reward 
grows (or shrinks) monotonically
in the number of pulls.
By contrast,
\cite{kleinberg2018recharging, basu2019blocking, cella2020stochastic} 
propose models
where rewards increase
as a
function of the time elapsed
since the last pull.
\cite{pike2019recovering} 
model the expected reward
as a function of the time
since the last pull
drawn from a Gaussian Process 
with known kernel. 
\cite{warlop2018fighting} 
propose a model where rewards 
are linear functions
of the recent history of actions
and \cite{mintz2020nonstationary}
model the reward as a function of a context 
that evolves according to known deterministic dynamics.
In rested bandits~\citep{gittins1979bandit},
an arm's rewards changes only when it is played,  
and in restless bandits \citep{whittle1988restless}
rewards evolve independently from the play of each arm. 

\paragraph{Key Differences}
This may be the first bandits paper
to model evolving rewards 
through continuous-state
linear stochastic dynamical systems
with unknown parameters. 
Our framework %
captures several important %
aspects of satiation:
rewards decline by diminishing amounts 
with consecutive pulls 
and rebound towards the baseline with disuse. 
Unlike models that depend 
only on fixed windows %
or the time since the last pull, 
our model expresses satiation more organically
as a quantity that evolves 
according to stochastic dynamics
and is shocked (upward) by pulls. 
To estimate the reward dynamics,
we leverage recent advances 
in the identification 
of linear dynamical systems~\citep{simchowitz2018learning,sarkar2019near}
that rely on the theory of self-normalized processes~\citep{pena2008self,abbasi2011improved}
and block martingale conditions~\citep{simchowitz2018learning}.

\begin{figure*}
     \centering
     \begin{subfigure}[b]{0.45\linewidth}
         \centering
         \includegraphics[width=\linewidth]{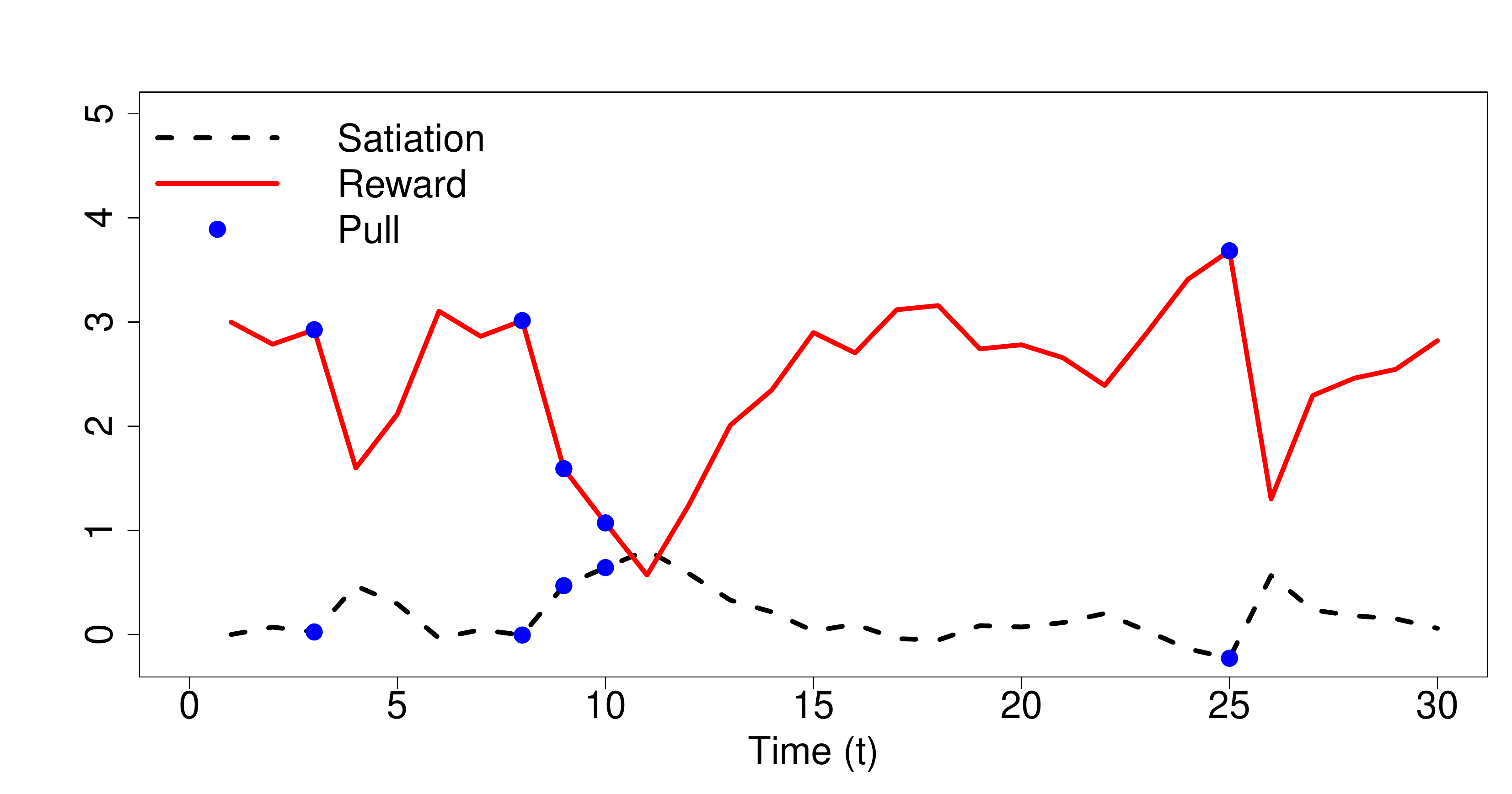}
         \caption{$\gamma_k = .5, \lambda_k = 3, b_k = 3$}
         \label{fig:gamma_p5_sigma_p1}
     \end{subfigure}
     \hfill
     \begin{subfigure}[b]{0.45\linewidth}
         \centering
         \includegraphics[width=\linewidth]{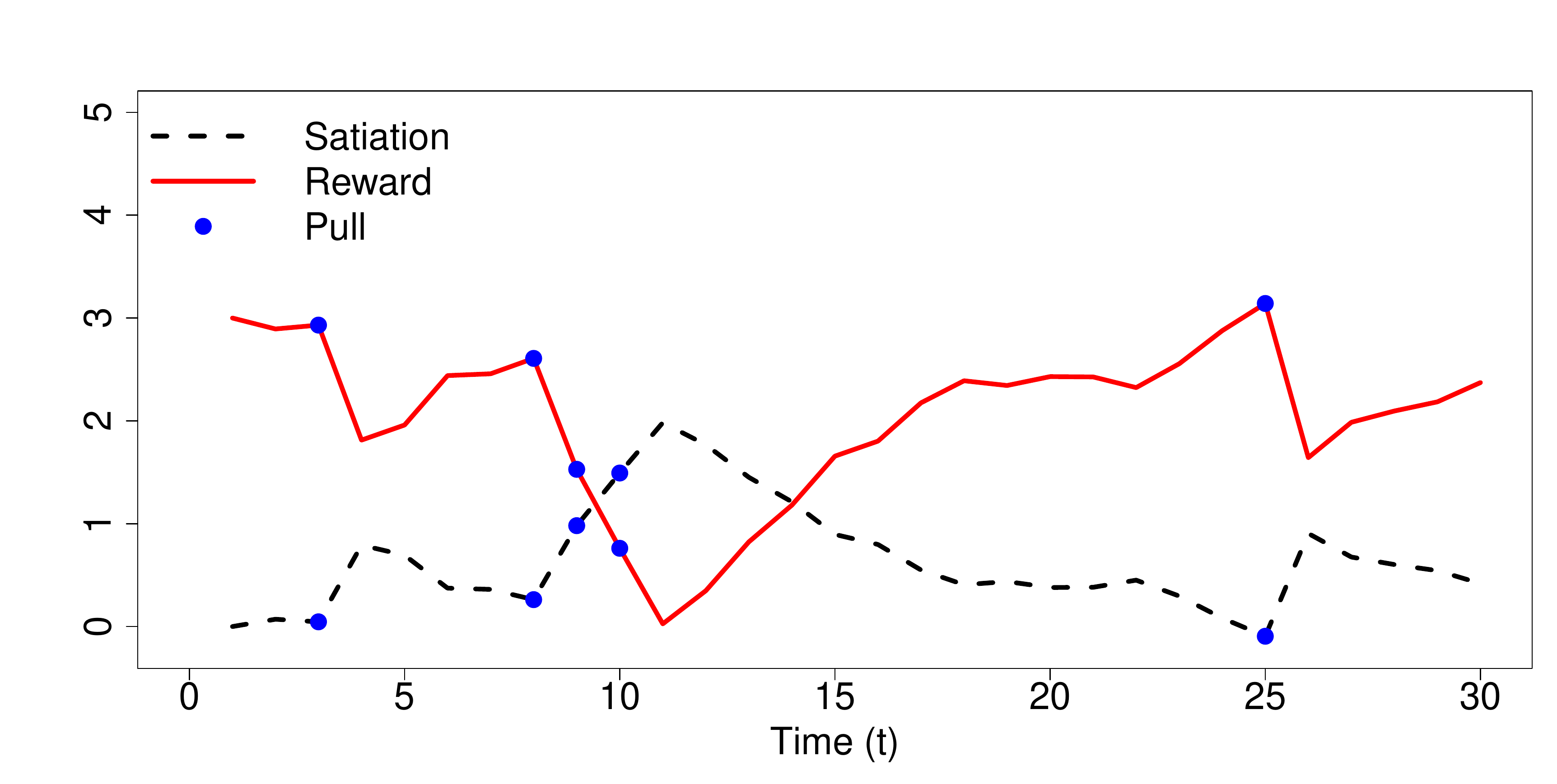}
         \caption{$\gamma_k = .8, \lambda_k = 1.5, b_k = 3$}
         \label{fig:gamma_p8_sigma_p1}
     \end{subfigure}
        \caption{
        These plots illustrate the satiation level and reward 
        of an arm from time $1$ to $30$. 
        The two plots are generated 
        with the same pull sequence,
        base rewards $b_k = 3$ 
        and realized noises with variance $\sigma_z = .1$. 
        In Figure~\ref{fig:gamma_p5_sigma_p1}, $\gamma_k = .5$
        and $\lambda_k = 3$.
        In Figure~\ref{fig:gamma_p8_sigma_p1}, 
        $\gamma_k = .8$
        and $\lambda_k = 1.5$.
        In both cases, the arm has started with $0$
        as its base satiation level. 
        \textbf{Black dashed line:} 
        the satiation level. 
        \textcolor{red}{\textbf{Red solid line:}} the reward.
        \textcolor{blue}{\textbf{Blue dots:}} time steps 
        where the arm is pulled.}
        \label{fig:dynamics_illustration}
        \vspace{-5px}
\end{figure*}

\section{Rebounding Bandits Problem Setup}
\label{sec:setup}

Consider the set of $K$ arms 
$[K] := \{1, \ldots, K\}$ 
with bounded base rewards $b_1, \ldots, b_K$. 
Given a horizon $T$, a policy 
$\pi_{1:T} \coloneqq (\pi_1, \ldots, \pi_T)$ 
is a sequence of actions,
where $\pi_t \in [K]$ depends 
on past actions and observed rewards. 
For any arm $k \in [K]$, 
we denote its pull history from $0$ to $T$ 
as the binary sequence 
$u_{k,0:T} \coloneqq (u_{k,0}, \ldots, u_{k,T})$,
where $u_{k,0} = 0$ and for $t \in [T]$, 
$u_{k,t} = 1$ if $\pi_t = k$ and $u_{k,t} = 0$ otherwise. 
The subsequence of $u_{k,0:T}$ 
from $t_1$ to $t_2$ (including both endpoints) is denoted 
by $u_{k,t_1:t_2}$.

At time $t$, each arm $k$
has a satiation level $s_{k,t}$ 
that depends on a \emph{satiation retention} factor 
$\gamma_k \in [0, 1)$, as follows
\begin{align}\label{eq:satiation-dynamics}
    s_{k,t} := \gamma_k (s_{k,t-1} + u_{k,t-1}) + z_{k,t-1}, 
    \; \forall t > t_0^k,
\end{align}
where $t_0^k := \min_t \{t:~u_{k,t} = 1\}$ %
is the first time arm $k$ is pulled and 
$z_{k,t-1}$ is independent and identically distributed noise 
drawn from $\mathcal{N}(0, \sigma_z^2)$,
accounting for incidental (uncorrelated) 
factors in the satiation dynamics.
Because satiation requires exposure,
arms only begin to have nonzero 
satiation levels after their first pull,
i.e., 
$s_{k, 0} = \ldots = s_{k,t_0^k} = 0$.

At time $t \in [T]$, 
if arm $k$ is played 
with a current satiation level $s_{k,t}$, 
the agent receives reward
$\mu_{k,t} := b_k - {\lambda_k s_{k,t}}$,
where $b_k$ is the base reward for arm $k$ 
and $\lambda_k \geq 0$ is a bounded 
\emph{exposure influence} factor. 
We use \emph{satiation influence} to 
denote the product of 
the exposure influence factor $\lambda_k$ %
and the satiation level $s_{k,t}$. %
In Figure~\ref{fig:dynamics_illustration},
we show how rewards evolve 
in response to both pulls
and the stochastic dynamics
under two sets of parameters. 
The expected reward %
of arm $k$ 
(where the expectation 
is taken over all noises 
associated with the arm)  
monotonically decreases 
by diminishing amounts 
with consecutive pulls 
and increases with disuse 
by diminishing amounts. 

\begin{remark}[negative expected reward]
We note that there exist choices of $b_k, \gamma_k, \lambda_k$ for which the expected reward of arm $k$ can be negative. 
In the traditional bandits setup, one must pull an arm at every time step. Thus, what matters are the relative rewards and the problem is mathematically identical, 
regardless of whether the expected rewards 
range from $-10$ to $0$ or $0$ to $10$. 
In addition, one might construct settings where negative expected rewards are reasonable. 
For example, when one of the arms 
corresponds to no recommendation 
with $0$ being its expected reward (e.g., $b_k=0$, $\lambda_k=0$), 
then
the interpretation of negative expected reward would be that the corresponding arm (item)  
is less preferred relative to not being recommended.
\end{remark}

Given horizon $T \geq 1$, %
we seek an optimal pull sequence $\pi_{1:T}$,
where $\pi_t$ %
depends on past rewards and actions 
$(\pi_1, \mu_{\pi_1, 1}, \ldots, \pi_{t-1}, \mu_{\pi_{t-1}, t-1})$
and maximizes the expected cumulative reward:
\begin{align}\label{eq:obj_G}
    G_T(\pi_{1:T}) := \mathbb{E}\left[ \textstyle \sum_{t=1}^T \mu_{\pi_t,t}\right].
\end{align}

\paragraph{Additional Notation}
Let $\overline{\gamma} \coloneqq \max_{k \in [K]} \gamma_k$ %
and $\overline{\lambda} \coloneqq \max_{k \in [K]} \lambda_k$. 
We use $a \lesssim b$ 
when $a \leq C b$ 
for some positive constant $C$.

\section{Planning with Known Dynamics}    
\label{sec:instantaneous-stochasticity}
Before we can hope to learn an optimal policy
with unknown stochastic dynamics,
we need to establish a procedure for planning
when the satiation retention factors, 
exposure influence factors,
and base rewards are known.
We begin by presenting several planning strategies
and analyzing them under deterministic dynamics,
where the past pulls exactly 
determine each arm's satiation level, i.e., 
$s_{k,t} = \gamma_k (s_{k, t-1} + 
u_{k,t-1})$, $\forall t > t_0^k$. 
With some abuse of notation, 
at time $t \geq 2$,
given a pull sequence $u_{k,0:t-1}$, 
we can express the satiation 
and the {expected}\footnote{
We use ``expected reward'' 
to emphasize that 
all results in this section also apply to settings where the satiation 
dynamics are deterministic 
but the rewards are stochastic, i.e., 
$\mu_{k,t} = b_k - \lambda_k s_{k,t} + e_{k,t}$
for independent mean-zero noises $e_{k,t}$.} reward
of each arm as 
\begin{align}\label{eq:expected-reward}
    {s}_{k,t}(u_{k,0:t-1})%
    &= \gamma_k \left(s_{k,t-1} + u_{k, t-1}\right) = 
    \gamma_k \left( \gamma_k \left(s_{k,t-2} + u_{k, t-2}\right) \right) + \gamma_k u_{k, t-1}
    = \textstyle \sum_{i=1}^{t-1} \gamma_k^{t-i} u_{k,i}, \nonumber\\
   {\mu}_{k,t}(u_{k,0:t-1})
&= %
b_k - \lambda_k \left( \textstyle  \sum_{i=1}^{t-1} \gamma_k^{t-i} u_{k,i}\right). 
\end{align}
At time $t=1$, we have that 
${s}_{k,1}(u_{k,0:0}) = 0$ 
and ${\mu}_{k,1}(u_{k,0:0}) = b_k$ 
for all $k \in [K]$. 
Since the arm parameters 
$\{\lambda_k, \gamma_k, b_k\}_{k=1}^K$ are known, 
our goal~\eqref{eq:obj_G} simplifies 
to finding a pull sequence that solves 
the following bilinear integer program:
\allowdisplaybreaks
\begin{align}\label{eq:obj}
    \!\!\max_{u_{k, t}} 
    \left\{
    \sum_{k=1}^K \sum_{t=1}^T u_{k,t}  \bigg(b_k - \lambda_k   \sum_{i=0}^{t-1} \gamma_k^{t-i} u_{k,i} \bigg)\! :\! 
    \begin{array}{l}
    \sum\limits_{k=1}^K u_{k,t} = 1,  \quad \forall t \in [T],\\ 
    u_{k,t} \in \{0,1\}, ~~
    u_{k,0} =0,  ~ \forall k \in [K], \forall t \in [T] 
    \end{array}\!\!\!\right\}
\end{align}
where the objective maximizes 
the expected cumulative reward 
associated with the pull sequence 
and the constraints ensure 
that at each time period 
we pull exactly one arm.
Note that~\eqref{eq:obj} 
includes products of decision variables $u_{k,t}$
leading to bilinear terms in the objective. 
In Appendix~\ref{appendix:restate-obj}, 
we provide an equivalent integer linear program.

\subsection{The Greedy Policy}
\label{sec:greedy}
At each step, the greedy policy $\pi^g$
picks the arm with the highest 
instantaneous expected reward.
Formally, at time $t$, 
given the pull history
$\{u_{k, 0:t-1}\}_{k=1}^K$, 
the greedy policy picks
\begin{align*}
    \pi^g_t \in \argmax_{k \in [K]} {\mu}_{k,t}(u_{k,0:t-1}).
\end{align*}
In order to break ties,
when all arms have the same expected reward, 
the greedy policy chooses 
the arm with the lowest index.

Note that the greedy policy
is not, in general, optimal.
Sometimes, we are better off allowing
the current best arm
to rebound even further, 
before pulling it again.
\begin{example}
Consider the case with two arms. 
Suppose that arm $1$ 
has base reward $b_1$, 
satiation retention factor 
$\gamma_1 \in (0,1)$,
and exposure influence factor $\lambda_1 = 1$. 
For any fixed time horizon $T > 2$, 
suppose that arm $2$ has 
$b_2 = b_1 + \frac{\gamma_2 - \gamma_2^{T}}{1-\gamma_2}$ 
where $\gamma_2 \in (0, 1)$ and $\lambda_2 = 1$. 
The greedy policy $\pi_{1:T}^g$
will keep pulling arm $2$ until time $T-1$ 
and then play arm $1$ (or arm $2$) at time $T$. 
This is true because if we keep 
pulling arm $2$ until $T-1$, 
at time $T$, we have
${\mu}_{2,T}(u_{2,0:T-1}) = b_1 = {\mu}_{1,T}(u_{1,0:T-1})$.  
However, the policy $\pi^n_{1:T}$,
where $\pi^n_t = 2$ if $t \leq T-2$, 
$\pi^n_{T-1}=1$, and $\pi^n_{T}=2$, 
obtains a higher expected cumulative reward. 
In particular, the difference 
$G_T(\pi^n_{1:T}) - G_T(\pi^g_{1:T})$ 
will be $\gamma_2 - \gamma_2^{T-1}$.
\label{ex:greedy-not-optimal}
\end{example}

\subsection{When is Greedy Optimal?}
When the satiation retention factors $\gamma_k=0$ 
for all $k \in [K]$, 
i.e., when the satiation effect is always $0$, 
we know that the greedy policy 
(which always plays the arm 
with the highest instantaneous expected reward) is optimal. 
However, when satiation can be nonzero,
it is less clear under what conditions 
the greedy policy performs optimally.
This question is of special interest
when we consider human decision-making,
since we cannot expect people 
to solve large-scale 
bilinear integer programs
every time they pick music to listen to.

In this section, we show that 
when all arms share the same properties
($\gamma_k, \lambda_k, b_k$ are identical for $k \in [K]$), 
the greedy policy is optimal. 
In this case, the greedy policy 
exhibits variety-seeking behavior
as it plays the arms cyclically. 
Interestingly, this condition aligns 
with early research that has motivated 
studies on satiation~\citep{tucker1964development, mcconnell1968development}:
when controlling for marketing variables (e.g., 
the arm parameters $\gamma_k, \lambda_k, b_k$), 
researchers still observe variety-seeking behaviors
of consumers (e.g., playing arms in a cyclic order).

\begin{assumption}\label{assump:same-dynamics}
$\gamma_1 = \ldots = \gamma_K = \gamma$,~ $\lambda_1 = \ldots = \lambda_K = \lambda$, 
and $b_1 = \ldots = b_K = b$. 
\end{assumption}

We start with characterizing the greedy policy when Assumption~\ref{assump:same-dynamics} holds.

\begin{lemma}[Greedy Policy Characterization]\label{lemma:greedy-periodic}
Under Assumption~\ref{assump:same-dynamics}
and the tie-breaking rule 
that when all arms have the same expected reward, 
the greedy policy chooses 
the one with the lowest arm index, 
the sequence of arms 
pulled by the greedy policy 
forms a periodic sequence:
$
    \pi_1 = 1, \pi_2 = 2, 
    \cdots, \pi_K = K, \text{ and } \pi_{t+K} = \pi_t, \; %
    \forall t \in \mathbb{N}_+.
$
\end{lemma}
In this case, the greedy policy is equivalent 
to playing the arms in a cyclic order. 
All proofs for the paper 
are deferred to the Appendices.

\begin{theorem}\label{thm:greedy-cumulative-reward}
Under Assumption~\ref{assump:same-dynamics}, given any horizon $T$, the greedy policy $\pi_{1:T}^g$ is optimal.
\end{theorem}

\begin{remark}
Theorem~\ref{thm:greedy-cumulative-reward} suggests
that when the (deterministic) satiation dynamics 
and base rewards are identical across arms, 
planning does not require knowledge 
of those parameters.
\end{remark}

Lemma~\ref{lemma:greedy-periodic} and Theorem~\ref{thm:greedy-cumulative-reward}
lead us to conclude the following result: 
when recommending items 
that share the same properties, 
the best strategy is to show the users 
a variety of recommendations 
by following the greedy policy.

On a related note,  Theorem~\ref{thm:greedy-cumulative-reward}
also gives 
an exact Max K-Cut  
of a complete graph $\mathcal{K}_T$ on $T$ vertices, 
where the edge weight connecting vertices $i$ and $j$ 
is given by $e(i, j) = \lambda\gamma^{|j-i|}$ for $i \neq j$. 
The Max K-Cut problem partitions the vertices of a graph 
into $K$ subsets $P_1, \ldots P_K$, such that 
the sum of the edge weights connecting the subsets 
are maximized~\citep{frieze1997improved}. 
Mapping the Max K-Cut problem back to our original setup, 
each vertex represents a time step. 
If vertex $i$ is assigned to subset $P_k$, 
it suggests that arm $k$ should be played at time $i$. 
The edge weights $e(i,j) = \lambda \gamma^{|j-i|}$ for $i \neq j$ 
can be seen 
as the reduction in satiation influence achieved 
by not playing the same arm 
at both time $i$ and time $j$. 
The goal~\eqref{eq:obj} is to maximize the total satiation influence reduction.

\begin{proposition}[Connection to Max K-Cut]
Under Assumption~\ref{assump:same-dynamics}, %
an optimal solution to~\eqref{eq:obj} 
is given by a Max K-Cut on $\mathcal{K}_T$,
where $\mathcal{K}_T$ is a complete graph on $T$ vertices 
with edge weights $e(i,j) = \lambda \gamma^{|j-i|}$ for all $i \neq j$. 
\label{thm:max k-cut}
\end{proposition}
Using Lemma~\ref{lemma:greedy-periodic} and Theorem~\ref{thm:greedy-cumulative-reward}, 
we obtain that an exact Max K-Cut 
of $\mathcal{K}_T$ 
is given by $\forall k \in [K], P_k = \{t \in [T]: t \equiv k \Mod{K}\}$, which may be a result of separate interest.

\subsection{The $w$-lookahead Policy}
\label{sec:known-different-dynamics}
To model settings where the arms correspond 
to items with different characteristics
(e.g., we can enjoy tacos on consecutive days
but require time to recover 
from a trip to the steakhouse)
we must allow the satiation parameters
to vary across arms.
Here, the greedy policy may not be optimal.
Thus, we consider more general lookahead policies
(the greedy policy is a special case).
Given a window of size $w$ 
and the current satiation levels, %
the $w$-lookahead policy picks actions 
to maximize the total reward
over the next $w$ time steps.
Let $l$ denote $\lceil T/w \rceil$.
Define $t_i = \min \{iw, T\}$ 
for $i \in [l]$ and $t_0 = 0$. 
More formally, the $w$-lookahead policy 
$\pi_{1:T}^w$ is defined as follows: 
for any $i \in [l]$, 
given the previously chosen arms'
corresponding pull histories   
$\{u^w_{k,0:t_{i-1}}\}_{k=1}^K$ where
$u^w_{k,0}=0$ and  
$u^w_{k,t}=1$ if (and only if) $\pi^w_t = k$, 
the next $w$ (or $T \text{ mod } w$) actions 
$\pi^w_{t_{i-1}+1:t_i}$ are given by
\begin{align}\label{eq:lookahead}
    \max_{\pi_{t_{i-1}+1:t_i}} 
    \left\{%
    \sum_{t={t_{i-1}+1}}^{t_i} 
    {\mu}_{\pi_t, t}(u_{\pi_t, 0:t-1}): 
    \begin{array}{l}%
    u_{k, 0:t_{i-1}} = u^w_{k, 0:t_{i-1}}, \quad \forall k \in [K],\\
    \sum_{k=1}^K u_{k,t}=1, \quad \forall t \in [T],  \\ %
    u_{k,t} \in \{0,1\}, \quad \forall k \in [K], t \in [t_i] %
    \end{array}
    \right\} %
\end{align}
In the case of a tie, 
one can pick any of the sequences 
that maximize~\eqref{eq:lookahead}. 
We recover the greedy policy 
when the window size $w=1$, 
and finding the $w$-lookahead policy 
for the window size $w=T$ 
is equivalent to solving~\eqref{eq:obj}. 

\begin{remark}
Another reasonable lookahead policy, 
which requires planning ahead at every time step, 
would be the following:
at every time $t$, plan for the next $w$ actions 
and  
follow them for a single time step. 
Studying the performance of such a policy is of future interest. 
To lighten the computational load, 
we adopt the current $w$-lookahead policy which only requires planning every $w$ time steps. 
\end{remark}

For the rest of the paper, 
we use 
$\texttt{Lookahead}(\{\lambda_k, \gamma_k, b_k\}_{k=1}^K, \{u^w_{k,0:t_{i-1}}\}_{k=1}^K, 
t_{i-1}, t_{i})$ 
to refer to the solution of~\eqref{eq:lookahead}, 
where 
the arm parameters are $\{\lambda_k, \gamma_k, b_k\}_{k=1}^K$, 
the historical pull sequences of all arms
till time $t_{i-1}$ are given by 
$\{u^w_{k,0:t_{i-1}}\}_{k=1}^K$, 
and the solution corresponds to the actions that should be taken 
for the next $t_i - t_{i-1}$ time steps.

\begin{theorem} \label{thm:w-lookahead}
Given any horizon $T$, 
let $\pi^*_{1:T}$ be a solution to~\eqref{eq:obj}. 
For a fixed window size $w \leq T$, we have that 
$$G_T(\pi^*_{1:T}) - G_T(\pi^w_{1:T}) 
\leq 
\frac{\overline{\lambda} \overline{\gamma} %
(1 - \overline{\gamma}^{T-w})}{(1 - \overline{\gamma})^2} \lceil T/w \rceil.$$ 
\end{theorem} 

\begin{remark}\label{rem:w-lookahead}
Note that when $w = T$, the $w$-lookahead policy 
by definition is the optimal policy 
and in such case, the upper bound 
for the optimality gap of $w$-lookahead established in Theorem~\ref{thm:w-lookahead} is also $0$. 
In contrast to the optimal policy, 
the computational benefit of the $w$-lookahead policy 
becomes apparent when the horizon $T$ is large 
since it requires solving 
for a much smaller program~\eqref{eq:lookahead}. 
In general, the $w$-lookahead policy 
is expected to perform much better 
than the greedy policy 
(which corresponds to the case of $w=1$)
at the expense of a higher computational cost. 
Finally, we note that for 
the window size of $w = \sqrt{T}$,
we obtain $G_T(\pi^*_{1:T}) - G_T(\pi^w_{1:T}) \leq O(\sqrt{T})$. 
\end{remark}

\section{Learning with Unknown Dynamics: Preliminaries}
\label{sec:new_stochastic_satiation}

When the satiation dynamics 
are unknown and stochastic 
($\sigma_z > 0$), 
the learner faces a continuous-state
partially observable MDP 
because the satiation levels
are not observable. 
To set the stage,
we first introduce our
state representation
(\S~\ref{sec:state}) 
and a regret-based performance 
measure (\S~\ref{sec:regrets}).
In the next section, we will introduce EEP, 
our algorithm for rebounding bandits.

\subsection{State Representation}
\label{sec:state}
Following~\cite{ortner2012regret},
at any time $t \in [T]$,
we define a state vector $x_t$ 
in the state space $\mathcal{X}$
to be $x_t = (x_{1,t}, n_{1,t}, x_{2,t}, n_{2, t}, \ldots, x_{K,t}, n_{K,t})$,  
where $n_{k,t} \in \mathbb{N}$
is the number of steps at time $t$
since arm $k$ was last selected
and $x_{k,t} $ is the satiation influence 
(product of $\lambda_k$ and the satiation level)
as of the most recent pull of arm $k$.
Since the most recent pull happens at $t - n_{k,t}$,
we have  $x_{k,t} = b_k - \mu_{k, t - n_{k,t}}
= \lambda_k s_{k, t - n_{k,t}}$.
Recall that $\mu_{k,t - n_{k,t}}$ 
is the reward collected 
by pulling arm $k$ at time $t-n_{k,t}$. 
Note that %
$b_k$ is directly observed when arm $k$
is pulled for the first time
because there is no satiation effect.
The state at the first time step
is $x_{1} = (0, \ldots, 0)$. 
Transitions between two states 
$x_t$ and $x_{t+1}$ 
are defined as follows:  
If arm $k$ is chosen at time $t$, 
and reward $\mu_{k,t}$ is obtained, 
then the next state $x_{t+1}$ will satisfy 
(i) for the pulled arm $k$,
    $n_{k, t+1} = 1$
    and $x_{k, t+1} = b_k - \mu_{k,t}$;
(ii) for other arms $k' \neq k$, 
    $n_{k', t+1} = n_{k',t} +1$ if $n_{k', t} \neq 0$,
    $n_{k', t+1} = 0$ if $n_{k', t} = 0$,
    and the satiation influence remains 
    the same $x_{k',t+1} = x_{k', t}$. 

Given $\{\gamma_k, \lambda_k, {b}_k\}_{k=1}^K$, 
the reward function $r: \mathcal{X} \times [K] \to \mathbb{R}$ 
represents the \emph{expected} reward 
of pulling arm $k$ under state $x_t$: 

If $n_{k, t} =0$, 
then ${r}(x_t, k) = b_{k}$.
If $n_{k,t} \geq 1$, 
${r}(x_t, k) = b_{k} - \gamma_k^{n_{k,t}} x_{k, t} - \lambda_k \gamma_k^{n_{k,t}}$,
which equals $\mathbb{E}[\mu_{k,t} | x_t]$,
where the expectation is taken over the noises 
in between the current pull 
and the last pull of arm $k$. 
See Appendix~\ref{sec:mdp-setup} for
the full description of the MDP setup 
(including the transition kernel 
and value function definition)
of rebounding bandits. 

\subsection{Evaluation Criteria: $w$-step Lookahead Regret}
\label{sec:regrets}

In reinforcement learning (RL), 
the performance of a learner 
is often measured through 
a regret that compares the 
expected cumulative reward obtained by 
the learner against that of an
optimal policy in a competitor class~\cite{lattimore2020bandit}.  
In most episodic (e.g., finite horizon) RL  literature~\cite{ortner2012online,jaksch2010near}, 
regrets are defined in terms of episodes.
In such cases, the initial state 
is reset (e.g., to a fixed state)
after each episode ends, 
independent of previous actions taken by the leaner.
Unlike these episodic RL setups, 
in rebounding bandits, 
we cannot restart from the initial state 
because the satiation level
cannot be reset   
and user's memory depends 
on past received recommendations. 
Instead,~\cite{pike2019recovering} 
proposed 
a version of \emph{$w$-step lookahead regret}
that 
divides the 
$T$ time steps 
into $\lceil T/w \rceil$ episodes
where each episode (besides the last)  
consists of $w$ time steps. 
At the beginning of each episode, 
the initial state is reset 
but depends on how the learner  
has interacted with the user previously.
In particular, at the beginning 
of episode $i+1$ (at time $t=iw+1$), 
given that the learner has  
played $\pi_{1:iw}$ with corresponding 
pull sequence $u_{k,0:iw}$ for $k \in [K]$, 
we reset the initial state to be 
$x^i = (\mu_{1,iw+1}(u_{1,0:iw}), n_{1,iw+1}, \ldots,  
\mu_{K,iw+1}(u_{K,0:iw}), n_{K,iw+1})$
where $\mu_{k,t}(\cdot)$ is defined 
in~\eqref{eq:expected-reward}
and $n_{k,iw+1}$ is the number of steps 
since arm $k$ is last pulled 
by the learner as of time $iw+1$. 
Then, given the learner's policy $\pi_{1:T}$,
where $\pi_t: \mathcal{X} \to [K]$,  
the $w$-step lookahead regret, 
against a competitor class 
$\mathcal{C}^w$ (which we define later), 
is defined as follows:
\begin{align}\label{eq:lookahead-regret}
     \text{Reg}^{w}(T) &= \textstyle \sum_{i=0}^{\lceil T/w \rceil-1} 
    \max_{\Tilde{\pi}_{1:w} \in \mathcal{C}^w}\mathbb{E}\left[ \sum_{j=1}^{\min\{w, T-iw\}} r(x_{iw+j}, \Tilde{\pi}_{j}(x_{iw+j})) \Big| x_{iw+1} = x^i%
    \right] \nonumber \\
    &\qquad \textstyle - 
    \mathbb{E}\left[ \sum_{j=1}^{\min\{w,T-iw\}}  r(x_{iw+j}, {\pi}_{iw+j}(x_{iw+j})) \Big| x_{iw+1} = x^i%
    \right],
\end{align}
where the expectation is taken over 
$x_{iw+2}, \ldots, x_{\min\{(i+1)w, T\}}$.

The competitor class $\mathcal{C}^w$ 
that we have chosen consists of policies 
that depend on time steps, i.e., 
$
    \mathcal{C}^w = \left\{\Tilde{\pi}_{1:w}: 
    \Tilde{\pi}_t = \Tilde{\pi}_t(x_t) = \Tilde{\pi}_t(x_t'), \tilde{\pi}_t \in [K], \forall t \in [w], x_t, x_t' \in \mathcal{X} \right\}. 
$
We note that $\mathcal{C}^w$ 
subsumes many traditional competitor classes
in bandits literature, 
including the class of fixed-action policies 
considered in adversarial bandits~\cite{lattimore2020bandit}
and the class of periodic ranking policies~\cite{cella2020stochastic}. 
In our paper, the $w$-lookahead policy 
(including the $T$-lookahead policy given by~\eqref{eq:obj})
is a time-dependent policy 
that belongs to $\mathcal{C}^w$, 
since at time $t$, it will play a fixed action by solving~\eqref{eq:lookahead} using the true 
reward parameters $\{\lambda_k, \gamma_k, b_k\}_{k=1}^K$.
The time-dependent competitor class $\mathcal{C}^w$ 
differs from a state-dependent competitor class 
which includes all measurable functions $\Tilde{\pi}_t$ 
that map from $\mathcal{X}$ to $[K]$.
The state-dependent competitor class contains 
the optimal policy $\pi^*$
where $\pi^*_t(x_t)$ depends
on not just the time step 
but also the exact state $x_t$.
Finding the optimal state-dependent 
policy requires optimal planning for 
a continuous-state MDP,
which relies on state space discretizion~\cite{ortner2012online} 
or function approximation 
(e.g., approximate dynamic programming algorithms~\citep{munos2007performance, ernst2005tree, riedmiller2005neural}).
In Appendix~\ref{appendix:mdp-related}, 
we provide discussion and analysis on 
an algorithm compared against 
the optimal state-dependent policy. 
We proceed the rest of the main paper
with $\mathcal{C}^w$ defined above.

When $w=1$, the $1$-step lookahead regret 
is also known as the {instantaneous} regret, 
which is commonly used 
in restless bandits literature
and some nonstationary bandits papers including~\cite{mintz2020nonstationary}. 
Note that low instantaneous regret 
does not imply high expected cumulative 
reward in the long-term, 
i.e., one may benefit more by 
waiting for certain arms to rebound.  
When $w=T$, we recover the full horizon regret. 
As we have noted earlier, 
finding the optimal competitor policy
in this case 
is computationally intractable 
because the number of states,
even when the satiation dynamics are deterministic, 
grows exponentially with the horizon $T$. 
Finally, we note that the 
$w$-step lookahead regret 
can be obtained for not just 
policies designed to look $w$ 
steps ahead but 
any given policy. 
For a more comprehensive discussion
on these notions of regret,
see~\cite[Section 4]{pike2019recovering}.

\section{Explore-Estimate-Plan}
\label{sec:etc}

We now present \emph{Explore-Estimate-Plan (EEP)},
an algorithm for learning in rebounding bandits
with stochastic dynamics and unknown parameters,
that (i) collects data by pulling 
each arm a fixed number of times;
(ii) estimates the model's parameters
based on the logged data;
and then (iii) plans according 
to the estimated model.
Finally, we analyze EEP's regret.

Because each arm's base reward 
is known from the first pull, 
whenever arm $k$ is pulled 
at time $t$ and $n_{k,t} \neq 0$,  
we measure the satiation influence 
$\lambda_k s_{k,t}$,
which becomes the next state $x_{k,t+1}$: 
\begin{align}\label{eq:influence-dynamics-main-text}
   x_{k,t+1} &= \lambda_{k} s_{k, t} = 
    \lambda_{k} \gamma_k^{n_{k,t}} s_{k, t -n_{k,t}} +  \lambda_{k}\gamma_k^{n_{k,t}} +  \lambda_{k} \textstyle 
    \sum_{i=0}^{n_{k,t} - 1} \gamma_k^{i}  z_{k, t-1-i} \nonumber \\
    &= \gamma_k^{n_{k,t}} x_{k, t+1-n_{k,t}} + \lambda_k \gamma_k^{n_{k,t}}  + \lambda_k   \textstyle 
    \sum_{i=0}^{n_{k,t}-1} \gamma_k^{i} z_{k, t-1-i}.   
\end{align}
We note that the current state $x_{k,t}$ 
equals $x_{k, t+1-n_{k,t}}$,
since $x_{k, t+1-n_{k,t}}$ 
is the last observed satiation influence for arm $k$ 
and $n_{k,t}$ is the number of steps 
since arm $k$ was last pulled. 

\subsection{The Exploration Phase: Repeated Pulls}
\label{sec:exploration}
We collect a dataset $\mathcal{P}_k^n$ 
by consecutively pulling 
each arm $n+1$ times, in turn,
where $n \geq \lfloor T^{2/3}/K \rfloor$
(Line 4-7 of Algorithm~\ref{algo:rebounding_etc}).
Specifically, for each arm $k \in [K]$, 
the dataset $\mathcal{P}^{n}_k$ 
contains a single trajectory 
of $n+1$ observed satiation influences 
$\tilde{x}_{k,1}, \ldots, \tilde{x}_{k, n+1}$, 
where $\tilde{x}_{k,1} = 0$
and $\tilde{x}_{k,j}$ ($j > 1$) is
the difference between the first reward 
and the $j$-th reward from arm $k$. 
Thus, for $\tilde{x}_{k,j}, \tilde{x}_{k,j+1} \in \mathcal{P}^{n}_{k}$, 
using~\eqref{eq:influence-dynamics-main-text} 
with $n_{k,t}=1$ (because pulls are consecutive),
it follows that 
\begin{align}\label{eq:affine-system}
    \tilde{x}_{k,j+1} = \gamma_k \tilde{x}_{k, j} + d_k + \tilde{z}_{k,j},
\end{align}
where $d_k=\lambda_k \gamma_k$ and $\tilde{z}_{k,j}$ 
are independent samples 
from $\mathcal{N}(0, \sigma_{z,k}^2)$ 
with $\sigma_{z,k}^2 = \lambda_k^2 \sigma_z^2$. 
In Appendix~\ref{appendix:other-exploration-strategy}, 
we discuss other exploration strategies 
(e.g., playing the arms cyclically)
for EEP and their regret guarantees.

\subsection{Estimating the Reward Model and Satiation Dynamics}
\label{sec:system-identification-main-text}

For all $k \in [K]$, given the dataset  $\mathcal{P}^{n}_{k}$, 
we estimate $A_k = (\gamma_k, d_k)^\top$
using the \emph{ordinary least squares estimator}:
$$\widehat{A}_k \in \argmin_{A \in \mathbb{R}^2}\|\mathbf{Y_k} - \mathbf{\overline{X}_k} A\|_2^2,$$
where $\mathbf{Y_k} \in \mathbb{R}^n$ 
is an $n$-dimensional vector 
whose $j$-th entry is $\tilde{x}_{k, j+1}$ 
and $\mathbf{\overline{X}_k} \in \mathbb{R}^{n \times 2}$ 
takes as its $j$-th row the vector 
$\overline{x}_{k,j} = (\tilde{x}_{k, j}, 1)^\top$, 
i.e., 
$\tilde{x}_{k, j+1}$ is treated to be
the response to the covariates $\overline{x}_{k,j}$.
This suggests that 
\begin{align}\label{eq:affine-estimator}
    \widehat{A}_k
    = \begin{pmatrix}
           \widehat{\gamma}_k\\
           \widehat{d}_k
    \end{pmatrix}
    = \left(\mathbf{\overline{X}_k}^\top \mathbf{\overline{X}_k}\right)^{-1} \mathbf{\overline{X}_k}^\top \mathbf{Y_k}, 
\end{align}
and we take 
$\widehat{\lambda}_k = |\widehat{d}_k/\widehat{\gamma}_k|$.

The difficulty in analyzing 
the ordinary least squares estimator~\eqref{eq:affine-estimator} 
for identifying an affine dynamical system~\eqref{eq:affine-system} 
using a single trajectory of data 
comes from the fact that 
the samples are not independent. 
Asymptotic guarantees of the 
ordinary least squares estimators in this case  
have been studied previously 
in the control theory and time series communities
\citep{hamilton1994time,ljung1999system}. 
Recent work on system identifications 
for linear dynamical systems 
focuses on the sample complexity~\cite{simchowitz2018learning,sarkar2019near}. 
Adapting the proof of~\cite[Theorem 2.4]{simchowitz2018learning}, 
we derive the following theorem 
for identifying our affine 
dynamical system~\eqref{eq:affine-system}. 

\begin{theorem} 
Fix $\delta \in (0, 1)$. %
For all $k \in [K]$, 
there exists a constant  
$n_0(\delta, k)$ 
such that if the dataset $\mathcal{P}_k^{n}$ 
satisfies $n \geq n_0(\delta, k)$,
then 
\allowdisplaybreaks
\begin{align*}
    \mathbb{P}\left(\|\widehat{A}_k - A_k\|_2 \gtrsim \sqrt{1/(\psi n)} \right) \leq \delta, 
\end{align*}
where 
$
   \psi = \sqrt{\min\left\{\frac{\sigma_{z,k}^2(1-\gamma_{k})^2}{{16 d_k^2(1-\gamma_k^2)} + (1-\gamma_{k})^2{\sigma_{z,k}^2}},  \frac{\sigma_{z,k}^2}{4(1-\gamma_k^2)} \right\}}.
$
\label{thm:a_b_estimation_m=1} 
\end{theorem}

\setlength{\textfloatsep}{10pt}
\begin{algorithm}[t]
\SetAlgoLined 
\KwIn{Lookahead window size $w$, Number of arms $K$, Horizon $T$}
\textbf{Initialize} $t=1$, $\pi_{1:T}$ to be an empty array of length $T$
and $\widetilde{T} = T^{2/3} + w - (T^{2/3} \text{ mod } w)$.\\
\For{$k = 1, \ldots, K $}{
    Set $t'=t$ and initialize an empty array $\mathcal{P}_k^n$.\\
    \For{ $c = 0, \ldots, \lfloor \widetilde{T}/K \rfloor$}{
        Play arm $k$ to obtain reward $\mu_{k,t'+c}$ and add $ \mu_{k,t'}-\mu_{k,t' +c}$ to $\mathcal{P}^n_k$.\\
        Set $\pi_t = k$
        and increase $t$ by $1$.
    } 
    Obtain $\widehat{\gamma}_k, \widehat{d}_k$ using the estimator~\eqref{eq:affine-estimator},
    set $\widehat{\lambda}_k = |\widehat{d}_k/\widehat{\gamma}_k|$
    and $\widehat{b}_k=\mu_{k,t'}$.
}
Let $t_0 = \widetilde{T}$,
set $\pi_{t:t_0} = (1, \ldots, \widetilde{T} - t+1)$, 
and play $\pi_{t:t_0}$.\\
\For{$i = 1, \ldots, \lceil \frac{T - t_0}{w} \rceil$}{
    Set $t_i = \min\{t_{i-1} + w, T\}$.\\
    Obtain $\pi_{t_{i-1}+1:t_i} = \texttt{Lookahead}(\{\widehat{\lambda}_k, \widehat{\gamma}_k, \widehat{b}_k\}_{k=1}^K, \{u_{k,0:t_{i-1}}\}_{k=1}^K, t_{i-1}, t_{i})$
    where $\{u_{k,0:t_{i-1}}\}_{k=1}^K$ 
    are the arm pull histories correspond to 
    $\pi_{1:t_{i-1}}$.\\
    Play $\pi_{t_{i-1}+1,t_i}$.
}
\caption{$w$-lookahead Explore-Estimate-Plan}
\label{algo:rebounding_etc}
\end{algorithm}

As shown in Theorem~\ref{thm:a_b_estimation_m=1}, 
when $d_k = \lambda_k \gamma_k $ gets larger, 
the convergence rate for $\widehat{A}_k$ gets slower.
Given a single trajectory of sufficient length,
we obtain $|\widehat{\gamma}_k - \gamma_k| \leq O(1/\sqrt{n})$
and $|\widehat{d}_k - d_k| \leq O(1/\sqrt{n})$.
In Corollary~\ref{cor:estimation-single_m=1}, 
we show that the estimator of $\lambda_k$ 
also achieves $O(1/\sqrt{n})$ estimation error. 

\begin{corollary} \label{cor:estimation-single_m=1}
Fix $\delta \in (0,1)$. 
Suppose that for all $k \in [K]$, %
we have $\mathbb{P}(\|\widehat{A}_k - A_k\|_2 \gtrsim 1/\sqrt{n} ) \leq \delta$ 
and $\widehat{\gamma}_k > 0$.
Then, with probability $1 - \delta$,
we have that for all $k \in [K]$,  
\begin{align*}
        |\widehat{\gamma}_k - \gamma_k|
    \leq O\left(\frac{1}{\sqrt{n}} \right),
    \quad |\widehat{\lambda}_k - \lambda_k| \leq O\left(\frac{1}{\sqrt{n}} \right)
    .
\end{align*}
\end{corollary}

\subsection{Planning and Regret Bound}
\label{sec:regret-bound}
In the planning stage of Algorithm~\ref{algo:rebounding_etc} (Line 11-15), 
at time $t_{i-1}+1$, 
the next $w$ arms to play 
are obtained through the \texttt{Lookahead} function
defined in~\eqref{eq:lookahead}
based on the estimated parameters 
from the estimation stage (Line 8).
Using the results in Corollary~\ref{cor:estimation-single_m=1},
we obtain the following sublinear 
regret bound for $w$-lookahead EEP.

\begin{theorem}\label{thm:regret-upper}
There exists a constant $T_0$ such that 
for all $T > T_0$ and $w \leq T^{2/3}$, 
the $w$-step lookahead regret 
of $w$-lookahead Explore-Estimate-Plan 
satisfies 
\begin{align*}
    \text{Reg}^{w}(T) \leq {O}(K^{1/2}T^{2/3}\log T). 
\end{align*}
\end{theorem} 
\begin{remark}
The fact that EEP incurs 
a regret of order $O(T^{2/3})$ 
is expected for two reasons:
First, EEP can be viewed as an 
explore-then-commit (ETC) algorithm
that first explores then exploits.
The regret of EEP resembles 
the $O(T^{2/3})$ regret of the ETC algorithm 
in the classical $K$-armed 
bandits setting~\cite{lattimore2020bandit}.
In rebounding bandits, 
the fundamental obstacle to mixing 
the exploration and exploitation stages
is the need to estimate the satiation dynamics. 
When the rewards of each arm
are not observed periodically, 
the obtained satiation influences
can no longer be viewed as samples
from the same time-invariant 
affine dynamical system,
since the parameters of the system 
depend on the duration between pulls. 
In practice, one may utilize 
the maximum likelihood estimator
to obtain estimates 
of the reward parameters %
but obtaining the sample complexity 
of such an estimator 
with dependent data
is difficult.
Second, it has been shown in
\cite{besbes2019optimal} 
that when the rewards of the arms 
have temporal variation 
that depends on the horizon $T$, 
the worst case instantaneous regret 
has a lower bound $\Omega(T^{2/3})$.  
On the other hand, 
in $K$-armed bandits, 
the regret (following the classical definition~\cite{lattimore2020bandit})
is lower bounded by $\Omega(T^{1/2})$,
and can be attained by methods 
like the upper confidence bound 
algorithm~\cite{lattimore2020bandit}. 
Precisely characterizing 
the regret lower bound for rebounding bandits 
is of future interest. 
\end{remark}

\section{Experiments}
\label{sec:experiment}

\begin{figure*}
     \centering
     \begin{subfigure}[b]{0.45\linewidth}
         \centering
         \includegraphics[width=\linewidth]{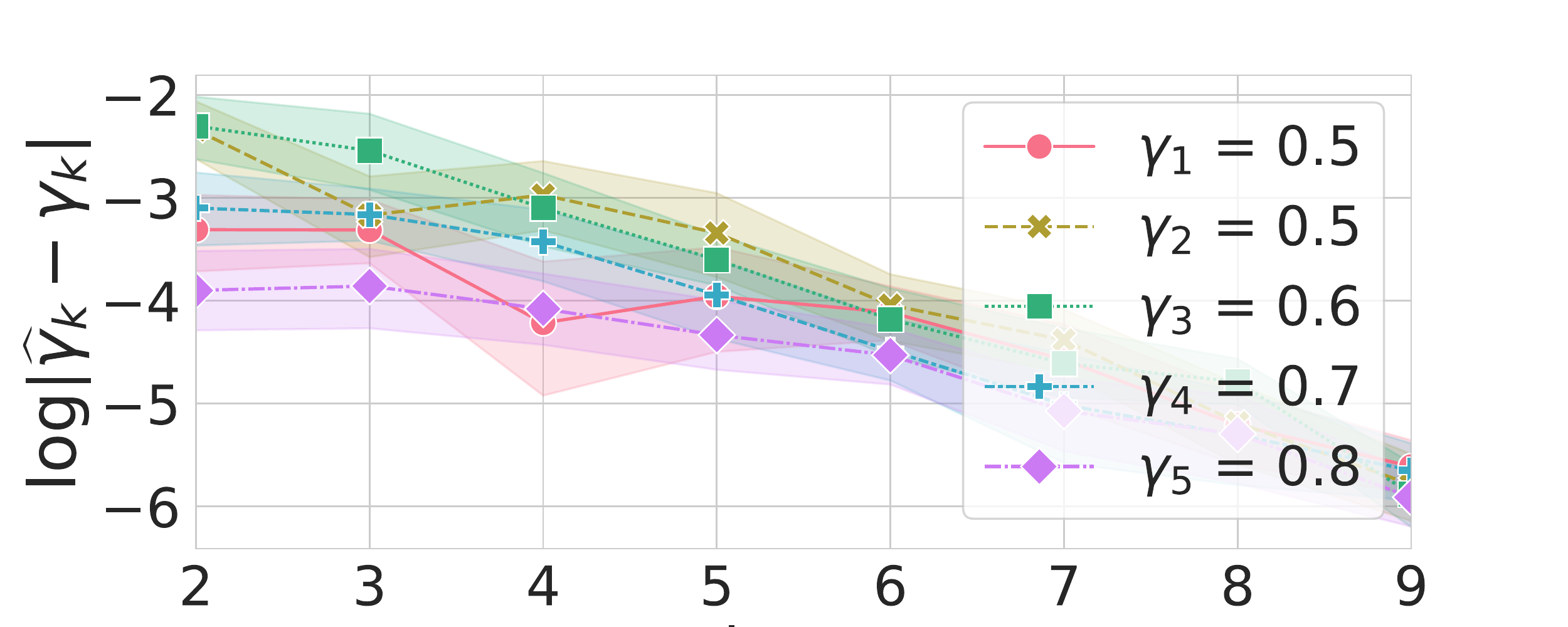}
         \caption{$%
         \log |\widehat{\gamma}_k - \gamma_k|$ v.s. $\log n$}
         \label{fig:gamma_m_1}
     \end{subfigure}
     \hfill
     \begin{subfigure}[b]{0.45\linewidth}
         \centering
         \includegraphics[width=\linewidth]{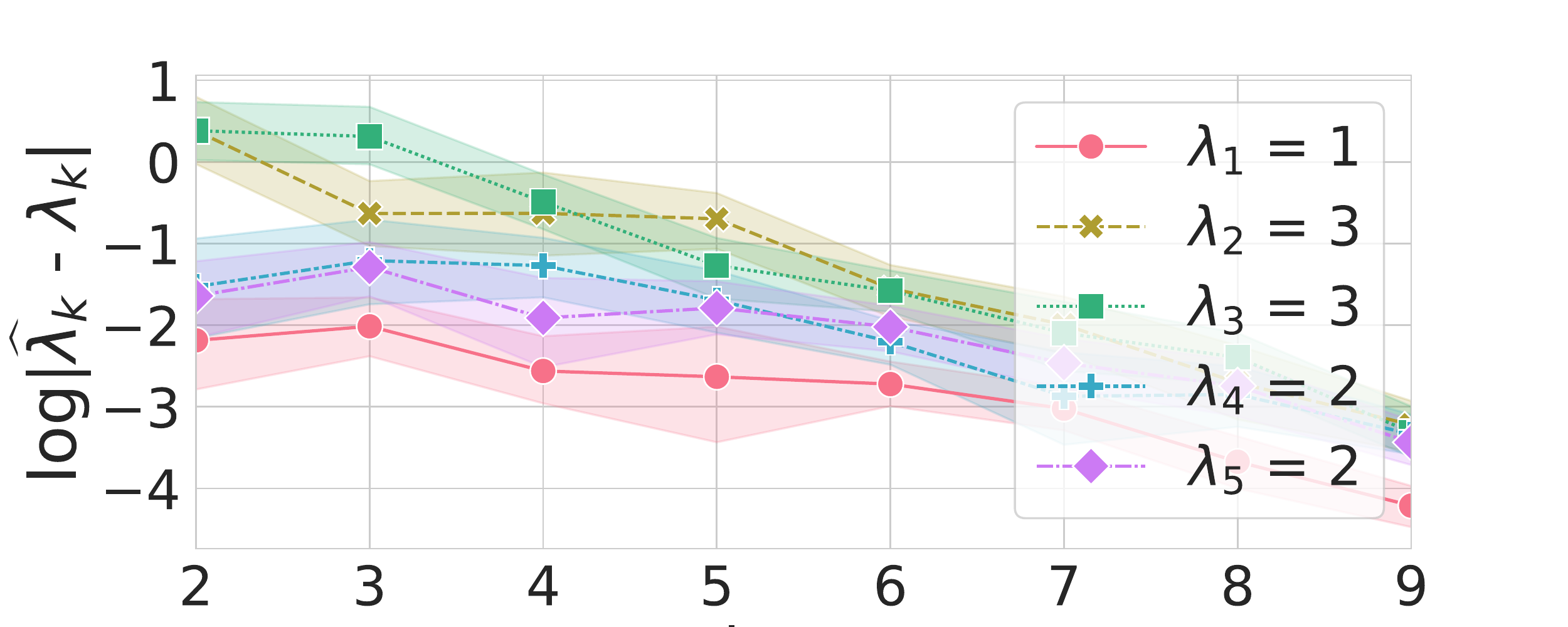}
         \caption{$%
         \log |\widehat{\lambda}_k - \lambda_k|$  v.s. $\log n$}
         \label{fig:lambda_m_1}
     \end{subfigure}
        \caption{Figure~\ref{fig:gamma_m_1} and~\ref{fig:lambda_m_1} are the $\log$-$\log$ plots 
        of absolute errors of $\widehat{\gamma}_k$ and $\widehat{\lambda}_k$ 
        with respect to the number of samples $n$ in a single trajectory.
        The results are averaged over $30$ random runs, 
        where the shaded area represents one standard 
        deviation. 
        }
        \label{fig:estimation_m_1}
        \vspace{-8px}
\end{figure*}

\begin{figure*}
     \centering
     \begin{subfigure}[b]{0.45\linewidth}
         \centering
         \includegraphics[width=\linewidth]{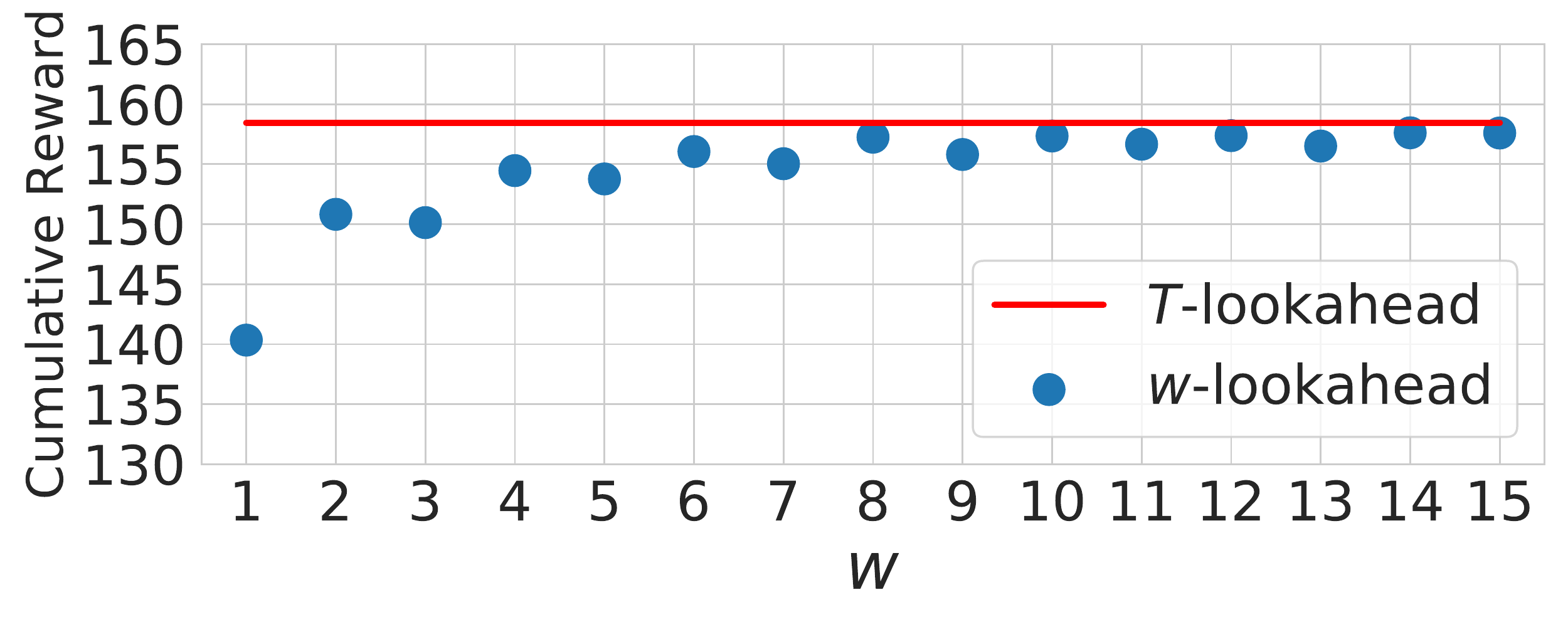}
         \caption{$T=30$}
         \label{fig:T_30_lookahead}
     \end{subfigure}
     \hfill
     \begin{subfigure}[b]{0.45\linewidth}
         \centering
         \includegraphics[width=\linewidth]{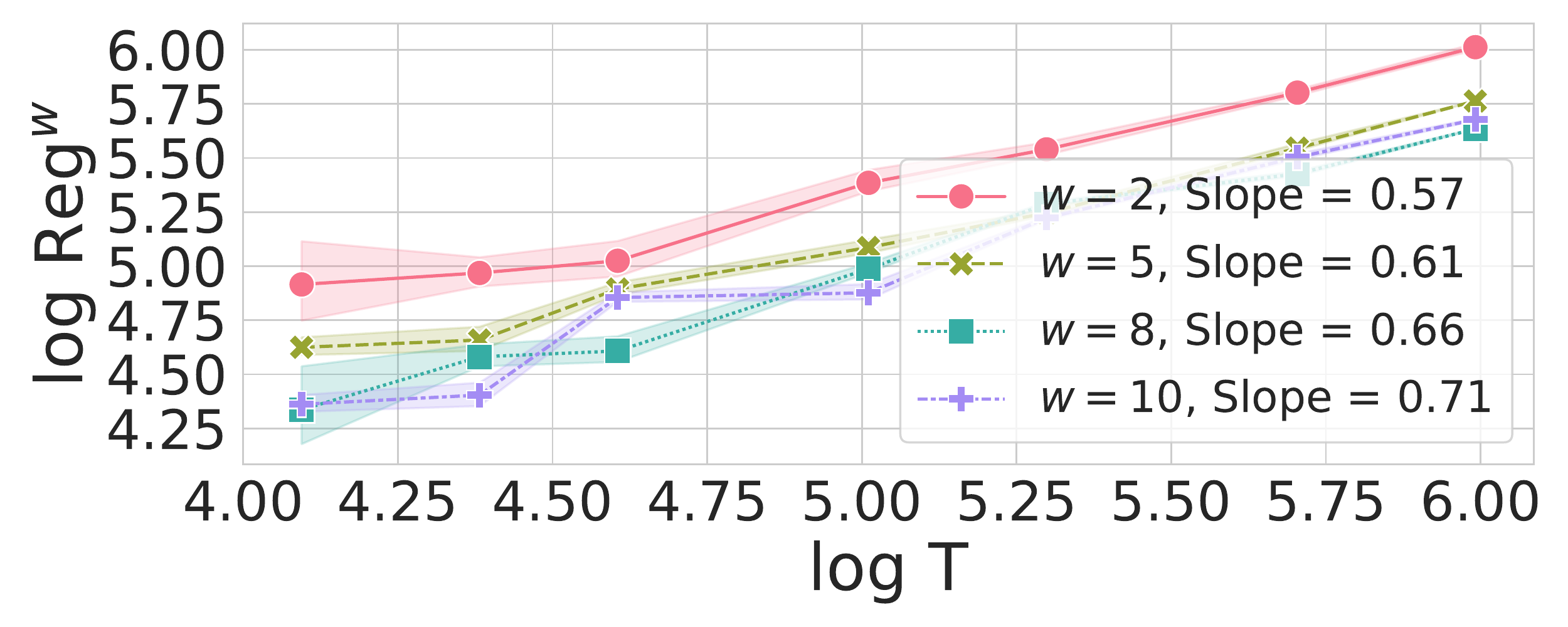}
         \caption{$\log \text{Reg}^w$ v.s. $\log T$}
         \label{fig:eep_performance}
     \end{subfigure}
        \caption{
        Figure \ref{fig:T_30_lookahead}
        shows the expected cumulative reward 
        collected by 
        the $T$-lookahead policy (red line) 
        and $w$-lookahead policy (blue dots)
        when $T=30$. 
        Figure~\ref{fig:eep_performance} 
        shows the log-log plot of the 
        $w$-step lookahead regret 
        of $w$-lookahead EEP 
        under different $T$ averaged over $20$ random runs.
        }
        \label{fig:lookahead_opt}
\end{figure*}

We now evaluate the performance of EEP experimentally, 
separately investigating the sample efficiency 
of our proposed estimators~\eqref{eq:affine-estimator} 
for learning the satiation and reward models 
(Figure~\ref{fig:estimation_m_1})
and the computational performance of 
the $w$-lookahead policies~\eqref{eq:lookahead}
(Figure~\ref{fig:T_30_lookahead}). 
For the experimental setup, 
we have $5$ arms with satiation retention factors 
$\gamma_1 = \gamma_2 = .5, \gamma_3 = .6, \gamma_4 =.7, \gamma_5 =.8$, 
exposure influence factors 
$\lambda_1 =1$, $\lambda_2 = \lambda_3 = 3, \lambda_4 = \lambda_5 = 2$, 
base rewards  
${b}_1 = 2, {b}_2 = 3, {b}_3 = 4, {b}_4 = 2, {b}_5 = 10$,
and noise with variance $\sigma_z=0.1$.

\paragraph{Parameter Estimation} 
We first evaluate our proposed estimator for 
using a single trajectory per arm 
to estimate the arm parameters $\gamma_k, \lambda_k$. 
In Figure~\ref{fig:estimation_m_1}, 
we show the absolute error (averaged over $30$ random runs)
between the estimated parameters 
and the true parameters for each arm. 
Aligning with our theoretical guarantees
(Corollary~\ref{cor:estimation-single_m=1}),
the log-log plots show that 
the convergence rate 
of the absolute error 
is on the scale of $O(n^{-1/2})$.

\paragraph{$w$-lookahead Performance} 
To evaluate $w$-lookahead policies, 
we solve~\eqref{eq:lookahead} 
using the true reward parameters 
and report expected cumulative rewards 
of the obtained $w$-lookahead policies 
(Figure~\ref{fig:T_30_lookahead}). 
Recall that the greedy policy 
is precisely the $1$-lookahead policy. 
In order to solve the resulting integer programs, 
we use Gurobi 9.1~\cite{gurobi} %
and set the number of threads
for solving the problem to $10$. 
When $T=30$, the $T$-lookahead policy 
(expected cumulative rewards 
given by the red line 
in Figure~\ref{fig:T_30_lookahead})
solved through~\eqref{eq:obj} is obtained in $1610$s.
On the other hand, 
all $w$-lookahead policies
(expected cumulative rewards 
given by the blue dots 
in Figure~\ref{fig:T_30_lookahead})
for $w$ in between $1$ and $15$ 
are solved within $2$s. %
We provide the results when $T=100$ 
in Appendix~\ref{appendix:experiment}.
Despite using significantly lower computational time, 
$w$-lookahead policies achieve 
a similar expected cumulative reward 
to the $T$-lookahead policy.

\paragraph{EEP Performance}
We evaluate the performance of EEP 
when $T$ ranges from $60$ to $400$.  
For each horizon $T$, 
we examine the $w$-step lookahead regret
of $w$-lookahead EEP where 
$w = 2, 5, 8, 10$. 
All results are averaged over $20$ random runs.
As $T$ increases, 
the exploration stage of EEP becomes 
longer, 
which results in collecting
more data for 
estimating the reward parameters
and lower variance 
of the parameter estimators. 
We fit a line for the regrets with
the same lookahead size $w$ 
to examine the order of the regret 
with respect to the horizon $T$. 
The slopes of the lines 
(see Figure~\ref{fig:eep_performance}'s legend) 
are close to $2/3$,
which aligns
with our theoretical guarantees 
(Theorem~\ref{thm:regret-upper}), 
i.e., the regrets are on the order of $O(T^{2/3})$. 
In Appendix~\ref{appendix:experiment}, 
we present 
additional experimental setups and
results.

\section{Conclusions}
While our work has taken strides
towards modeling the exposure-dependent
evolution of preferences 
through dynamical systems, 
there are many avenues for future work.
First, while 
our
satiation dynamics 
are independent across arms,
a natural extension %
might allow interactions among the arms.
For example, a diner sick of pizza 
after too many trips to Di Fara's,
likely would also avoid Grimaldi's
until the satiation effect wore off.
On the system identification side,
we might overcome our reliance
on evenly spaced pulls,
producing more adaptive algorithms
(e.g., optimism-based algorithms)
that can refine their estimates,
improving the agent's policy
even past the pure exploration period.
Finally, our satiation model captures just one plausible dynamic according to which preferences might evolve in response to past recommendations.  
Characterizing other such dynamics 
(e.g., the formation of brand loyalty 
where the rewards of an arm increase with more pulls) in bandits setups 
is of future interest.

\section*{Acknowledgement} 
LL is generously supported by an Open Philanthropy AI Fellowship. 
The authors would like to thank David Childers, Biswajit Paria, Eyan P. Noronha, Sai Sandeep and Max Simchowitz for very helpful discussions, 
and Stephen Tu for his insightful suggestions 
on system identification of affine dynamical systems.

\bibliographystyle{plain}
\bibliography{personalization-refs}

\newpage
\appendixwithtoc
\newpage

\section{Integer Linear Programming Formulation}
\label{appendix:restate-obj}
The bilinear integer program of~\eqref{eq:obj} admits the following equivalent linear integer programming formulation: 
\begin{align*}
\begin{split}
    \max_{u_{k, t}, z_{k,t,i}} %
    &\sum_{k \in [K]} \sum_{t \in [T]}   b_k u_{k,t} - \lambda_k  \sum_{i=0}^{t-1} \gamma_k^{t-i} z_{k,t, i}
    \\ 
    \text{s.t. } 
    &\sum_{k \in [K]} u_{k,t} = 1,  \qquad \forall t \in [T], \\
    &z_{k, t, i} \leq u_{k,i},~~ 
    z_{k,t,i} \leq u_{k,t},~~ u_{k,i} + u_{k,t} - 1 \leq z_{k, t, i},  
    \qquad \forall k \in [K], 
    t \in [T], i \in \{0, \ldots, t-1\}, \\
    & u_{k,t} \in \{0,1\},~~
    u_{k,0} =0,  \qquad \forall k \in [K], t \in [T], \\
    & z_{k,t,i} \in \{0,1\}, 
    \qquad \forall k \in [K], 
    t \in [T], i \in \{0, \ldots, t-1\}.
\end{split}
\end{align*}

\newpage

\section{Proofs and Discussion of Section~\ref{sec:instantaneous-stochasticity}}
\label{appendix:known-deterministic-proof}

\subsection{Proof of Lemma~\ref{lemma:greedy-periodic}}
\begin{proof}%
When the expected rewards of all arms are the same, 
we know that the arm with the lowest index 
will be chosen and thus the first $K$ pulls 
will be $\pi_1=1, \ldots, \pi_K=K$. 
We will complete the proof through induction.
Suppose that the greedy pull sequence is periodic 
with $\pi_1=1, \ldots, \pi_K=K$ and $\pi_{t +K} = \pi_t$ until time $h > K$. 
We define $k'$ to be $h \text{ mod } K$ and $n$ to be  $(h - k')/K$.
We will show that $\pi_{h+1} = 1$ if $\pi_h = K$
and $\pi_{h+1} = \pi_h + 1$ otherwise.
When $k' = 0$ (i.e., $\pi_h = K$), 
all arms have been pulled exactly $n$ times as of time $h$.
By the induction assumption, we know that 
$
    u_{1,1:h-K} = u_{2,2:h-K+1}
    = \ldots =  u_{K, K:h}, %
$ 
which implies that last time when each arm is pulled, 
all of them have the same expected rewards, i.e., 
\begin{align*}
    &{\mu}_{1, h-K+1}(u_{1,0:h-K}) ={\mu}_{2, h-K+2}(u_{2,0:h-K+1}) = \cdots = {\mu}_{K, h}(u_{K, 0:h-1}). \\
    \text{Moreover, }&u_{1,h-K+1:h} = (1,\underbrace{0, \cdots 0}_\text{K times}),~~ 
    u_{2,h-K+1:h} = (1,\underbrace{0, \cdots 0}_\text{K-1 times}),~~
    \cdots,~~  u_{K, h:h} = (1).
\end{align*}
Therefore, by~\eqref{eq:expected-reward}, at time $h + 1$, arm $1$ 
has the highest expected reward and will be chosen. 
In the case where $k' >0$ (i.e., $\pi_h = k'$), 
we let $h'\coloneqq h-k'$. 
We have that ${\mu}_{1, h'-K+1}(u_{1,0:h'-K}) = \ldots = {\mu}_{K,h}(u_{K, 0:h'-1})$ 
and $s = {s}_{1, h'-K+1}(u_{1,0:h'-K})=\ldots 
= {s}_{K, h'}(u_{K,0:h'-1}) \leq \frac{\gamma^K}{1 - \gamma^K}$.
Then, at time $h+1$, the %
satiation level for the arms will be  
${s}_{k, h+1}(u_{k,0:h}) = \gamma^{k'-k+1} \left(1+ \gamma^{K} s\right)$
for all $k \leq k'$ and 
${s}_{k, h+1}(u_{k,0:h}) = \gamma^{K-k+k'+1} s$ for all $k > k'$. 
Thus, the arm with the lowest %
satiation level 
will be $\pi_{h+1} = k'+1 = \pi_h + 1$,
since 
${s}_{k'+1, h+1}(u_{k'+1,0:h}) < {s}_{1,h+1}(u_{1,0:h})$. 
Consequently, the greedy policy will select arm $\pi_h + 1$ at time $h+1$. 
\end{proof}

\subsection{Proof of Theorem~\ref{thm:greedy-cumulative-reward}}
\label{appendix:greedy-opt-proofs}
\begin{proof}%
First, when $T \leq K$, greedy policy is optimal since its cumulative expected  reward is $Tb$. 
So, we consider the case of $T > K$. 
Assume for contradiction that there exists another policy $\pi_{1:T}^o$ that is optimal and is not greedy, i.e., $\exists t \in [T], \pi^o_t \notin \argmax_{k \in [K]} b - \lambda s^o_{k,t}$ where $s_{k,t}^o$ denotes the %
satiation level of arm $k$ at time $t$ under the policy $\pi_{1:T}^o$. 
We will construct a new policy $\pi_{1:T}^n$ that obtains a higher cumulative expected  reward than $\pi_{1:T}^o$. 
Throughout the proof, we use $s^n_{k,t}$ to denote the %
satiation levels for the new policy. 

We first note two illustrative facts to give the intuition of the proof.  

\emph{Fact 1:} Any policy $\pi_{1:T}^o$ that does not pick the arm with the lowest %
satiation level (i.e., highest expected reward) at the last time step $T$ is not optimal.\\   
\emph{Proof of Fact 1:} %
In this case, the policy $\pi^n_{1:T} = (\pi^o_1, \ldots, \pi^o_{T-1}, \pi_T)$ where $\pi_T \in \argmax_{k \in [K]} {b}-\lambda s^o_{k,T}$ will obtain a higher cumulative expected reward.

\emph{Fact 2:} If a policy $\pi_{1:T}^o$ picks the lowest %
satiation level for the final pull $\pi^o_T$ but does not pick the arm with the lowest %
satiation level at time $T-1$,
we claim that $\pi^n_{1:T} = (\pi^o_1, \ldots, \pi^o_{T-2}, \pi^o_{T}, \pi_{T-1}^o) \neq \pi^o_{1:T}$ obtains a higher cumulative expected reward. \\
\emph{Proof of Fact 2:} 
First, note that $\pi_{T-1}^o \neq \pi^o_{T}$ because otherwise $\pi_{T-1}^o$ is the arm with the lowest satiation level at $T-1$. 
Moreover, at time $T-1$, $\pi^o_T \in \argmin_k s^o_{k, T-1}$ has the smallest  %
satiation, since if not, then there exists another arm $k \neq \pi_T^o$ and $k \neq \pi_{T-1}^o$ that has a smaller %
satiation level than $\pi_T^o$ at time $T-1$. 
In that case, $\pi_T^o$ will not be the arm with the lowest %
satiation at time $T$, which is a contradiction. 
Then, we deduce $s^o_{\pi^o_{T-1}, T-1} > s^o_{\pi_T^o,T-1}$. 
Combining this with $\pi^o_{T-1} \neq \pi^o_T$, we arrive at 
\begin{align*}
    G_T(\pi_{1:T}^n) - G_T(\pi_{1:T}^o) = \lambda (1-\gamma) \left(s^o_{\pi^o_{T-1}, T-1}  - s^o_{\pi_T^o,T-1} \right)> 0. 
\end{align*}

For the general case, given any policy $\pi_{1:T}^o$ that is not a greedy policy,
we construct the new policy $\pi_{1:T}^n$ that has a higher cumulative expected reward through the following procedure:
\begin{enumerate}
    \item Find $t^* \in [T]$ such that for all $t > t^*$, 
    $\pi^o_t \in \argmax_{k \in [K]} {b} - \lambda s^o_{k,t}$
    and 
    $\pi^o_{t^*} \notin \argmax_{k \in [K]} {b} - \lambda s^o_{k,t^*}$.
    Further, we know that $\pi_{t^*+1}^o \in \argmax_{k \in [K]} {b} - \lambda s^o_{k,t^*}$, using the same reasoning as the above example, i.e., otherwise $\pi_{t^*+1}^o \notin \argmax_{k \in [K]} {b} - \lambda s^o_{k,t^*+1}$.  
    To ease the notation, we use $k_1$ to denote $\pi_{t^*}^o$ and $k_2$ to denote  $\pi_{t^*+1}^o$. 
    \item For the new policy, we choose $\pi^n_{1:{t^*+1}} = (\pi^o_{1}, \ldots, \pi^o_{t^*-1}, k_2, k_1)$. 
    Let $A^o_{t_1, t_2}$ denote the set $\{t' : t^*+2 \leq t' \leq t_2, \pi^o_{t'} = \pi^o_{t_1}\}$. 
    $A^o_{t_1, t_2}$ contains a set of time indices in between  $t^*+2$ and $t_2$ 
    when arm $\pi^o_{t_1}$ is played under policy 
    $\pi^o_{1:T}$. 
    We construct the following three sets $T_A := \{t: t^*+2 \leq t \leq T, |A^o_{t^*, t}| < |A^o_{t^*+1, t}|\}$, $T_B := \{t: t^*+2 \leq t \leq T, |A^o_{t^*, t}| > |A^o_{t^*+1, t}|\}$ and 
    $T_C := \{t: t^*+2 \leq t \leq T, |A^o_{t^*, t}| = |A^o_{t^*+1, t}|\}$.
    For time $t \geq t^*+2$, we consider the following three cases: 
    \begin{enumerate}
    \item[Case I.] 
    $T_B = \varnothing$, which
    means that at any time $t$ in between $t^*+2$ and $T$, 
    arm $k_1$ is played %
    more than arm $k_2$ from $t^*+2$ to $t$. 
    In this case, the new policy follows $\pi^n_{t^*+2:T} = \pi^o_{t^*+2:T}$.
    \item[Case II.] 
    $T_A = \varnothing$, which means that at any time $t$ in between $t^*+2$ and $T$, 
    arm $k_2$ is played %
    more than arm $k_1$ from $t^*+2$ to $t$.
    In this case, the new policy satisfies: 
    for all $t \geq t^*+2 $, 
    1) $\pi^n_t = \pi^o_t$ if $\pi^o_t \neq k_1$ and $\pi^o_t \neq k_2$;
    2) $\pi^n_t = k_2$ if $\pi^o_t = k_1$; and 3)
    $\pi^n_t = k_1$ if $\pi^o_t = k_2$. 
    \item[Case III.] 
    $T_A \neq \varnothing$ and $T_B \neq \varnothing$.
    Then, starting from $t^*+2$, if $t \in T_A$,  $\pi^n_t$ follows the new policy construction in Case I, i.e., $\pi^n_t = \pi^o_t$. 
    If $t \in T_B$, $\pi^n_t$ follows the new policy construction in Case II. 
    Finally, for all $t \in T_C$, 
    define $t'_{A,t} = \max_{\substack{t' \in T_A:\\ t' < t}} t'$ and $t'_{B, t} = \max_{\substack{t' \in T_B:\\ t' < t}} t'$. 
    If $t'_{A,t} > t'_{B,t}$, then $\pi^n_t$ follows the new policy construction as Case I. %
    If $t'_{A,t} < t'_{B,t}$, $\pi^n_t$ follows the new policy construction as Case II. We note that $t'_{A,t} \neq t'_{B,t}$ since $T_A \cap T_B = \varnothing$. 
    \end{enumerate}
\end{enumerate}
When $T_A = \varnothing$ and $T_B = \varnothing$, 
we know that $k_1$ and $k_2$ are not played in $\pi^o_{t^*+2:T}$.
In this case, the new policy construction can follow either Case I or Case II.
To complete the proof, we state some facts first:
\begin{itemize}
    \item From $t^*$, the expected rewards collected by the policies $\pi_{1:T}^o$ and $\pi_{1:T}^n$ only differ at times when arm $k_1$ or arm $k_2$ is played.
    \item $\pi_{1:t^*+1}^n$ obtains a higher cumulative expected reward than $\pi_{1:t^*+1}^o$. 
    \item 
    At time $t^*+2$, the new policy follows that 
    $s^n_{k_1,t^*+2} = \gamma + \gamma^2 s^o_{k_1,t^*}$
    and 
    $s^n_{k_2, t^*+2} = \gamma^2 + \gamma^2 s^o_{k_2,t^*}$. 
    On the other hand, the old policy has 
    $s^o_{k_1,t^*+2} = \gamma^2 + \gamma^2 s^o_{k_1,t^*}$ and  $s^o_{k_2,t^*+2} = \gamma + \gamma^2 s^o_{k_2,t^*}$.
\end{itemize}

Let $N_{k_1} := \{t: t^* +2 \leq t \leq T, \pi^o_t = k_1\}$ 
and $N_{k_2} := \{t: t^* +2 \leq t \leq T, \pi^o_t = k_2\}$ 
denote the sets of time steps when $k_1$ and $k_2$ are played in $\pi^o_{1:T}$. 
For a given %
satiation level $x$ at time $t'$ together with the time steps the arm is pulled $N_k$, we have that at time $t \geq t'$, the arm has %
satiation level $g_{N_k}(x, t, t') =\gamma^{t - t'} x 
+ \sum_{N_{k, i} < t} \gamma^{t - N_{k,i}}$ where $ N_{k,i}$ is the $i$-th smallest element in $N_k$. 

In Case I, the difference of the cumulative expected rewards between the two policies satisfies:
\begin{align*}
    &G_T(\pi_{1:T}^n) - G_T(\pi^o_{1:T})  > 
    \sum_{i=1}^{|N_{k_2}|} -\lambda g_{N_{k_2}}(s^n_{k_2, t^*+2}, N_{k_2, i}, t^*+2)
    + \lambda g_{N_{k_2}}(s^o_{k_2, t^*+2}, N_{k_2, i}, t^*+2)\\
    &\qquad \qquad + \sum_{j=1}^{|N_{k_1}|} -\lambda g_{N_{k_1}}(s^n_{k_1, t^*+2}, N_{k_1, j}, t^*+2) 
    + \lambda g_{N_{k_1}}(s^o_{k_1, t^*+2}, N_{k_1, j}, t^*+2)\\
    &= \lambda
    \left(
    s^o_{k_2, t^*+2}
    - s^n_{k_2, t^*+2}
     \right)
    \sum_{i=1}^{|N_{k_2}|} 
    \gamma^{N_{k_2, i} - (t^*+2)} 
    + 
    \lambda
    \left(
    s^o_{k_1, t^*+2}
    - 
    s^n_{k_1, t^*+2}
     \right)
    \sum_{j=1}^{|N_{k_1}|} 
    \gamma^{N_{k_1, j} - (t^*+2)}
    > 0,
\end{align*}
where we have used the fact that 
$s^o_{k_2, t^*+2}
    - s^n_{k_2, t^*+2} = -\left(s^o_{k_1, t^*+2}
    - 
    s^n_{k_1, t^*+2} \right) > 0$, $|N_{k_2}| \geq |N_{k_1}|$
and for all $j \in [|N_{k_1}|]$, $N_{k_2, j} < N_{k_1, j}$. 
In Case II, similarly, we have that
\begin{align*}
    &G_T(\pi_{1:T}^n) - G_T(\pi_{1:T}^o) 
    > 
    \sum_{j=1}^{|N_{k_1}|} 
    -\lambda g_{N_{k_1}}(s^n_{k_2, t^*+2}, N_{k_1, j}, t^*+2)
    + \lambda g_{N_{k_1}}(s^o_{k_1, t^*+2}, N_{k_1, j}, t^*+2)\\
    &\qquad \qquad + \sum_{i=1}^{|N_{k_2}|} 
    -\lambda g_{N_{k_2}}(s^n_{k_1, t^*+2}, N_{k_2, i}, t^*+2) 
    + \lambda 
    g_{N_{k_2}}(s^o_{k_2, t^*+2}, N_{k_2, i}, t^*+2)\\
     &= \lambda 
     \left(
     s^o_{k_1, t^* +2}
     - s^n_{k_2, t^* +2} \right) 
    \sum_{j=1}^{|N_{k_1}|} \gamma^{N_{k_1, j} - (t^*+2)}
    +  \lambda
    \left(
    s^o_{k_2, t^*+2}
    - s^n_{k_1, t^*+2}
     \right)
    \sum_{i=1}^{|N_{k_2}|} 
    \gamma^{N_{k_2, i} - (t^*+2)} > 0, 
\end{align*}
since 
$s^o_{k_1, t^* +2}
     - s^n_{k_2, t^* +2} = - \left(s^o_{k_2, t^*+2}
    - s^n_{k_1, t^*+2} \right) > 0$,
$|N_{k_2}| \leq |N_{k_1}|$ and   
for all $i \in [|N_{k_2}|]$, $N_{k_1, i} < N_{k_2, i}$.

Finally, for Case III, the new policy construction is a mix of Case I and Case II. 
We represent the time interval $[t^*+2, T]$ to be $[t^*+2, T] = [t_{i_{1}, s_1}, t_{i_1, e_{1}}] \cup [t_{i_2, s_2}, t_{i_2, e_2}] \cup \cdots \cup [t_{i_M, s_M}, t_{i_M, e_M}]$  
where 
$t^*+2 = t_{i_1, s_1} \leq \ldots \leq t_{i_M, s_M} = T$, 
$\cap_{m=1}^M [t_{i_m, s_m}, t_{i_m, e_m}] = \varnothing$
and 
$M-1$ is the number of new policy construction switches happen in between $t^*+2$ and $T$. 
We say that a new policy construction switch happens at time $t$ if 
the policy construction follows Case I at time $t-1$ but follows Case II at time $t$ or vice versa. 
Each $i_m \neq i_{m-1}$ can take values I or II, representing which policy construction rule is used between the time period $t_{i_m, s_m}$ and $t_{i_m, e_m}$. 
For any time index set $V$, we use the notation $V[t_{i_m, s_m}, t_{i_m, e_m}] \coloneqq \{t \in V: t_{i_m, s_m} \leq t \leq t_{i_m, e_m}\}$. 

We notice that at any switching time $t_{i_m, s_m}$, the number of previous pulls of arm $k_1$ and $k_2$ from time $t_{i_{m-1}, s_{m-1}}$ to $t_{i_{m-1}, e_{m-1}}$ are equivalent, which is denoted by $l_m= |N_{k_1}[t_{i_m, s_m}, t_{i_m, e_m}]| = |N_{k_2}[t_{i_m, s_m}, t_{i_m, e_m}]|$ for all $m < M$. 
From our analysis of Case I and Case II, we know that 
to show that $\pi^n_{1:T}$ obtains a higher cumulative expected reward, 
it suffices to prove:
for all $m < M $ such that 
\begin{align*}
    &s^o_{k_2, t_{i_m, s_m}} - s^n_{k_2, t_{i_m, s_m}} = - \left(s^o_{k_1, t_{i_m, s_m}} -  s^n_{k_1, t_{i_m, s_m}}\right) > 0,\\
    &s^o_{k_1, t_{i_m, s_m}} - s^n_{k_2, t_{i_m, s_m}} = - \left(s^o_{k_2, t_{i_m, s_m}} -  s^n_{k_1, t_{i_m, s_m}}\right) > 0,
\end{align*}
we have  
\begin{align*}
    &s^o_{k_2, t_{i_{m+1}, s_{m+1}}} - s^n_{k_2, t_{i_{m+1}, s_{m+1}}} = - \left(s^o_{k_1, t_{i_{m+1}, s_{m+1}}} -  s^n_{k_1, t_{i_{m+1}, s_{m+1}}}\right) > 0,\\
    &s^o_{k_1, t_{i_{m+1}, s_{m+1}}} - s^n_{k_2, t_{i_{m+1}, s_{m+1}}} = - \left(s^o_{k_2, t_{i_{m+1}, s_{m+1}}} -  s^n_{k_1, t_{i_{m+1}, s_{m+1}}}\right) > 0.
\end{align*}

We will establish these facts in Lemma~\ref{lemma:greedy-opt-auxillary}. 
Finally, we note that the above required conditions are held at time $t_{i_1, s_1} = t^*+2$.  
\end{proof}

\begin{lemma} 
Let $N_k[t_s, t_e]$ denote the set of time steps when arm $k$ is pulled in between (and including) time $t_s$ and $t_e$ under policy $\pi_{1:T}^o$. 
Let $s_{k,t}^o$ and $s_{k,t}^n$ represent the %
satiation level 
of arm $k$ at time $t$ when following the policy $\pi_{1:T}^o$ and $\pi_{1:T}^n$, respectively. 
For two different arms $k_1$ and $k_2$, suppose that at time $t_s$ we have 
\begin{align*}
    &s^o_{k_2, t_s} - s^n_{k_2, t_s} = - \left(s^o_{k_1, t_s} -  s^n_{k_1, t_s}\right) > 0,\\
    &s^o_{k_1, t_s} - s^n_{k_2,t_s} = - \left(s^o_{k_2, t_s} -  s^n_{k_1, t_s}\right) > 0. 
\end{align*} 
Further, suppose that %
from time $t_s$ to $t_e$, %
$\pi^n_{1:T}$ follows either Case I (or Case II) of new policy construction (see proof of Theorem~\ref{thm:greedy-cumulative-reward} for their definitions); 
and at time $t'_s = t_e+1$, 
the new policy construction for $\pi^n_{1:T}$ has switched to Case II (or Case I if Case II is used from $t_s$ to $t_e$). 
Then at time $t_s'$, we have that 
\begin{align*}
    &s^o_{k_2, t_s'} - s^n_{k_2,t_s'} = - \left(s^o_{k_1, t_s'} -  s^n_{k_1, t_s'}\right) > 0, \\
    &s^o_{k_1, t_s'} - s^n_{k_2, t_s'} = - \left(s^o_{k_2, t_s'} -  s^n_{k_1, t_s'}\right) > 0.
\end{align*}
\label{lemma:greedy-opt-auxillary}
\end{lemma}

\begin{proof}[Proof of Lemma~\ref{lemma:greedy-opt-auxillary}]
Following the definition in the proof of Theorem~\ref{thm:greedy-cumulative-reward}, 
given that at time $t_s$, arm $k$ has %
satiation $s$, 
let $g_{N_k[t_s, t_e]}(s, t_s', t_s)$ denote the %
satiation level of arm $k$ at time $t'_s$ after being pulled at the time steps in the set $N_k[t_s, t_e]$.  
Let $N_{k,i}[t_s, t_e]$ be the $i$-th smallest element in the set $N_{k}[t_s, t_e]$. 
From the definition of the new policy construction given in the proof of Theorem~\ref{thm:greedy-cumulative-reward}, 
we also know that  (1) $N := |N_{k_1}[t_s, t_e]| = |N_{k_2}[t_s, t_e]|$; 
(2) if Case I is applied in between $t_s$ and $t_e$, 
we have that for all $i \in [N]$, 
$N_{k_2, i}[t_{s}, t_{e}] < N_{k_1, i}[t_{s}, t_{e}]$;
and (3)
if Case II is applied in between $t_s$ and $t_e$, 
we have that for all $i \in [N]$, 
$N_{k_2, i}[t_{s}, t_{e}] > N_{k_1, i}[t_{s}, t_{e}]$. 

We first consider the setting when Case I new policy construction is applied,  
then at time $t'_{s}$, we can show that 
\allowdisplaybreaks
\begin{align*}
    s^o_{k_1, t_s'} - s^n_{k_2, t_s'} 
    =& g_{N_{k_1}[t_s, t_e]}\left(s^o_{k_1,t_s},  t'_s, t_s\right) 
    - g_{N_{k_2}[t_s, t_e]}\left(s^n_{k_2, t_s},t'_{s}, t_s\right) \\
    =& \gamma^{t'_{s} - t_{s}} \left( s^o_{k_1, t_{s}} - s^n_{k_2, t_{s}}\right) + \sum_{i=1}^{l} \gamma^{t'_{s} - N_{k_1, i}[t_{s}, t_{e}]} 
    - \gamma^{t'_{s} - N_{k_2, i}[t_{s}, t_{e}]} \\
    =& 
    \gamma^{t'_{s} - t_{s}} \left( s^n_{k_1, t_s} - s^o_{k_2, t_s}\right) + \sum_{i=1}^{l} \gamma^{t'_{s} - N_{k_1, i}[t_{s}, t_{e}]} 
    - \gamma^{t'_{s} - N_{k_2, i}[t_{s}, t_{e}]}\\
    =& s^n_{k_1, t'_{s}} - s^o_{k_2, t'_{s}}
    > 0, 
\end{align*}
where the last inequality has used the fact that when we use Case I construction, we have $N_{k_2, i}[t_{s}, t_{e}] < N_{k_1, i}[t_{s}, t_{e}]$. 
Meanwhile, we also have that 
\begin{align*}
    s^o_{k_2, t'_{s}} - s^n_{k_2, t'_{s}} 
    =& g_{N_{k_2}[t_{s}, t_{e}]}\left(s^o_{k_2, t_{s}},  t'_{s}, t_{s}\right) 
    - g_{N_{k_2}[t_{s}, t_{e}]}\left(s^n_{k_2, t_{s}}, t'_{s} , t_{s}\right)  \\
    =& \gamma^{t'_{s} - t_{s}} \left( s^o_{k_2, t_{s}} - s^n_{k_2, t_{s}}\right) = - \gamma^{t'_{s} - t_{s}} \left( s^o_{k_1, t_{s}} - s^n_{k_1, t_{s}}\right) \\
    =&  - \left(s^o_{k_1, t'_{s}} -  s^n_{k_1, t'_{s}}\right) > 0. 
\end{align*}

When Case II new policy construction is applied, 
then at time $t'_{s}$, we get 
\begin{align*}
    s^o_{k_1, t'_{s}} - s^n_{k_2, t'_{s}} 
    =& g_{N_{k_1}[t_{s}, t_{e}]}\left(s^o_{k_1, t_{s}},  t'_{s}, t_{s}\right)
    - g_{N_{k_1}[t_{s}, t_{e}]}\left(s^n_{k_2, t_{s}}, t'_{s} ,t_{s}\right) \\
    =& \gamma^{t'_{s} - t_{s}} \left( s^o_{k_1, t_{s}} - s^n_{k_2, t_{s}}\right)  
    = -  \gamma^{t'_{s} - t_{s}} \left( s^o_{k_2, t_{s}} - s^n_{k_1, t_{s}}\right) \\
    =& - \left(s^o_{k_2, t'_{s}} -  s^n_{k_1, t'_{s}}\right)
    > 0, 
\end{align*}
since $s^o_{k_1, t_s} - s^n_{k_2, t_s} > 0$. 
On the other hand, we have that 
\begin{align*}
    s^o_{k_2, t'_{s}} - s^n_{k_2, t'_{s}} 
    =& g_{N_{k_2}[t_{s}, t_{e}]}\left(s^o_{k_2, t_{s}},  t'_{s}, t_{s}\right) 
    - g_{N_{k_1}[t_{s}, t_{e}]}\left(s^n_{k_2, t_{s}}, t'_{s} , t_{s}\right)  \\
    =& \gamma^{t'_{s} - t_{s}} \left( s^o_{k_2, t_{s}} - s^n_{k_2,t_{s}}\right) 
    + \sum_{i=1}^{l} \gamma^{t'_{s} - N_{k_2, i}[t_{s}, t_{e}]} 
    - \gamma^{t'_{s} - N_{k_1, i}[t_{s}, t_{e}]} \\
    =& 
    \gamma^{t'_{s} - t_{s}} \left( s^n_{k_1, t_{s}} - s^o_{k_1, t_{s}}\right) 
    + \sum_{i=1}^{l} \gamma^{t'_{s} - N_{k_2, i}[t_{s}, t_{e}]} 
    - \gamma^{t'_{s} - N_{k_1, i}[t_{s}, t_{e}]}\\
    =& s^n_{k_1, t'_{s}} - s^o_{k_1, t'_{s}}
    > 0, 
\end{align*}
where the last inequality is true because  
when Case II new policy construction is applied, we have 
$N_{k_1, i}[t_{s}, t_{e}] < N_{k_2, i}[t_{s}, t_{e}]$.  
\end{proof}

\subsection{Proof of Proposition~\ref{thm:max k-cut}}

\begin{proof}%
If $T \leq K$, a Max K-Cut of $\mathcal{K}_T$ is $\forall k \in [T], P_k = \{k\}$, 
which is the same as an optimal solution to~\eqref{eq:obj}.  
Let $\ind\{\cdot\}$ denote the indicator function. 
When $T > K$, the integer program in~\eqref{eq:obj} is equivalent to
\allowdisplaybreaks
\begin{align*}
    &\max_{\substack{u_{k,t} \in \{0,1\}: \\ \forall t \in [T], \sum_{k} u_{k,t} = 1}}
    \sum_{k=1}^K {b}u_{k,1} +
    \sum_{k =1}^K \sum_{t =2}^T 
    \left({b}u_{k,t} - \lambda  \sum_{i=1}^{t-1} \gamma^{t-i} u_{k,i}  u_{k,t}   \right)\\
    =
    &\max_{\substack{P_1, \ldots, P_K \subseteq [T]: \\ \cup_k P_k = [T], \\\forall k \neq k', P_k \cap P_{k'} = \varnothing}} 
     \sum_{k=1}^K {b}\ind\{1 \in P_k\} +
    \sum_{k =1}^K \sum_{t =2}^T 
    \left({b}\ind\{t \in P_k\} - \lambda  \sum_{i=1}^{t-1} \gamma^{t-i} \ind\{i \in P_k\} \ind\{t \in P_k\}  \right)
    \\
    =
    &\max_{\substack{P_1, \ldots, P_K \subseteq [T]: \\ \cup_k P_k = [T], \\\forall k \neq k', P_k \cap P_{k'} = \varnothing}}
    T{b} - \sum_{k=1}^K  \; \sum_{\substack{t, i \in P_k:\\ i < t}} \lambda \gamma^{t-i}\\
    = 
    &T{b} - \sum_{t=2}^T\sum_{i=1}^{t-1} \lambda  \gamma^{t-i} 
    + \max_{\substack{P_1, \ldots, P_K \subseteq [T]: \\ \cup_k P_k = [T], \\\forall k\neq k', P_k \cap P_{k'} = \varnothing}}
    \sum_{k=1}^{K-1} \sum_{k'=k+1}^K \sum_{\substack{t \in P_k, \\ i \in P_{k'}:\\ i <t}} \lambda \gamma^{t-i},
\end{align*}
where the second equality uses the fact $\sum_{k=1}^K \ind\{t \in P_k\} = 1$ for all $t \in [T]$ 
and the third equality is true because for any $P_1, \ldots P_K$ such that $\forall k \neq k'$, $P_k \cap P_{k'} = \varnothing$ and $\cup_k P_k = [T]$, we have 
\allowdisplaybreaks
$$\text{Total Edge Weights of } \mathcal{K}_T = \sum_{t=2}^T \sum_{i=1}^{t-1} e(t,i) = \sum_{\substack{t, i \in [T]:  i < t, \\
\exists k \in [K], i, t \in P_k}}  e(t, i) \quad +  
 \sum_{\substack{t, i \in [T]:  i < t, \\
\forall k \in [K], i, t \notin P_k}}  e(t, i). $$
\end{proof}

\subsection{Proof of Theorem~\ref{thm:w-lookahead}}

\begin{proof}%
Given $\pi^*_{1:T}$ and $\pi^w_{1:T}$,  
define a set of new policies $\{\tilde{\pi}^i_{1:T}\}_{i=1}^{l-1}$ 
such that for all $i$, $\tilde{\pi}^i_{1:T} = (\pi^w_{1:iw}, \pi^*_{iw+1:T})$. 
Based on this, we have the following decomposition
$$G_T(\pi^*_{1:T}) - G_T(\pi^w_{1:T}) = \underbrace{G_T(\pi^*_{1:T}) - G_T(\tilde{\pi}^1_{1:T})}_{A_0} + \left(\sum_{i=1}^{l-2} 
\underbrace{G_T(\tilde{\pi}^i_{1:T}) - G_T(\tilde{\pi}^{i+1}_{1:T})}_{A_i}\right) 
+ \underbrace{G_T(\tilde{\pi}^{l-1}_{1:T}) -G_T(\pi^w_{1:T})}_{A_{l-1}}.$$

To distinguish the past pull sequences of each arm under different policies, 
we use the following notations:
${\mu}_{k, t}(u_{k,0:t-1}; \pi')$ gives the expected reward of arm $k$
at time $t$ by following pull sequence $\pi'_{1:t-1}$.
By the definition of $\pi^w_{1:T}$, we have that 
\allowdisplaybreaks
\begin{align*}
    A_0 &= \sum_{t=1}^w {\mu}_{\pi^*_t, t}(u_{\pi^*_t,0:t-1}; \pi^*)
     - {\mu}_{\pi^w_t, t}(u_{\pi^w_t,0:t-1};\pi^w)  + 
\sum_{t=w+1}^T {\mu}_{\pi^*_t, t}(u_{\pi^*_t,0:t-1};\pi^*)
- {\mu}_{\pi^*_t, t}(u_{\pi^*_t,0:t-1}; {\tilde{\pi}^1})\\
&\leq \sum_{t=w+1}^T {\mu}_{\pi^*_t, t}(u_{\pi^*_t,0:t-1};\pi^*)
- {\mu}_{\pi^*_t, t}(u_{\pi^*_t,0:t-1}; {\tilde{\pi}^1}),
\end{align*}
where the inequality follows from the fact that $\pi^w_{1:w}$ is optimal for~\eqref{eq:obj} when $T=w$.    
Similarly, we obtain that for all $i \in [l-2]$, 
\begin{align*}
    A_i &= \sum_{t=1}^{iw}
    \underbrace{{\mu}_{\pi^w_t, t}(u_{\pi^w_t,0:t-1};\pi^w)
    -{\mu}_{\pi^w_t, t}(u_{\pi^w_t,0:t-1};\pi^w)
    }_{=0}  + 
    \sum_{t=iw+1}^{(i+1)w}
    \underbrace{{\mu}_{\pi^*_t, t}(u_{\pi^*_t,0:t-1}; \tilde{\pi}^i)
    - {\mu}_{\pi^w_t, t}(u_{\pi^w_t,0:t-1}; {\pi}^w)}_{\leq 0} \\
    &\qquad + \sum_{t=(i+1)w+1}^T
    {\mu}_{\pi^*_t, t}(u_{\pi^*_t,0:t-1}; {\tilde{\pi}^i})
    - {\mu}_{\pi^*_t, t}(u_{\pi^*_t,0:t-1};\tilde{\pi}^{i+1}) \\
    &\leq \sum_{t=(i+1)w+1}^T
    {\mu}_{\pi^*_t, t}(u_{\pi^*_t,0:t-1}; {\tilde{\pi}^i})
    - {\mu}_{\pi^*_t, t}(u_{\pi^*_t,0:t-1};\tilde{\pi}^{i+1}).
\end{align*}
Finally, we have 
$A_{l-1} = \sum_{t=(l-1)w+1}^T {\mu}_{\pi^*_t, t}(u_{\pi^*_t,0:t-1};{\tilde{\pi}^{l-1}}) - {\mu}_{\pi^w_t, t}(u_{\pi_t^w,0:t-1};{{\pi}^w}) \leq 0$. 
To complete the proof, it suffices to use the fact that for all $i \in \{1, \ldots, l-1\}$, 
\begin{align*}
    \max_{\substack{\pi'_{1:T}, \pi_{1: T}: \\
     \pi'_{iw+1:T} = \pi_{iw+1:T}}} 
     \sum_{t=iw+1}^T 
    {\mu}_{\pi_t, t}(u_{\pi_t,0:t-1};\pi)
    - {\mu}_{\pi_t, t}(u_{\pi_t,0:t-1}; \pi')
    &\leq \sum_{t=0}^{T-iw-1} \overline{\lambda}\overline{\gamma}^t \frac{ \overline{\gamma}%
    }{1 - \overline{\gamma}} 
    \leq \frac{\overline{\lambda} \overline{\gamma}%
    (1 - \overline{\gamma}^{T-iw})}{ (1 - \overline{\gamma})^2} \\
    &\leq \frac{\overline{\lambda} \overline{\gamma}
    (1 - \overline{\gamma}^{T-w})}{ (1 - \overline{\gamma})^2} ,
\end{align*}
where the first inequality holds because for any arm,
the maximum satiation level discrepancy 
under two pull sequences (after $iw$ time steps) is $\overline{\gamma}%
/(1-\overline{\gamma})$ 
and from time $iw+1$ till time $T$, 
the objective will be maximized when the arm with the maximum %
satiation discrepancy is played all the time. 
\end{proof}

\newpage
\section{%
More Discussion on Learning with Unknown Dynamics}
\label{appendix:mdp-related}

As we have noted in Section~\ref{sec:new_stochastic_satiation},
when the learner makes a decision on which arm to pull, 
the learner does not observe the hidden satiation level the user has for the arms.
The POMDP the learner faces 
can be cast as a fully observable MDP (Appendix~\ref{sec:mdp-setup}) where
the estimated reward model
(Appendix~\ref{sec:offline-estiamtion})
can be used
for planning (Appendix~\ref{sec:value-difference}). 
In addition to policies that are 
time-dependent (actions taken by time-dependent policies only depend on the time steps at which they are taken)
considered in Section~\ref{sec:etc},
we also consider state-dependent policies 
where the states are continuous. 

\subsection{MDP Setup}
\label{sec:mdp-setup} 
We begin with describing the full MDP setup
of rebounding bandits, 
including the state representation 
and reward function defined in Section~\ref{sec:state}.
Following~\cite{ortner2012regret}, 
at any time $t \in [T]$,
we define our state vector to be $x_t = (x_{1,t}, n_{1,t}, x_{2,t}, n_{2, t}, \ldots, x_{K,t}, n_{K,t})$,  
where 
$n_{k,t} \in \mathbb{N}$
is the number of steps 
since arm $k$ is last selected
and 
$x_{k,t} $ 
is 
the satiation influence%
as of the most recent pull 
of arm $k$.
Since the most recent pull happens at 
$t - n_{k,t}$,
we have  $x_{k,t} =
b_k - \mu_{k, t - n_{k,t}}
= \lambda_k s_{k, t - n_{k,t}}
$. 
We note that $b_k$ can be obtained when arm $k$
is pulled for the first time 
since the satiation effect is $0$ 
if an arm has not been pulled before. 
The initial state is $x_{\text{init}} = (0, \ldots, 0)$. 
Transitions between two states $x_t$ and $x_{t+1}$ are defined as follows:  
If arm $k$ is chosen at time $t$, i.e., 
$\pi_t = k$, and reward $\mu_{k,t}$ is obtained, 
then the next state $x_{t+1}$ will be: 
\begin{enumerate}[label=A.\arabic*,ref=A.\arabic*]
    \item \label{item:a_1} For the pulled arm $k$, %
    $n_{k, t+1} = 1$
    and $x_{k, t+1} = b_k - \mu_{k,t}$. 
    \item \label{item:a_2} For other arms $k' \neq k$, 
    $n_{k', t+1} = n_{k',t} +1$ if $n_{k', t} \neq 0$ 
    and $n_{k', t+1} = 0$ if $n_{k', t} = 0$. 
    The satiation influence remains the same,
    i.e., $x_{k',t+1} = x_{k', t}$. 
\end{enumerate}
For all $ x_t \in \mathcal{X}$ and $k \in [K]$, 
we have that %
$\mathbb{E}[x_{k,t}] \leq \overline{\lambda} \overline{\gamma}/(1 -\overline{\gamma})$ 
and 
$\text{Var}[x_{k,t}]  \leq \overline{\lambda}^2 \sigma_z^2/(1 - \overline{\gamma}^2)$. Hence,   
for any $\delta \in (0,1)$, 
$ \mathbb{P}\left(\max_{k, t} |x_{k,t}| \geq B(\delta) \right) \leq \delta$,
where
\begin{align}
     B(\delta):= \frac{\overline{\lambda} \overline{\gamma}}{1 - \overline{\gamma}}+ \overline{\lambda}\sigma_z \sqrt{\frac{2 \log (2KT/\delta)}{1 - \overline{\gamma}^2}}. 
\label{eq:x_bound}
\end{align}

The MDP the learner faces can be described as a tuple 
$\mathcal{M} := \langle x_\text{init}, [K], \{\gamma_k, \lambda_k, {b}_k\}_{k=1}^K, T \rangle$
of the initial state $x_\text{init}$, 
actions (arms) $[K]$, 
the horizon $T$ and parameters 
$\{\gamma_k, \lambda_k, {b}_k\}_{k=1}^K$. 
Let $\Delta(\cdot)$ denote the probability simplex.
Given $\{\gamma_k, \lambda_k, {b}_k\}_{k=1}^K$, 
the expected reward $r: \mathcal{X} \times [K] \to \mathbb{R}$ 
and transition functions $p: \mathcal{X} \times [K] \times [T] \to \Delta(\mathcal{X})$ 
are defined as follows:
\begin{enumerate}
    \item ${r}:\mathcal{X} \times [K] \to \mathbb{R}$
    gives the \emph{expected} reward of pulling arm $k$ conditioned on $x_t$, 
    i.e., $r(x_t, k) = \mathbb{E}[\mu_{k,t} | x_t]$.\footnote{By conditioning on $x_t$, 
    we mean conditioning on the $\sigma$-algebra generated by past actions and observed rewards.} 
    If $n_{k, t} =0$, 
    then ${r}(x_t, k) = b_{k}$.
    If $n_{k,t} \geq 1$, 
    ${r}(x_t, k) = b_{k} - \gamma_k^{n_{k,t}} x_{k, t} - \lambda_k \gamma_k^{n_{k,t}}$.  
    \item When pulling arm $k$ at time $t$ and state $x_t$, 
    $p\left(x_{t+1}|x_t, k, t \right) = 0$ 
    if $x_{t+1}$ does not satisfy~\ref{item:a_1} or~\ref{item:a_2}. 
    When $x_{t+1}$ fulfills both~\ref{item:a_1} and~\ref{item:a_2}, 
    we consider two cases of $x_t$. 
    If %
    $n_{k,t} \neq 0$, %
    then the transition function $p\left(x_{t+1}|x_t, k, t \right)$ 
    is given by the Gaussian density with mean 
    $\gamma_k^{n_{k,t}} (x_{k, t} + \lambda_k)$
    and variance $\lambda_k^2 \sigma_z^2 \sum_{i=0}^{n_{k,t}-1} \gamma_k^{2i}$,
    as illustrated in~\eqref{eq:influence-dynamics}. 
    If $n_{k,t} = 0$, %
    then $p(x_{t+1}|x_t, k, t) = 1$ 
    since for the first pull of arm $k$, the obtained reward $\mu_{k,t} = b_k$. 
\end{enumerate}
At time $t$, the learner follows
an action $\pi_t:\mathcal{X} \to [K]$ 
that depends on the state. 
We use $V_{t,\mathcal{M}}^{\pi}: \mathcal{X} \to \mathbb{R}$ 
to denote the value function of policy $\pi_{1:T}$
at time $t$ under MDP $\mathcal{M}$:   
$V_{t,\mathcal{M}}^{\pi}(x_t)
= {r}(x_t, \pi_t(x_t))%
+ \mathbb{E}_{x_{t+1} \sim p(\cdot | x_t, \pi_t(x_t), t)}[V_{t+1, \mathcal{M}}^{\pi}(x_{t+1})]$ %
and $V_{T+1, \mathcal{M}}^\pi(x) = 0$ for all $x \in \mathcal{X}$. 
To restate our goal~\eqref{eq:obj_G} 
in terms of the value function: %
for an MDP $\mathcal{M}$, 
we would like to find a policy $\pi_{1:T}$ 
that maximizes
\begin{align*}
    V^{\pi}_{1, \mathcal{M}}({x}_\text{init}) = \mathbb{E}\left[
     \sum_{t=1}^T {r}(x_t, \pi_t(x_t)) \Bigg| x_1  = {x}_\text{init}\right]. %
\end{align*}
To simplify the notation, we %
use $\pi$ to refer to a policy $\pi_{1:T}$. 
Given an MDP $\mathcal{M}$, 
we denote its optimal policy by $\pi^*_{\mathcal{M}}$ 
and the value function for the optimal policy 
by $V^*_{t, \mathcal{M}}$, i.e., $V^*_{t, \mathcal{M}}({x}):= V^{\pi^*_\mathcal{M}}_{t, \mathcal{M}}({x})$. 

\subsection{Exploration and Estimation of the Reward Model}
\label{sec:offline-estiamtion}

As we have discussed in \S~\ref{sec:exploration},
based on our satiation and reward models, 
the satiation influence $x_{k,t}$ of arm $k$   
forms a dynamical system 
where we only observe the value of the system when arm $k$ is pulled. %
When arm $k$ is pulled at time $t$ %
and $n_{k,t} \neq 0$,  
we observe the satiation influence 
$\lambda_k s_{k,t}$ which becomes
the next state $x_{k,t+1}$, i.e., 
\begin{align}
    x_{k,t+1} &= \lambda_{k} s_{k, t} = 
    \lambda_{k} \gamma_k^{n_{k,t}} s_{k, t -n_{k,t}} +  \lambda_{k}\gamma_k^{n_{k,t}} +  \lambda_{k}\sum_{i=0}^{n_{k,t} - 1} \gamma_k^{i}  z_{k, t-1-i} \nonumber \\
    &= \gamma_k^{n_{k,t}} x_{k, t+1-n_{k,t}} + \lambda_k \gamma_k^{n_{k,t}}  + \lambda_k  \sum_{i=0}^{n_{k,t}-1} \gamma_k^{i} z_{k, t-1-i}.   
\label{eq:influence-dynamics}
\end{align}
We note that %
the current state 
$x_{k,t}$ equals to $x_{k, t+1-n_{k,t}}$  
since 
$x_{k, t+1-n_{k,t}}$ is the last observed satiation influence for arm $k$ and $n_{k,t}$ is the number of steps since arm $k$ is last pulled. 

\paragraph{Exploration Settings}
Depending on the nature of the recommendation domain, 
we consider two types of exploration settings:
one where the users only 
interact with the recommendation 
systems for a short time  
after they log in to the service (Appendix~\ref{sec:multiple-trajectory})
and the other where the users tend to 
interact with the system 
for a much longer time, e.g., 
automated music playlisting (Appendix~\ref{sec:single-trajectory}).
In the first case, 
the learner collects 
multiple ($n$) short trajectories 
of user utilities,
while in the second case,
similar to \S~\ref{sec:system-identification-main-text},
the learner obtains  
a single trajectory of 
user utilities that has length $n$.
In both settings, we obtain that 
under some mild conditions, 
the estimation errors of our estimators for 
$\gamma_k$ and $\lambda_k$ 
are $O(1/\sqrt{n})$. 

\paragraph{Exploration Strategies}
Generalizing from the case where arms are pulled 
repeatedly, 
we explore by pulling the same arm at a fixed interval $m$.
In particular, 
when $m=1$, the exploration strategy is the same as 
repeatedly pulling the same arm for multiple times,
which is the exploration strategy used in \S~\ref{sec:exploration}.
When $m=K$, the exploration strategy 
is to pull the arms in a cyclic order. 
We present the estimator for $\gamma_k, \lambda_k$
using the dataset collected by 
this exploration strategy in both 
the multiple trajectory and single trajectory settings.

\subsubsection{Estimation using Multiple Trajectories}
\label{sec:multiple-trajectory}
For each arm $k \in [K]$, we use $\mathcal{D}^{n,m}_k$ 
to denote a dataset containing 
$n$ trajectories 
of evenly spaced observed satiation influences 
that are collected by our exploration phase.
The time interval between two pulls of an arm 
is denoted by $m$. 
Each trajectory is of length at least $T_\text{min}+1$ 
for $T_\text{min} > 1$. 
For trajectory $i \in [n]$, the observed satiation influences 
are denoted by 
$\tilde{x}^{(i)}_{k,1}, \ldots, \tilde{x}^{(i)}_{k, T_\text{min}+1}, \ldots$, 
where $\tilde{x}^{(i)}_{k,1} = 0$ is the initial satiation influence 
and the rest of the satiation influences 
$\tilde{x}^{(i)}_{k,j}$ ($j > 1$) 
is the difference between the first received reward, i.e.,
the base reward ${b}_k$, 
and the reward from the $j$-th pull of arm $k$.  
In other words, for $\tilde{x}^{(i)}_{k,j}, 
\tilde{x}^{(i)}_{k,j+1} \in \mathcal{D}^{n,m}_{k}$, 
it follows that 
\begin{align}
    \tilde{x}^{(i)}_{k,j+1} = a_k \tilde{x}^{(i)}_{k, j} + d_k + \tilde{z}^{(i)}_{k,j},
\label{eq:multiple-traj-dynamics}
\end{align}
where $a_k = \gamma_k^m$, $d_k = \lambda_k \gamma_k^m$ and 
$\tilde{z}^{(i)}_{k,j}$ are the independent samples 
from $\mathcal{N}\left(0, \sigma_{z,k}^2\right)$ 
with $\sigma_{z,k}^2 = \lambda_k^2 \sigma_z^2 (1-\gamma_k^{2m})/(1-\gamma_k^2)$. 

To estimate $d_k$, we use the estimator 
$\widehat{d}_k = \frac{1}{n} \sum_{i=1}^n \tilde{x}^{(i)}_{k,2} = d_k + 
\frac{1}{n} \sum_{i=1}^n \tilde{z}^{(i)}_{k,1} $. 
By the standard Gaussian tail bound, we obtain that 
for $\delta \in (0,1)$, 
with probability $1-\delta$, 
\begin{align}
    |\widehat{d}_k - d_k| \leq \sqrt{\frac{2 \sigma_{z,k}^2 \log(2/\delta)}{n}} =: \epsilon_{d}(n, \delta, k).
\label{eq:d_error}
\end{align}
When estimating $a_k$, 
we first take the difference between 
the first $\Tmin + 1$ entries of 
two trajectories $i$ and $2i$ 
for $i \in \lfloor n/2 \rfloor$ 
and obtain a new trajectory 
$\tilde{y}^{(i)}_{k,1}, \ldots, \tilde{y}^{(i)}_{k, \Tmin+1}$
where $\tilde{y}^{(i)}_{k,j} = \tilde{x}^{(i)}_{k,j} - \tilde{x}^{(2i)}_{k,j}$ for $j \in [\Tmin + 1]$. 
We note that the new trajectory forms a linear dynamical system 
without the bias term $d_k$, i.e., 
\begin{align*}
    \tilde{y}^{(i)}_{k, j+1} = a_k \tilde{y}^{(i)}_{k,j} + \tilde{w}^{(i)}_{k,j},
\end{align*}
where $\tilde{w}^{(i)}_{k,j}$ are samples from $\mathcal{N}(0, 2 \sigma_{z,k}^2)$. 
We use the ordinary least squares estimator to estimate $a_k$:  
\begin{align}
    \widehat{a}_k &= \argmin_{a} \sum_{i=1}^{\lfloor n/2 \rfloor} \left(\tilde{y}^{(i)}_{k,\Tmin+1} - a\tilde{y}^{(i)}_{k,\Tmin}\right)^2 \nonumber \\
    &= \frac{\sum_{i=1}^{\lfloor n/2 \rfloor} \tilde{y}^{(i)}_{k,\Tmin} \tilde{y}^{(i)}_{k,\Tmin+1}}{\sum_{i=1}^{\lfloor n/2 \rfloor} \left(\tilde{y}^{(i)}_{k,\Tmin}\right)^2}. \label{eq:estimator_multiple_traj_a}
\end{align}

\begin{theorem}\citep[Theorem II.4]{matni2019tutorial}\label{thm:multiple_a_b}
Fix $\delta \in (0,1)$. Given $n \geq 64 \log(2/\delta)$, 
with probability $1-\delta$, we have that 
\allowdisplaybreaks
\begin{align}
    |\widehat{a}_k - a_k| \leq %
    4\sqrt{\frac{2\log(4/\delta)}{n\sum_{t=0}^{\Tmin} a_k^{2t}}} =: \epsilon_{a}(n, \delta, k).
\label{eq:a_error}
\end{align}
\end{theorem}
We notice that as the minimum length of the trajectory
gets greater, 
the upper bound of the estimation error of $a_k$ gets smaller. 
Using our estimators for $a_k$ and $d_k$, 
we estimate $\gamma_k$ and $\lambda_k$ 
through $\widehat{\gamma}_k = |\widehat{a}_k|^{1/m}$
and $\widehat{\lambda}_k = |\widehat{d}_k/\widehat{a}_k|$. 

\begin{corollary} \label{thm:estimation}
Fix $\delta \in (0,1)$. 
Suppose that for all $k \in [K]$, we are given $\mathcal{D}_k^{n,m}$ 
where  $n \geq 64 \log(2/\delta)$ 
and $\widehat{a}_k > 0$ where $\widehat{a}_k$ is defined in~\eqref{eq:estimator_multiple_traj_a}. 
Then, 
with probability $1 - \delta$, we have that for all $k \in [K]$,   
\begin{align*}
    |\widehat{\gamma}_k - \gamma_k| \leq 
    \frac{\epsilon_a(n,\delta/K, k)}{\gamma_k^{m-1}}
    = O\left(\frac{1}{\sqrt{n}} \right)
    &\quad\text{ and }\quad |\widehat{\lambda}_k - \lambda_k| \leq O\left(\frac{1}{\sqrt{n}} \right)
    .
\end{align*}
\end{corollary}

The proof of Corollary~\ref{thm:estimation}
can be found in Appendix~\ref{appendix:multiple-trajectories}.
In the case where we are have collected $n$ trajectories of 
evenly spaced user utilities for each arm, 
when the sample size $n$ is sufficient large, 
the estimation errors of 
$\widehat{\gamma}_k$ and $\widehat{\lambda}_k$ are  
$O(1/\sqrt{n})$.

\subsubsection{Estimation using a Single Trajectory}
\label{sec:single-trajectory}
In the case where the learner 
gets to interact with the user for a long period of time 
(which is the setting considered in \S~\ref{sec:new_stochastic_satiation} and \S~\ref{sec:etc}), 
we collect a single trajectory of evenly spaced arm pulls
for each arm: 
for each arm $k \in [K]$, 
we use $\mathcal{P}^{n,m}_k$ 
to denote a dataset containing a single trajectory of 
$n+1$ observed satiation influences $\tilde{x}_{k,1}, \ldots, \tilde{x}_{k, n+1}$, 
where similar to the multiple trajectories case, 
$\tilde{x}_{k,1} = 0$,
$\tilde{x}_{k,j}$ ($j > 1$) 
is the difference between the first received reward 
and the $j$-th received reward
and the time interval between two consecutive pulls is $m$. 
Thus, for $\tilde{x}_{k,j}, \tilde{x}_{k,j+1} \in \mathcal{P}^{n,m}_{k}$, 
it follows that 
\begin{align}
    \tilde{x}_{k,j+1} = a_k \tilde{x}_{k, j} + d_k + \tilde{z}_{k,j},
\label{eq:affine}
\end{align}
where $a_k, d_k$ and 
$\tilde{z}_{k,j}$ are defined the same as 
the ones in~\eqref{eq:multiple-traj-dynamics}. 
For all $k \in [K]$, given $\mathcal{P}^{n,m}_{k}$, 
we use the following estimators to estimate $A_k = (a_k, d_k)^\top$, 
\begin{align}
    \widehat{A}_k
    =\begin{pmatrix}
           \widehat{a}_k\\
           \widehat{d}_k
    \end{pmatrix}
    = (\mathbf{\overline{X}_k}^\top \mathbf{\overline{X}_k})^{-1} \mathbf{\overline{X}_k}^\top \mathbf{Y_k},
\label{eq:estimator}
\end{align}
where $\mathbf{Y_k} \in \mathbb{R}^n$ is an $n$-dimensional vector 
whose $j$-th entry is $\tilde{x}_{k, j+1}$ 
and $\mathbf{\overline{X}_k} \in \mathbb{R}^{n \times 2}$ 
has its $j$-th row to be the vector $\overline{x}_{k,j} = (\tilde{x}_{k, j}, 1)^\top$.  
Finally, we take $\widehat{\gamma}_k = |\widehat{a}_k|^{1/m}$
and $\widehat{\lambda}_k = |\widehat{d}_k/\widehat{a}_k|$. 
We note that $\widehat{A}_k = \argmin_{A_k \in \mathbb{R}^2}\|\mathbf{Y_k} - \mathbf{\overline{X}_k} A_k\|_2^2$, i.e., it is %
the ordinary least squares estimator for $A_k$ 
given the dataset that treats $\tilde{x}_{k, j+1}$ 
to be the response of the covariates $\overline{x}_{k,j}$.

As we have noted earlier (\S~\ref{sec:system-identification-main-text}), 
unlike the multiple trajectories setting, 
in the single trajectory case, 
the difficulty in analyzing the ordinary least squares estimator~\eqref{eq:estimator} comes from the fact that the samples are not independent. 
Asymptotic guarantees of the ordinary least squares estimators 
in this case  
have been studied previously in control theory and time series community 
\citep{hamilton1994time,ljung1999system}. 
The recent work on system identifications for 
linear dynamical systems 
focuses on studying 
the sample complexity 
of the problem~\citep{simchowitz2018learning,sarkar2019near}. 
Adapting the proof of~\cite[Theorem 2.4]{simchowitz2018learning}, we derive the following theorem for identifying our affine dynamical system~\eqref{eq:affine}. 

\begin{theorem} 
Fix $\delta \in (0, 1)$. %
For all $k \in [K]$, 
there exists a constant  
$n_0(\delta, k)$ 
such that if the dataset $\mathcal{P}_k^{n,m}$ satisfies $n \geq n_0(\delta, k)$,
then 
\allowdisplaybreaks
\begin{align*}
    \mathbb{P}\left(\|\widehat{A}_k - A_k\|_2 \gtrsim \sqrt{1/(\psi n)} \right) \leq \delta, 
\end{align*}
where 
$
   \psi = \sqrt{\min\left\{\frac{\sigma_{z,k}^2(1-a_{k})^2}{{16 d_k^2(1-a_k^2)} + (1-a_{k})^2{\sigma_{z,k}^2}},  \frac{\sigma_{z,k}^2}{4(1-a_k^2)} \right\}}.
$
\label{lemma:a_b_estimation} 
\end{theorem}

As shown in Theorem~\ref{lemma:a_b_estimation}, 
when $d_k = \lambda_k \gamma_k^m $ gets larger, 
the rates of convergence for $\widehat{A}_k$ gets slower.
Given that we have a single trajectory of sufficient length, 
$|\widehat{a}_k - a_k| \leq O(1/\sqrt{n})$
and $|\widehat{d}_k - d_k| \leq O(1/\sqrt{n})$.
Similar to the multiple trajectories case,
as shown in Corollary~\ref{thm:estimation-single}, 
the estimators of $\gamma_k$ and $\lambda_k$ 
also achieve $O(1/\sqrt{n})$ estimation error. 

\begin{corollary} \label{thm:estimation-single}
Fix $\delta \in (0,1)$. 
Suppose that for all $k \in [K]$, %
we have $\mathbb{P}(\|\widehat{A}_k - A_k\|_2 \gtrsim 1/\sqrt{n} ) \leq \delta$ 
and $\widehat{a}_k > 0$ where
$\widehat{A}_k$ and $\widehat{a}_k$ are defined in~\eqref{eq:estimator}. 
Then, 
with probability $1 - \delta$, we have that for all $k \in [K]$,   
\begin{align*}
    |\widehat{\gamma}_k - \gamma_k|
    \leq O\left(\frac{1}{\sqrt{n}} \right)
    &\quad\text{ and }\quad |\widehat{\lambda}_k - \lambda_k| \leq O\left(\frac{1}{\sqrt{n}} \right)
    .
\end{align*}
\end{corollary}

In the next section, we assume that the satiation and reward %
models are estimated using the 
dataset collected by the proposed 
exploration strategies
and 
estimators for  
multiple trajectories or 
a single trajectory of user utilities. 
We will show that performing planning based on these estimated models will %
give us policies that perform well 
for the true MDP.

\subsection{Planning}
\label{sec:value-difference}

For a continuous-state MDP, 
planning can be done through either 
dynamic programming with a 
discretized state space 
or approximate dynamic programming
that uses function approximations. 
In Appendix~\ref{sec:state-dependent}, 
we consider the case where we are given a 
continuous-state MDP planning oracle 
and provide guarantees of the optimal state-dependent 
policy planned under the estimated satiation 
dynamics and reward model. 
Within the state-dependent policies, 
we also consider a set of policies 
that only depend on time (Appendix~\ref{sec:time-dependent-policy}), i.e.,
the time-dependent competitor class defined in
\S~\ref{sec:regrets}. 
In addition to not requiring discretization 
of the state space to solve the planning problem, 
such policies can be deployed to settings where 
user utilities are hard to attain
after the exploration stage. 
We will show that 
using the dataset 
(collected by our exploration strategy in Appendix~\ref{sec:offline-estiamtion})
with sufficient trajectories 
(or a sufficient long trajectory)
to estimate $\{\gamma_k, \lambda_k\}_{k=1}^K$,
the optimal policy
$\pi^*_{\widehat{\mathcal{M}}}$ 
for 
$\widehat{\mathcal{M}} = \langle x_1, [K], \{\widehat{\gamma}_k, \widehat{\lambda}_k, {b}_k\}_{k=1}^K, T \rangle$
also performs well
in the original MDP $\mathcal{M}$. 
We note that $b_k$ is known exactly 
since it is the same as the 
first observed reward for arm $k$, 
as discussed in Appendix~\ref{sec:offline-estiamtion}. 

\subsubsection{Time-dependent Policy}
\label{sec:time-dependent-policy}
We first show that finding the optimal time-dependent policy 
is equivalent to solving the bilinear program~\eqref{eq:obj}. 

\begin{lemma}
Consider a policy $\pi$ that depends only on the time step $t$ but not the state $x_t$, i.e., $\pi$ satisfies $\pi_t = \pi_t(x_t) = \pi_t(x_t')$ for all $t \in [T]$ and $x_t, x_t' \in \mathcal{X}$. 
Then, we have 
\allowdisplaybreaks
\begin{align*}
    V_{1,\mathcal{M}}^{\pi}({x}_\text{init}) 
    = \sum_{t=1}^T {\mu}_{\pi_t, t}(u_{\pi_t, 0:t-1}),
\end{align*}
where $u_{\pi_t, 0:t-1}$ is the corresponding pull sequence of arm $\pi_t$ 
under policy $\pi$ and ${\mu}_{k,t}$ 
is defined in~\eqref{eq:expected-reward}. 
\label{lemma:deterministic-stochastic-transfer}
\end{lemma}

\begin{remark}
We denote the policy 
obtained by solving~\eqref{eq:obj} 
using model parameters in $\mathcal{M}$
by $\pi^T_\mathcal{M}$. 
Because solving~\eqref{eq:obj} 
is equivalent to maximizing $\sum_{t=1}^T {\mu}_{\pi_t,t}(u_{\pi_t, 0:t-1})$, Lemma~\ref{lemma:deterministic-stochastic-transfer} 
suggests that, for MDP $\mathcal{M}$, the best policy $\pi$
that depends only on the time step $t$ but not the exact state $x_t$
(which we refer as time-dependent policies), 
is $\pi^T_\mathcal{M}$. 
\end{remark}

\begin{proposition} \label{proposition:mdp-invariant-policy-gaurantee}
Fix $\delta \in (0,1)$. 
Suppose that for all $k \in [K]$,  
we are given $\mathcal{D}_k^{n,m}$ such that 
$n \geq 64 \log (2/\delta)$ 
and $\widehat{a}_k \in (\underline{a},\overline{a})$
for some $0 < \underline{a} < \overline{a} < 1$ almost surely
where $\widehat{a}_k$ is defined in~\eqref{eq:estimator_multiple_traj_a}. 
Consider a policy $\pi$ that depends on only the time step $t$ but not the state $x_t$. %
Then, with probability $1-\delta$, we have that 
$$
    |V_{1,\mathcal{M}}^{\pi}({x}_\text{init}) - V_{1, \widehat{\mathcal{M}}}^{\pi} ({x}_\text{init}) | \leq 
    O\left(\frac{T}{\sqrt{n}} \right). 
$$
\end{proposition}

\begin{remark}\label{remark:sqrt_T_opt_offline}
Proposition~\ref{proposition:mdp-invariant-policy-gaurantee}  
applies to time-dependent policies. %
Such policies can be constructed from an optimal solution to~\eqref{eq:obj} 
or the $w$-lookahead policy~\eqref{eq:lookahead}. 
From these results, we deduce that when the historical trajectory is of size $n=O(T)$, 
the $\sqrt{T}$-lookahead policy $\pi^w_{\widehat{\mathcal{M}}}$
obtained from solving~\eqref{eq:lookahead} with the parameters from the estimated MDP $\widehat{\mathcal{M}}$ 
will be $O(\sqrt{T})$-separated from 
the optimal time-dependent policy $\pi_\mathcal{M}^T$ 
obtained by solving~\eqref{eq:obj} with the true 
parameters of $\mathcal{M}$. 
That is, 
\begin{align*}
0 &\leq  V^{\pi_\mathcal{M}^T}_{1, \mathcal{M}}({x}_\text{init})
- V^{\pi^w_{\widehat{\mathcal{M}}}}_{1,\mathcal{M}}({x}_\text{init})
= 
V^{\pi_\mathcal{M}^T}_{1, \mathcal{M}}({x}_\text{init})
- V^{\pi_\mathcal{M}^T}_{1, \widehat{\mathcal{M}}}({x}_\text{init})
+
V^{\pi_\mathcal{M}^T}_{1, \widehat{\mathcal{M}}}({x}_\text{init})
- V^{\pi^T_{\widehat{\mathcal{M}}}}_{1,\widehat{\mathcal{M}}}({x}_\text{init}) \\
&\qquad\qquad\qquad +  V^{\pi_{\widehat{\mathcal{M}}}^T}_{1, \widehat{\mathcal{M}}}({x}_\text{init})
- V^{\pi^w_{\widehat{\mathcal{M}}}}_{1,\widehat{\mathcal{M}}}({x}_\text{init})
+
V^{\pi^w_{\widehat{\mathcal{M}}}}_{1,\widehat{\mathcal{M}}}({x}_\text{init})
- V^{\pi^w_{\widehat{\mathcal{M}}}}_{1,\mathcal{M}}({x}_\text{init})\\
&\leq 
|V^{\pi_\mathcal{M}^T}_{1, \mathcal{M}}({x}_\text{init})
- V^{\pi_\mathcal{M}^T}_{1, \widehat{\mathcal{M}}}({x}_\text{init})|
+
|V^{\pi_{\widehat{\mathcal{M}}}^T}_{1, \widehat{\mathcal{M}}}({x}_\text{init})
- V^{\pi^w_{\widehat{\mathcal{M}}}}_{1,\widehat{\mathcal{M}}}({x}_\text{init})|
+
|V^{\pi^w_{\widehat{\mathcal{M}}}}_{1,\widehat{\mathcal{M}}}({x}_\text{init})
- V^{\pi^w_{\widehat{\mathcal{M}}}}_{1,\mathcal{M}}({x}_\text{init})|\\
&\leq O(\sqrt{T}), 
\end{align*}
where %
the second inequality follows from the fact that $V^{\pi_\mathcal{M}^T}_{1, \widehat{\mathcal{M}}}({x}_\text{init})
- V^{\pi^T_{\widehat{\mathcal{M}}}}_{1,\widehat{\mathcal{M}}}({x}_\text{init}) \leq0$ (since 
for the MDP $\widehat{\mathcal{M}}$, $\pi^T_{\widehat{\mathcal{M}}}$
is the optimal time-dependent policy), 
and the third (last) inequality is derived by applying
Proposition~\ref{proposition:mdp-invariant-policy-gaurantee} twice and using Remark~\ref{rem:w-lookahead}. 
\end{remark}

\subsubsection{State-dependent Policy}
\label{sec:state-dependent}

In Proposition~\ref{proposition:mdp-gaurantee},
we show that the difference between 
the value of the optimal state-dependent policy $\pi_{{\mathcal{M}}}^*$,
and the value 
of the optimal state-dependent policy $\pi_{\widehat{\mathcal{M}}}^*$
planned under the estimated $\widehat{\mathcal{M}}$ 
is of order $O(T^2/\sqrt{n})$ where $n$ 
is the number of historical trajectories 
if we use multiple trajectories to estimate $\gamma_k$ and $\lambda_k$.

\begin{proposition}\label{proposition:mdp-gaurantee}
Fix $\delta \in (0,1)$. 
Suppose that for all $k \in [K]$,  
we are given $\mathcal{D}_k^{n,m}$ such that 
$n \geq 64 \log (2/\delta)$ 
and $\widehat{a}_k \in (\underline{a},\overline{a})$ 
for some $0 < \underline{a} < \overline{a} < 1$ almost surely
where $\widehat{a}_k$ is defined in~\eqref{eq:estimator_multiple_traj_a}. 
Then, with probability $1-\delta$, 
$$
    |V_{1,\mathcal{M}}^{*}({x}_\text{init}) - V_{1, \mathcal{M}}^{\pi_{\widehat{\mathcal{M}}}^*} ({x}_\text{init}) | \leq 
    O\left(\frac{T^2}{\sqrt{n}} \right).  
$$
\end{proposition}

\begin{remark}
The assumptions in Proposition~\ref{proposition:mdp-invariant-policy-gaurantee}
and~\ref{proposition:mdp-gaurantee}
correspond to the case where we use multiple trajectories
to estimate the satiation dynamics and reward model.
They can be replaced by conditions on 
single trajectory datasets %
when one uses a single trajectory to estimate the parameters.
\end{remark}

In summary, as Proposition~\ref{proposition:mdp-gaurantee} suggests,
when given a continuous-state MDP planning oracle,
our algorithm obtain a policy $\pi^*_{\widehat{\mathcal{M}}}$ that 
is $O(T^2/\sqrt{n})$ away from the optimal policy $\pi^*_{\mathcal{M}}$ under the true MDP $\mathcal{M}$
where the size of the exploration stage for our algorithm (EEP) is $O(Kn)$
and the horizon of the exploitation/planning stage is $T$.
We also note that the optimal state-dependent policy $\pi^*_{\mathcal{M}}$ 
is the optimal competitor policy 
when the competitor class (\S~\ref{sec:regrets}) contains 
all measurable functions from $\mathcal{X}$ to $[K]$.

\newpage
\section{Proofs of Section~\ref{sec:system-identification-main-text}
and Appendix~\ref{sec:single-trajectory}}
\label{appendix:etc_related_proofs}

\subsection{Proof of Theorem~\ref{thm:a_b_estimation_m=1} and Theorem~\ref{lemma:a_b_estimation}}
\label{appendix:estimation}

We notice that Theorem~\ref{thm:a_b_estimation_m=1}
is a consequence of Theorem~\ref{lemma:a_b_estimation}
when $m=1$.
More specifically, 
the dataset $\mathcal{P}_k^n$ and the parameter $A_k = (\gamma_k, \lambda_k \gamma_k)^\top$ in Theorem~\ref{thm:a_b_estimation_m=1}
is a special case of the 
dataset  $\mathcal{P}_k^{n,m}$ 
and parameter $A_k = (\gamma_k^m, \lambda_k \gamma^m_k)^\top$
considered in Theorem~\ref{lemma:a_b_estimation} by taking $m=1$. 
Thus, below we directly present the proof of Theorem~\ref{lemma:a_b_estimation}
where we use the notation from Theorem~\ref{lemma:a_b_estimation}
(and Appendix~\ref{sec:single-trajectory}), i.e., 
$a_k = \gamma_k^m$
and $d_k = \lambda_k\gamma_k^m$.

We begin with presenting some key results  from~\cite{simchowitz2018learning}; 
we utilize these results in establishing the 
sample complexity of our estimator 
for identifying an affine dynamical system in Appendix~\ref{sec:single-trajectory}.

\begin{definition}
\citep[Definition 2.1]{simchowitz2018learning}
Let $\{\phi_t\}_{t \geq 1}$ be an $\{\mathcal{F}_t\}_{t \geq 1}$-adapted random process taking values in $\mathbb{R}$. 
We say $(\phi_t)_{t \geq 1}$ satisfies the $(k, \nu, p)$-block martingale small-ball (BMSB) condition if, for any $j \geq 0$, one has $\frac{1}{k} \sum_{i=1}^k \mathbb{P}(|\phi_{j+i}| \geq \nu | \mathcal{F}_j) \geq p$ almost surely. 
Given a process $(X_t)_{t \geq 1}$ taking values in $\mathbb{R}^d$, we say that it satisfies the $(k, \Gamma_{\text{sb}}, p)$-BMSB condition for $\Gamma_{\text{sb}} \succ 0$ if for any fixed $w$ in the unit sphere of $\mathbb{R}^d$, the process $\phi_t := \langle w, X_t \rangle$ satisfies $(k, \sqrt{w^\top \Gamma_{\text{sb}} w}, p)$-BMSB.
\label{defn:bmsb}
\end{definition}

\begin{proposition}
\citep[Proposition 2.5]{simchowitz2018learning}
Fix a unit vector $w \in \mathbb{R}^d$,
define $\phi_t = w^\top X_t $. 
If the scalar process $\{\phi_t\}_{t \geq 1}$ 
satisfies the $(l, \sqrt{w^\top \Gamma_\text{sb}w}, p)$-BMSB condition for some $\Gamma_\text{sb} \in \mathbb{R}^{d \times d}$,
then 
\begin{align*}
    \mathbb{P}\left(\sum_{t=1}^n \phi_t^2 \leq \frac{{w^\top \Gamma_\text{sb}w} p^2}{8} l \lfloor T/l \rfloor \right)
    \leq \exp{\left(-\frac{\lfloor T/l \rfloor p^2}{8}\right)}.
\end{align*}
\label{prop:lwm-key-prop}
\end{proposition}

\begin{theorem}%
\citep[Theorem 2.4]{simchowitz2018learning}
Fix $\delta \in (0,1)$, $T \in \mathbb{N}$ and $0 \prec \Gamma_\text{sb} \preceq \overline{\Gamma}$. 
Then if $(X_t, Y_t)_{t \geq 1} \in (\mathbb{R}^d \times \mathbb{R}^n)^n$ is a random sequence such that (a) $Y_t= A X_t + \eta_t$, where 
$\mathcal{F}_t = \sigma(\eta_1, \ldots, \eta_t)$ and 
$\eta_t|\mathcal{F}_{t-1}$ is $\sigma^2$-sub-Gaussian and mean zero, (b) $X_1, \ldots, X_T$ satisfies the $(l, \Gamma_\text{sb}, p)$-BMSB condition, and (c) $\mathbb{P}(\sum_{t=1}^n X_t X_t^\top \npreceq T \overline{\Gamma}) \geq \delta$. 
Then if 
\begin{align*}
    T \geq \frac{10 l}{p^2} \left( \log\left(1/\delta\right)
    + 2d \log(10/p) + \log \det(\overline{\Gamma} \Gamma_\text{sb}^{-1})\right),
\end{align*}
we have that for $\widehat{A} = \argmin_{A \in \mathbb{R}^{n \times d}} \sum_{t=1}^T \|Y_t -  A X_t\|_2^2$, 
\begin{align*}
    \mathbb{P} \left( \|\widehat{A} - A\|_\text{op} > \frac{90 \sigma}{p} 
    \sqrt{\frac{n + d \log (10/p) + \log \det \left( \overline{\Gamma} \Gamma_\text{sb}^{-1} \right) + \log(1/\delta)}{T \lambda_{\min}(\Gamma_{\text{sb}})}}\right) \leq 3 \delta.
\end{align*}
\label{thm:max}
\end{theorem}

We note that in the proof of Theorem~\ref{thm:max} in~\cite{simchowitz2018learning}, 
condition (b) is used through applying Proposition~\ref{prop:lwm-key-prop} 
to ensure that for any unit vector $w \in \mathbb{R}^d$, 
\begin{align}
    \mathbb{P}\left(\sum_{t=1}^T \langle w, X_t \rangle^2 \leq \frac{(w^\top \Gamma_\text{sb} w) p^2}{8} l \lfloor T/l \rfloor \right)
    \leq \exp{\left(-\frac{\lfloor T/l \rfloor p^2}{8}\right)}.
\label{eq:condition-b-surrogate}
\end{align}
To apply Theorem~\ref{thm:max} in our setting 
to obtain Theorem~\ref{lemma:a_b_estimation}, 
we verify condition $(a)$ and $(c)$. 
For condition $(b)$, 
we show a result similar to~\eqref{eq:condition-b-surrogate}.
The below technical lemmas are used in our 
proof of Theorem~\ref{lemma:a_b_estimation}. 

\begin{lemma}
Let $a, b$ be scalars with $b > 0$.
Suppose that $X \sim N(a, b)$.
Then for any $\theta \in [0, 1]$,
\begin{align*}
    \mathbb{P}( \abs{X} \geq \sqrt{\theta (a^2 + b)} ) \geq \frac{(1-\theta)^2}{9}  .
\end{align*}
\label{prop:paley_zygmund}
\end{lemma}

\begin{proof}
By the Paley-Zygmund inequality,
\begin{align*}
   \mathbb{P}( \abs{X} \geq \sqrt{\theta \mathbb{E}[X^2]}) =  \Pr( X^2 \geq \theta \mathbb{E}[X^2] ) \geq (1-\theta)^2 \frac{\mathbb{E}[ X^2 ]^2}{\mathbb{E}[ X^4 ]}.
\end{align*}
Using the mean and variance of non-central chi-squared distributions, we obtain that 
\begin{align*}
    \mathbb{E}[X^2] &= a^2 + b, \\
    \mathbb{E}[X^4] &= a^4 + 6 a^2 b + 3 b^2 = (a^2 + 3b)^2 - 6b^2.
\end{align*}
Plugging them back to the Paley-Zygmund inequality, we have that 
\begin{align*}
   \mathbb{P}( \abs{X} \geq \sqrt{\theta (a^2 + b)}  \geq \frac{(1-\theta)^2}{9},
\end{align*}
where the last inequality uses 
the fact that $\mathbb{E}[X^4] \leq (a^2 + 3b)^2 \leq 9 (a^2 + b)^2 = 9 \mathbb{E}[X^2]^2$. 
\end{proof}

\allowdisplaybreaks
\begin{lemma}
Let $\{\phi_t\}_{t \geq 1}$ be a scalar process 
satisfying that 
\begin{align*}
    \frac{1}{l}
    \sum_{i=1}^l \mathbb{P}(|\phi_{t+i}| \geq \nu_t | \mathcal{F}_t) \geq p,
\end{align*}
for $\nu_t$ depending on $\mathcal{F}_t$.
If $\mathbb{P}(\min_{t} \nu_t \geq \nu) \geq 1 - \delta$
for $\nu > 0$ that depends on $\delta$, 
then 
\begin{align*}
    \mathbb{P}\left(\sum_{t=1}^T \phi_t^2 \leq \frac{\nu^2 p^2}{8} l \lfloor T/l \rfloor \right)
    \leq \exp{\left(-\frac{3\lfloor T/l \rfloor p}{4}\right)} + \delta.
\end{align*}
\label{lemma:varying-nu}
\end{lemma}

\begin{proof}
We begin with partitioning $Z_1, \ldots, Z_T$
into $S:= \lfloor T/l \rfloor$ blocks of size $l$. 
Consider the random variables 
\begin{align*}
    B_{j} = \ind\left(\sum_{i=1}^l \phi^2_{jl+i} \geq \frac{\nu_{jl}^2 pk}{2} \right),
    \; \text{for } 0 \leq j \leq S-1.
\end{align*}
We observe that 
\begin{align*}
    \Prob\left(\sum_{t=1}^T \phi_t^2 \leq \frac{\nu^2 p^2}{8} l \lfloor T/l \rfloor \right)
     &= \Prob\left( \left\{\sum_{t=1}^T \phi_t^2 \leq \frac{\nu^2 p^2}{8} l \lfloor T/l \rfloor\right\} \cap \{ \min_t \nu_t \geq \nu \} \right) \\
     &\qquad +  \Prob\left( \left\{\sum_{t=1}^T \phi_t^2 \leq \frac{\nu^2 p^2}{8} l \lfloor T/l \rfloor\right\} \cap \{ \min_t \nu_t < \nu \} \right) \\
     &\leq  \Prob\left( \left\{\sum_{t=1}^T \phi_t^2 \leq \frac{\nu_{\lfloor t/l \rfloor l}^2 p^2}{8} l S \right\} \cap \{ \min_t \nu_t \geq \nu \} \right) + \Prob( \min_t \nu_t < \nu )  \\
     &\leq \mathbb{P}\left(\sum_{t=1}^T \phi_t^2 \leq \frac{\nu_{\lfloor t/l \rfloor l}^2 p^2}{8} k S \right) + \delta. 
\end{align*}

Using Chernoff bound, we obtain that 
\begin{align*}
     \mathbb{P}\left(\sum_{t=1}^T \phi_t^2 \leq \frac{\nu_{\lfloor t/l \rfloor l}^2 p^2}{8} k S \right)
     &\leq \Prob\left(\sum_{j=0}^{S-1} \sum_{i=1}^l \phi_{jl+i}^2 \leq \frac{\nu_{jl}^2 p^2}{8} l S  \right)
     = \Prob\left(\sum_{j=0}^{S-1} \sum_{i=1}^l \phi_{jl+i}^2 \leq \frac{\nu_{jl}^2 p^2}{8} l S  \right)\\
     &\leq \Prob \left(\sum_{j=0}^{S-1} B_j \leq \frac{p}{4} S \right)
     \leq \inf_{\lambda \leq 0} e^{-\frac{pS}{4}} 
     \mathbb{E}[e^{\lambda \sum_{j=0}^{S-1} B_j}],
\end{align*}
where the second to the last inequality 
uses the fact that 
$\frac{\nu_{jl}^2 pl}{2} B_j \leq \sum_{i=1}^l \phi_{jl+i}^2 $
Further, we have that 
\begin{align*}
    \mathbb{E}[B_j | \mathcal{F}_{jl}]
    &= \Prob\left(\sum_{i=1}^l \phi^2_{jl+i} \geq \frac{\nu_{jl}^2 pl}{2} \Big| \mathcal{F}_{jl} \right)
    \geq \Prob\left(\frac{1}{l} \sum_{i=1}^l \ind\left\{|\phi_{jl+i}| \geq \nu_{jl} \right\} \geq \frac{p}{2}\Big| \mathcal{F}_{jl}\right)\\
    &\geq \frac{p}{2},
\end{align*}
where 
the first inequality uses the fact that 
$\frac{1}{\nu_{jl}^2} \phi_{jl+i}^2 \geq \ind\{\phi_{jl+i}| \geq \nu_{jl}\}$
and 
the last inequality uses 
the fact that for a random variable $X$ supported on 
$[0,1]$ almost surely such that $\mathbb{E}[X] \geq p$
for some $p \in (0,1)$, then for all $t \in [0,p]$, 
$\Prob\left(X \geq t\right) \geq \frac{p-t}{1-t}$.
This is true because 
\begin{align*}
    \Prob\left(X \geq t \right) = \int_t^1 d\Prob(x)
    \geq \int_t^1 x d\Prob(x) 
    =  \int_0^1 x d\Prob(x) - \int_0^t x d\Prob(x)
    = p - t\left(1 - \Prob\left(X \geq t \right) \right).
\end{align*}
In our case, $\E\left[\frac{1}{l} \sum_{i=1}^l \ind\left\{|\phi_{jl+i}| \geq \nu_{jl} \right\}\Big| \mathcal{F}_{jl} \right]
=\frac{1}{l} \sum_{i=1}^l \Prob\left( |\phi_{jl+i}| \geq \nu_{jl} \Big| \mathcal{F}_{jl}  \right) \geq p$. 
Thus, we obtain that for $\lambda \leq 0$, i.e., $e^\lambda \leq 1$,
\begin{align*}
    \mathbb{E}[e^{\lambda B_j} | \mathcal{F}_{jl}]
    = e^\lambda \Prob\left( B_j = 1 \Big| \mathcal{F}_{jl}\right) + \Prob\left(B_j = 0\right)
    = (e^\lambda - 1)\mathbb{E}[B_j | \mathcal{F}_{jl}] + 1 \leq  (e^\lambda - 1) \frac{p}{2} + 1.
\end{align*}
By law of iterated expectation, we obtain that \begin{align*}
    \mathbb{E}[e^{\lambda \sum_{j=0}^{S-1} B_j}]
    = \mathbb{E}\left[e^{\lambda \sum_{j=0}^{S-2} B_j}\mathbb{E}[e^{\lambda  B_j}|\mathcal{F}_{(S-1)k}]\right]
    \leq \left((e^\lambda - 1) \frac{p}{2} + 1 \right)
    \mathbb{E}\left[ e^{\lambda \sum_{j=0}^{S-2} B_j}\right]
    \leq \left((e^\lambda - 1) \frac{p}{2} + 1 \right)^{S}.
\end{align*}
Finally, we need to find 
\begin{align*}
    \inf_{\lambda \leq 0} e^{-pS/4} \left((e^\lambda - 1) \frac{p}{2} + 1 \right)^{S}.
\end{align*}
We can see that $\lambda^* = -\infty$, which gives that 
\begin{align*}
    \inf_{\lambda \leq 0} e^{-pS/4} \left((e^\lambda - 1) \frac{p}{2} + 1 \right)^{S}
    = e^{-pS/4} \left(1 - \frac{p}{2}  \right)^{S}
    \leq e^{-pS/4}  e^{-pS/2}
    = e^{-3 pS/4},
\end{align*}
where we have used the fact that $1+x \leq e^x$ for all real-valued $x$. 
\end{proof}

To apply Theorem~\ref{thm:max}, 
we first recall that the affine dynamical system 
we aim to identify is as follows:
\begin{align*}
    \tilde{x}_{k,j+1} = a_k \tilde{x}_{k, j} + d_k + \tilde{z}_{k,j},
\end{align*}
where $\tilde{x}_{k, 1} =0$, 
$a_k \in (0,1)$ and $\tilde{z}_{k,j} \sim \mathcal{N}(0, \sigma_{z,k}^2)$. 
We define the following quantities  
\begin{align*}
    \Gamma_{k,j} := \sigma_{z,k}^2 \sum_{i=0}^{j-1} a_k^{2i}, 
    \qquad
    d_{k,j} := \sum_{i=0}^{j-1} a_k^j d_k,
\end{align*}
and $\Gamma_{k,\infty} = \sigma_{z,k}^2 \sum_{i=0}^{\infty} a_k^{2i}
= \frac{\sigma_{z,k}^2}{1-a_k^2}$. 
We notice that for all $t \in [T]$, 
$j \geq 1$, 
\begin{align*}
    \tilde{x}_{k,t+j}|\tilde{x}_{k,t}
    \sim \mathcal{N} \left(a_k^j \tilde{x}_{k,t} + d_{k,j}, \Gamma_{k,j}\right).
\end{align*}

\begin{lemma}
Fix $t \geq 0$ and $j \geq 1$.
Recall that $\overline{x}_{k,t} := (\tilde{x}_{k,t}, 1) \in \mathbb{R}^{2}$.
Fix a unit vector $w \in \mathbb{R}^{2}$.
For any $\epsilon \in (0, 1)$, we have 
\begin{align*}
   \Prob\left(\abs{ \langle w, \overline{x}_{k,t+j} \rangle}  \geq \frac{1}{\sqrt{2}}\sqrt{\min\left\{1 - \epsilon, \Gamma_{k,j} - \left(\frac{1}{\epsilon} - 1\right)(a_k^j \tilde{x}_{k,t}  + d_{k,j})^2 \right\}}\right) \geq \frac{1}{36}
\end{align*}
\label{lemma:bmsb_nu_t}
\end{lemma}

\begin{proof}
By Lemma~\ref{prop:paley_zygmund}, 
we have that for any unit vector $w \in \mathbb{R}^2$, 
\begin{align*}
    \Prob\left\{ \abs{ \langle w, \overline{x}_{k,t+j} \rangle}  \geq \frac{1}{\sqrt{2}} \sqrt{ \left( w_1 \left(a_k^j \tilde{x}_{k,t}  + d_{k,j}\right) + w_2\right)^2 +  w^2_1\Gamma_{k,j}} \; \bigg|\; \overline{x}_{k,t} \right\}  \geq \frac{1}{36}.
\end{align*}
For all $\epsilon \in (0,1)$, we have 
\begin{align*}
    (( w_1 (a_k^j \tilde{x}_{k,t} + d_{k,j}) + w_2)^2 + w^2_1 \Gamma_{k,j} 
    &= \left( w_1 \left(a_k^j \tilde{x}_{k,t} + d_{k,j}\right)\right)^2 
    + w_2^2 + 2 w_2 w_1 \left(a_k^j \tilde{x}_{k,t} + d_{k,j}\right) + w^2_1 \Gamma_{k,j} \\
    &\geq (1 - \epsilon) w_2^2 
    - \left(\frac{1}{\epsilon} - 1\right)\left( w_1 \left(a_k^j \tilde{x}_{k,t}+ d_{k,j}\right)\right)^2 
    + w_1^2 \Gamma_{k,j} \\
    &\geq \min\left\{1 - \epsilon, \Gamma_{k,j} - \left(\frac{1}{\epsilon} - 1\right)(a_k^j \tilde{x}_{k,t} + d_{k,j})^2 \right\}.
\end{align*}
\end{proof}

\begin{lemma}
Fix $\delta \in (0,1)$. 
$\{\overline{x}_{k,t}\}_{t=1}^n$ satisfy
that for any unit vector $w \in \mathbb{R}^2$, 
\begin{align*}
    \mathbb{P}\left(\sum_{t=1}^n { \langle w, \overline{x}_{k,t} \rangle}^2 \leq \frac{\psi^2 p^2}{16} j_\star \lfloor n/j_\star \rfloor \right)
    \leq \exp{\left(-\frac{3\lfloor n/j_\star \rfloor p}{4}\right)} + \delta
\end{align*}
with $p = 1/72$, 
\begin{align*}
       j_\star &:= \left\lceil \max\left\{ 
       - \log_{a_k} \left( 1 + (1-a_k)\frac{\sqrt{2 \Gamma_{k,\infty} \log(n/\delta)}}{d_k}\right)
       , -\log_{a_k}\sqrt{2} \right\} \right\rceil, \\
       \psi &:= \sqrt{\min\left\{ \frac{ \Gamma_{k,\infty}}{\frac{16 d_k^2}{(1-a_{k})^2} + \Gamma_{k,\infty}}, \frac{ \Gamma_{k,\infty}}{4} \right\}}.
\end{align*}
\label{lemma:block-size}
\end{lemma}

\begin{proof}
Fix $\delta \in (0,1)$. 
Recall that from Lemma~\ref{lemma:varying-nu}, 
we have shown that 
for all $t \geq 0$ and $k \geq 1$,
given a unit vector $w \in \mathbb{R}^{2}$,
for any $\epsilon \in (0, 1)$, we have
\begin{align*}
     \Prob\left\{ \abs{ \langle w, \overline{x}_{k,t+j} \rangle}  \geq \frac{1}{\sqrt{2}}\sqrt{\min\left\{1 - \epsilon, \Gamma_{k,j} - \left(\frac{1}{\epsilon} - 1\right)(a_k^j \tilde{x}_{k,t} + d_{k,j})^2 \right\}}\right\}  \geq \frac{1}{36}.
\end{align*}
Denote $q_{t,j} = a_k^j \tilde{x}_{k,t}  + d_{k,j}$ 
where $\tilde{x}_{k,t} \sim \mathcal{N}(d_{k,t}, \Gamma_{k,t})$. 
Fix $\delta \in (0,1)$. 
Using the standard Gaussian tail bound and the union bound, 
we have that with probability $1-\delta$, 
\begin{align*}
    \max_{t \in [T]} {q_{t,j}} 
    \leq a_k^j \left( \frac{ d_k}{1-a_k} + \sqrt{2 \Gamma_\infty \log(n/\delta)} \right)
    + \frac{d_k}{1-a_k}.
\end{align*}
When $j \geq j_\star$, $\Gamma_{k,j} \geq \Gamma_{k,\infty}/2$,  
and with probability $1-\delta$, 
$\max_{t \in [T]} q_{t,j} \leq \frac{2 d_k}{1-a_k}$. 
Thus, for $j \geq j_\star$, and 
\begin{align*}
    \epsilon = \frac{\frac{4 d_k^2}{(1-a_k)^2}}{\frac{4 d_k^2}{(1-a_k)^2} + \Gamma_\infty/4 },
\end{align*}
we have 
\begin{align*}
    \nu^2_{t,j} := &\min\left\{ 1-\varepsilon, \Gamma_{k,j} - \left(\frac{1}{\varepsilon} - 1\right) q_{t,j}^2 \right\}\\
    \geq&\min\left\{ 1-\varepsilon, \Gamma_{k,\infty}/2 - \left(\frac{1}{\varepsilon} - 1\right) \frac{4 d_k^2}{(1-a_k)^2} \right\}\\
    \geq& \min\left\{ 
     \frac{\Gamma_{k,\infty}}{\frac{16 d_k^2}{(1-a_k)^2} + \Gamma_{k,\infty} },
     \frac{\Gamma_{k,\infty}}{4}
    \right\} = \psi^2. 
\end{align*}

Putting it altogether, we have 
\begin{align*}
    \frac{1}{2j_\star} \sum_{j=1}^{2j_\star}\Prob\left( \abs{ \langle w, \overline{x}_{k,t+j} \rangle}  \geq \nu_{t,j}/\sqrt{2} | \calF_t \right)
    \geq \frac{1}{2j_\star} \sum_{j=j_\star}^{2j_\star}\Pr(  \abs{ \langle w, \overline{x}_{k,t+j} \rangle}  \geq \nu_{t,j_\star}/\sqrt{2} | \calF_t) \geq \frac{1}{72}.
\end{align*}
Further, we have 
\begin{align*}
    \Prob\left(\min_{t \in [T]} \nu^2_{t, j_\star} \geq \psi^2 \right) \geq 1-\delta. 
\end{align*} 
Applying Lemma~\ref{lemma:varying-nu}, 
we have that for $p = \frac{1}{72}$, 
\begin{align*}
    \mathbb{P}\left(\sum_{t=1}^n { \langle w, \overline{x}_{k,t} \rangle}^2 \leq \frac{\psi^2 p^2}{16} j_\star \lfloor n/j_\star \rfloor \right)
    \leq \exp{\left(-\frac{3\lfloor n/j_\star \rfloor p}{4}\right)} + \delta.
\end{align*}
\end{proof}

\begin{proof}[Proof of Theorem~\ref{lemma:a_b_estimation}] 
Based on our setup, 
condition $(a)$ of Theorem~\ref{thm:max} is satisfied. 
For any $n$, 
using Lemma~\ref{lemma:block-size} with $\delta = \text{exp}(-n)$, 
we have that 
\begin{align*}
    \forall w \in \mathbb{R}^2, 
    \quad
    \mathbb{P}\left(\sum_{t=1}^n { \langle w, \overline{x}_{k,t} \rangle}^2 \leq \frac{\psi^2 p^2}{16} j_\star \lfloor n/j_\star \rfloor \right)
    \leq \exp{\left(-\frac{3\lfloor n/j_\star \rfloor p}{4}\right)}
    + \delta 
    \leq 2 \exp{\left(-\frac{3\lfloor n/j_\star \rfloor p}{4}\right)},
\end{align*}
with $p = 1/72$, 
\begin{align*}
       j_\star &:= \left\lceil \max\left\{ 
       - \log_{a_k} \left( 1 + (1-a_k)\frac{\sqrt{2 \Gamma_{k,\infty} (\log(n) +n)}}{d_k}\right)
       , -\log_{a_k}\sqrt{2} \right\} \right\rceil, \\
       \psi &:= \sqrt{\min\left\{ \frac{ \Gamma_{k,\infty}}{\frac{16 d_k^2}{(1-a_{k})^2} + \Gamma_{k,\infty}}, \frac{ \Gamma_{k,\infty}}{4} \right\}}.
\end{align*}
Thus, we have provided 
a similar result to~\eqref{eq:condition-b-surrogate},
which is what condition (b) of Theorem~\ref{thm:max} is used for. 
In this case, we have $\Gamma_\text{sb} = \psi I$ where $I$
is a $2 \times 2$ identity matrix. 
Finally, to verify condition (c), 
we notice that we have 
\begin{align*}
    {\overline{\Gamma}_{k,j}} := \mathbb{E}[\overline{x}_{k,j} \overline{x}_{k,j}^\top] =  
    \begin{pmatrix}
        \frac{b_k^2(1-a_k^{j-1})^2}{(1-a_k)^2}+ \frac{\sigma_{z,k}^2(1 - a_k^{2j - 2})}{1 - a_k^2} & \frac{(1-a^{j-1}_k)b_k}{1-a_k} \\
         \frac{(1-a^{j-1}_k)b_k}{1-a_k}  & 1
    \end{pmatrix}.
\end{align*} 
and we denote 
\begin{align*}
    \overline{\Gamma} := \overline{\Gamma}_{k,n} + 
    \begin{pmatrix}
    0 & 0 \\
    0 & 1
    \end{pmatrix} + \Gamma_\text{sb}, 
\end{align*}
which gives that 
$0 \prec \Gamma_\text{sb} \prec \overline \Gamma$ and 
for all $j \geq 1$, $0 \preceq \overline{\Gamma}_{k, j} \prec \overline{\Gamma}$. 
Then, we have that  
\begin{align*}
    \mathbb{P}\left( \mathbf{\overline{X}}_k^\top \mathbf{\overline{X}}_k
    \npreceq  \frac{2n}{\delta} \overline{\Gamma}
    \right) 
    &= 
    \mathbb{P} \left( 
    \lambda_{\max}\left(
        (n \overline{\Gamma})^{-1/2} 
        \mathbf{\overline{X}}_k^\top \mathbf{\overline{X}}_k
        (n \overline{\Gamma})^{-1/2} 
    \right) \geq \frac{2}{\delta}
    \right)\\
    &\leq \frac{\delta}{2} \mathbb{E}
    \left[\lambda_{\max}\left(
        (n\overline{\Gamma})^{-1/2} 
        \mathbf{\overline{X}}_k^\top \mathbf{\overline{X}}_k
        (n \overline{\Gamma})^{-1/2} 
    \right) \right]\\
    &\leq \frac{\delta}{2}  \mathbb{E}
    \left[\text{tr}\left(
        (n \overline{\Gamma})^{-1/2} 
        \mathbf{\overline{X}}_k^\top \mathbf{\overline{X}}_k
        (n \overline{\Gamma})^{-1/2} 
    \right) \right] \leq \delta,
\end{align*}
where the last inequality is true since 
$
   \mathbb{E}
    \left[
        \mathbf{\overline{X}}_k^\top \mathbf{\overline{X}}_k
     \right]
    = 
    \sum_{j=1}^n {\Gamma}_{k,j}
    \preceq n \overline{\Gamma}
$
(for all $j \in [n]$, $\text{trace}(\overline{\Gamma} - \overline{\Gamma}_{k,j})>0$
and $\text{det}(\overline{\Gamma} - \overline{\Gamma}_{k,j})>0$). 
Following Theorem~\ref{thm:max}, for $\delta \in (0,1)$, 
when the number of samples satisfy that 
\begin{align*}
    \frac{n}{ j_\star} \geq \frac{10}{p^2} \left( \log\left(1/\delta\right)
    + 4 \log(10/p) + \log \det(\overline{\Gamma} \Gamma_\text{sb}^{-1})\right),
\end{align*}
we have that  
\begin{align*}
    \mathbb{P} \left( \|\widehat{A}_k - A_k\|_2 > \frac{90 \sigma_{z,k}}{p} 
    \sqrt{\frac{1 + 2 \log (10/p) + \log \det \left( \overline{\Gamma} \Gamma_\text{sb}^{-1} \right) + \log(1/\delta)}{n \psi}}\right) \leq 3 \delta.
\end{align*} 
\end{proof}

\subsection{Proof of Corollary~\ref{cor:estimation-single_m=1}
and Corollary~\ref{thm:estimation-single}} 

Similar to Appendix~\ref{appendix:estimation}, 
Corollary~\ref{cor:estimation-single_m=1}
is a special case of Corollary~\ref{thm:estimation-single}
when $m=1$. 
Hence, we directly present the proof of Corollary~\ref{thm:estimation-single} below.

\begin{proof}[Proof of Corollary~\ref{thm:estimation-single}]%
Fix $\delta \in (0,1)$. 
We have that with probability $1-\delta$,  
$
     \epsilon(n, \delta, k):= \|\widehat{A}_k - A_k\|_2 \leq O(1/\sqrt{n}). 
$
With probability at least $1 - \frac{\delta}{K}$, 
$\epsilon_{a_k} := |\widehat{a}_k - a_k| \leq \|\widehat{A}_k - A_k\|_2 = \epsilon(n, \delta/K, k) = O(1/\sqrt{n})$ and $\epsilon_{d_k} := |\widehat{d}_k - b_k|\leq \|\widehat{A}_k - A_k\|_2 = \epsilon(n, \delta/K, k) = O(1/\sqrt{n})$. 
When $m=1$, then $|\widehat{\gamma}_k - \gamma_k| = ||\widehat{a}_k| - a_k| \leq  |\widehat{a}_k - a_k| = \epsilon_{a_k} \leq \epsilon(n, \delta/K, k)$. 
When $m \geq 2$, since $\gamma_k \neq 0$, we have that 
\begin{align*}
    |\widehat{\gamma}_k - \gamma_k| 
    &= \left|\frac{|\widehat{a}_k| - a_k}{|\widehat{a}_k|^{(m-1)/m} + |\widehat{a}_k|^{(m-2)/m} \gamma_k + \ldots + \gamma_k^{m-1}}\right| \leq \frac{|\widehat{a}_k - a_k|}{\gamma_k^{m-1}}. 
\end{align*}
On the other hand,
we obtain that 
\begin{align*}
    |\widehat{\lambda}_k - \lambda_k|
    &= \left| \left| \frac{\widehat{d}_k}{\widehat{a}_k}\right| - \frac{d_k}{a_k} \right|
    \leq 
    \left|\frac{\widehat{d}_k}{\widehat{a}_k} - \frac{{d}_k}{\widehat{a}_k} 
    +\frac{{d}_k}{\widehat{a}_k}- \frac{d_k}{a_k}\right| 
    \leq 
     \frac{\epsilon_{d_k}}{\widehat{a}_k} +
    \frac{\lambda_k \epsilon_{a_k}}{\widehat{a}_k} 
    \leq O\left(\frac{1}{\sqrt{n}}\right). 
\end{align*}
The proof completes as follows:
\begin{align*}
     &\mathbb{P}\left(\forall k \in [K], |\widehat{\gamma}_k - \gamma_k| \leq O(1/\sqrt{n}), |\widehat{\lambda}_k - \lambda_k|  \leq O(1/\sqrt{n}) \right)
    \geq \prod_{k=1}^K \left(1 - \frac{\delta}{K}\right)
    \geq 1-\delta,
\end{align*}
where the last inequality follows from Bernoulli's inequality. 
\end{proof}

\newpage
\section{Additional Proofs and Discussion of Section~\ref{sec:etc}}
\label{appendix:etc-related}

\subsection{Proof of Theorem~\ref{thm:regret-upper}}
\begin{lemma}\label{lemma:eep_obj_transformation}
Consider any episode $i+1$ (from time $t_{i}+1$ to $t_{i+1}$) 
where 
the initial state 
$x^{i} = (\mu_{1,t_{i}+1}(u_{1,0:t_{i}}), n_{1,t_{i}+1}, \ldots,  
\mu_{K,t_{i}+1}(u_{K,0:t_{i}}), n_{K,t_{i}})$
and $\{u_{k, 0:t_{i}}\}_{k=1}^K$ are the past pull sequences 
of the proposed policy $\pi_{1:t_i}$. 
For all $\Tilde{\pi}_{t_{i}+1:t_{i+1}}$
such that $\Tilde{\pi}_t = \Tilde{\pi}_t(x_t) 
= \Tilde{\pi}_t(x_t'), \Tilde{\pi}_t \in [K], 
\forall t \in [t_{i}+1, t_{i+1}], x_t, x_t' \in \mathcal{X}$,
we have that 
\begin{align*}
& \sum_{t=t_{i}+1}^{t_{i+1}} \mathbb{E}_{x_{t_{i}+2},\ldots, x_{t_i}}
\left[  r(x_{t}, \Tilde{\pi}_{t}(x_{t})) | x_{t_{i}+1}  = x^{i} \right] 
= \sum_{t=t_{i}+1}^{t_{i+1}} {\mu}_{k, t}(u_{k, 0:t-1}),
\end{align*}
where $\{u_{k,t_{i}+1:t_{i+1}}\}_{k=1}^K$ 
is the arm pull sequence of $\Tilde{\pi}_{t_{i}+1:t_{i+1}}$.
\end{lemma}

\begin{proof}
Let $k$ denote $\Tilde{\pi}_t$ 
where $t \in \{t_{i}+1, \ldots, t_{i+1}\}$.
Recall that we use 
$u_{k, 0:t-1}$ to denote the pull sequence 
of arm $k$ under policy $\Tilde{\pi}_{1:t_{i+1}}
= (\pi_{1:t_i}, \Tilde{\pi}_{t_{i}+1:t_{i+1}})$.
If $k$ has not been pulled before time $t$ by $\Tilde{\pi}_{1:t_{i+1}}$, 
then 
$ \mathbb{E}_{x_{t_i+2},\ldots, x_{t_{i+1}}}\left[  r(x_{t}, 
\Tilde{\pi}_{t}) | x_{t_i+1} = x^{i}\right] 
= b_{\pi_t} = {\mu}_{\pi_t, t}(u_{\pi_t, 0:t-1})$. 
If $k$ has been pulled before, 
then 
let $q_{1}, \ldots, q_n$ denote the time steps that arm $k$ 
has been pulled before time $t$ by $\Tilde{\pi}_{1:t_{i+1}}$, i.e., 
$u_{k,q_i} = 1$ for $i \in [n]$ 
and $u_{k,t'}=0$ for $t' \notin \{q_1, \ldots, q_n\}$.
We have that for $t \in \{t_{i}+1, \ldots, t_{i+1}\}$, 
\begin{align*}
    &\mathbb{E}_{x_{t_i+2},\ldots, x_{t_{i+1}}}\left[  r(x_{t}, 
    \Tilde{\pi}_{t}) | x_{t_i+1} = x^{i}\right]\\
    =&  b_k - \left( \mathbb{E}_{x_{t_i+2}, \ldots, x_{t_{i+1}-1}}
    \left[\mathbb{E}_{x_{t_{i+1}}} \left[
    \gamma_k^{n_{k,t_{i+1}}} x_{k, t_{i+1}} 
    + \lambda_k \gamma_k^{n_{k,t_{i+1}}}
    \right] | x_{t_i+1} = x^{i}\right] \right)\\ 
    =& b_k - \left(\mathbb{E}_{x_{t_i+2}, \ldots, x_{q_n}}
    \left[\mathbb{E}_{x_{q_n+1}} \left[ \gamma_k^{n_{k,t_{i+1}}} 
    x_{k,q_n+1}
    + \lambda_k \gamma_k^{n_{k,t_{i+1}}}\right] 
    | x_{t_i+1} = x^{i}\right] \right)\\
    =& b_k - \left(\mathbb{E}_{x_{t_i+2}, \ldots, x_{q_n}}
    \left[ \gamma_k^{n_{k,t_{i+1}}} 
    \left( \gamma_k^{n_{k, q_n}} x_{k, q_n} 
    + \lambda_k \gamma_k^{n_{k, q_n}}
    \right) 
    + \lambda_k \gamma_k^{n_{k,t_{i+1}}} 
    | x_{t_i+1} = x^{i}\right] \right)\\
    =& \ldots = b_k - \lambda_k \left(\gamma_k^{n_{k,t_{i+1}}} 
    + \gamma_k^{n_{k,t_{i+1}} + n_{k,q_n}} + \ldots 
    + \gamma_k^{n_{k,t{i+1}} + n_{k,q_n} + \ldots n_{k, q_{1}}}
    \right)\\
    =& {\mu}_{k, t}(u_{k, 0:t-1}),
\end{align*}
where the second equality is true because 
when arm $k$ is not pulled for example at time $t_{i+1}-1$, 
the state for arm $k$ at time $t_{i+1}-1$ will satisfy that  
$x_{k, t_{i+1}} = x_{k,t_{i+1}-1}$ 
and $n_{k,t_{i+1}} = n_{k,t_{i+1}-1}+1$ with probability $1$. 
In this case, we have that 
\begin{align*}
\mathbb{E}_{x_{t_{i+1}}} \left[
\gamma_k^{n_{k,t_{i+1}}} x_{k, t_{i+1}} + \lambda_k \gamma_k^{n_{k,t_{i+1}}}
\big|x_{t_{i+1}-1}\right]
&= 
\gamma_k^{n_{k,t_{i+1}-1} + 1} x_{k, t_{i+1}-1} 
+ \lambda_k \gamma_k^{n_{k,t_{i+1}-1}+1}\\
&= 
\gamma_k^{n_{k,t_{i+1}}} x_{k, t_{i+1}-1} + \lambda_k \gamma_k^{n_{k,t_{i+1}}}.
\end{align*}
The third equality is true since 
when arm $k$ is pulled for example at time $q_n$, 
then we have that
\begin{align*}
&\mathbb{E}_{q_{n+1} \sim p_\mathcal{M}(\cdot| x_{q_n}, k, q_n)} \left[
\gamma_k^{n_{k,t_{i+1}}} x_{k, q_{n}+1} + \lambda_k \gamma_k^{n_{k,t_{i+1}}}
\right] \\
=& \gamma_k^{n_{k,t_{i+1}}} \left( \gamma_k^{n_{k, q_n}} x_{k, q_n} 
+ \lambda_k \gamma_k^{n_{k, q_n}} \right)+ \lambda_k \gamma_k^{n_{k,t_{i+1}}},
\end{align*}
where $p_\mathcal{M}$ is given in Appendix~\ref{sec:mdp-setup}. 
The second to last last equality holds because 
$x_{k,t_i+1} = \mu_{k,t_i+1}(u_{k,0:t_i})$ 
where $\mu_{k, t}(\cdot)$ is defined in~\eqref{eq:expected-reward}.
\end{proof}

\begin{lemma}\label{lemma:simplify_w_lookahead_regret}
For any episode $i+1$ (from time $t_{i}+1$ to $t_{i+1}$), %
given the past arm pull sequences $\{u_{k, 0:t_{i}}\}_{k=1}^K$ 
of the proposed policy $\pi_{1:t_{i}}$, 
the optimal time-dependent competitor policy $\Tilde{\pi}_{t_{i}+1:t_{i+1}}$,
where $\Tilde{\pi}_t = \Tilde{\pi}_t(x_t) 
= \Tilde{\pi}_t(x_t'), \Tilde{\pi}_t \in [K], 
\forall t \in [t_{i}+1, t_{i+1}], x_t, x_t' \in \mathcal{X}$, 
for this episode
is given by  $\texttt{Lookahead}(\{\lambda_k, \gamma_k, b_k\}_{k=1}^K, 
\{u_{k, 0:t_{i}}\}_{k=1}^K,t_{i}, t_{i+1})$
where $\{\lambda_k, \gamma_k, b_k\}_{k=1}^K$ are the true reward parameters 
for the rebounding bandits instance.
\end{lemma}

\begin{proof}
    By Lemma~\ref{lemma:eep_obj_transformation}, 
    we have that the optimal time-dependent competitor policy 
    $\Tilde{\pi}_{t_{i}+1:t_{i+1}}$
    maximizes 
    $
       \sum_{t=t_{i}+1}^{t_{i+1}} {\mu}_{k, t}(u_{k, 0:t-1}),
    $
    by choosing $u_{k, t_{i}+1:t_{i+1}}$. 
    Thus, by the definition of \texttt{Lookahead}~\eqref{eq:lookahead},
    given our proposed policy $\pi_{1:t_i}$,
    the optimal  time-dependent competitor policy 
    is given by $\texttt{Lookahead}(\{\lambda_k, \gamma_k, b_k\}_{k=1}^K, 
    \{u_{k, 0:t_{i}}\}_{k=1}^K,t_{i}, t_{i+1})$.
\end{proof}

\begin{proof}[Proof of Theorem~\ref{thm:regret-upper}]
Using Lemma~\ref{lemma:simplify_w_lookahead_regret}, 
we have that given our policy $\pi_{1:T}$
and its corresponding pull sequence $u_{k,0:t-1}$ for $k \in [K], t \in [T]$, 
the optimal competitor policy for episode $i+1$ where $i \in \{0, \ldots, \lfloor T/w \rfloor \}$
(episode $i+1$ ranges from time $t_{i} +1= iw+1$ to $t_{i+1} = \min\{iw+w, T\}$)
is given by 
$\texttt{Lookahead}(\{\lambda_k, \gamma_k, b_k\}_{k=1}^K, 
\{u_{k, 0:t_{i}}\}_{k=1}^K,t_{i}, t_{i+1})$. 
We use $\mathbf{M}(\{\lambda_k, \gamma_k, b_k\}_{k=1}^K, 
\{u_{k, 0:t_{i}}\}_{k=1}^K,t_{i}, t_{i+1})$ 
to denote the 
(optimal) objective value of~\eqref{eq:lookahead} 
given by $\texttt{Lookahead}(\{\lambda_k, \gamma_k, b_k\}_{k=1}^K, 
\{u_{k, 0:t_{i}}\}_{k=1}^K,t_{i}, t_{i+1})$.
Denote $\overline{b} = \max_k b_k$
and $\underline{b} = \min_k b_k$. 

\paragraph{Exploration Stage}
Recall that in Algorithm~\ref{algo:rebounding_etc}, 
we have defined $\widetilde{T} = T^{2/3} + w - (T^{2/3} \text{ mod } w)$ 
which is a multiple of $w$.
For the first $\widetilde{T}$ time steps, 
as defined in Algorithm~\ref{algo:rebounding_etc}, 
our policy ${\pi}_{1:\widetilde{T}}$ 
is a time-dependent policy, i.e., 
it satisfies that ${\pi}_t = {\pi}_t(x_t) 
= {\pi}_t(x_t'), {\pi}_t \in [K], 
\forall t \in [1, \widetilde{T}], x_t, x_t' \in \mathcal{X}$.
Using~\ref{lemma:eep_obj_transformation},
we obtain that the regret for the first $\widetilde{T}/w$ episodes 
is given by 
\begin{align*}
    & \sum_{i=0}^{\Tilde{T}/w-1} 
    \max_{\Tilde{\pi}_{1:w} \in \mathcal{C}^w}\mathbb{E}\left[ \sum_{j=1}^{w} r(x_{iw+j}, \Tilde{\pi}_{j}(x_{iw+j})) \Big| x_{iw+1} = x^i
    \right]   \\
    &\qquad \qquad - 
    \sum_{i=0}^{\Tilde{T}/w-1} \mathbb{E}\left[  r(x_{iw+j}, {\pi}_{iw+j}(x_{iw+j})) \Big| x_{iw+1} = x^i%
    \right] \\
    \leq& \sum_{i=0}^{\Tilde{T}/w-1} \mathbf{M}(\{\lambda_k, \gamma_k, b_k\}_{k=1}^K, \{u_{k, 0:iw}\}_{k=1}^K,iw, iw+w)
    - \widetilde{T}\left(\underline{b} - \frac{\overline{\lambda} \overline{\gamma}}{1-\overline{\gamma}}\right)\\
    \leq& \widetilde{T} \left(\overline{b} - \underline{b} + \frac{\overline{\lambda} \overline{\gamma}}{1-\overline{\gamma}}\right)
    \lesssim \widetilde{T} \lesssim T^{2/3}.
\end{align*}
since $\widetilde{T} \leq T^{2/3}+ w$ and by assumption, $w \leq T^{2/3}$. 

\paragraph{Estimation Stage}
By Theorem~\ref{thm:a_b_estimation_m=1} and Corollary~\ref{cor:estimation-single_m=1}, 
we have that for any $\delta \in (0,1)$ and $n \geq n_0(\delta, k)$
where $n_0(\delta, k)$ depends on $\delta$ logarithmically, with probability $1-\delta$,
for all $k \in [K]$ 
$|\widehat{\gamma}_k - \gamma_k| \leq \frac{C_{\gamma_k} \log\left({1}/{\delta}\right) }{\sqrt{n}}$
and $|\widehat{\lambda}_k - \lambda_k| \leq \frac{C_{\lambda_k}\log\left({1}/{\delta}\right)}{\sqrt{n}}$ 
when $\widehat{\gamma}_k > 0$.

We define two numbers  
$T_0' := \min_T \{T: (\sum_{k=1}^K n_0(k,T^{-1/3}))^{3/2} = C_1 K (\log T)^{3/2} < T\}$
and 
$T_0'' := \min_T \left\{T :\max_k \gamma_k + \frac{C_{\gamma_k}}{\sqrt{T^{2/3}/K}} 
< 1\right\}$.
These two numbers exist as $T$ can be chosen to be arbitrarily large. 
Take $T_0 = \max\{T_0', T_0''\}$. 
Then for all $T \geq T_0$, 
with probability $1-\delta$ where $\delta = T^{-1/3}$,
we have that 
$\forall k \in [K], |\widehat{\gamma}_k - \gamma_k| 
\leq \epsilon_\gamma = O(\sqrt{K}T^{-1/3}\log T)$,
$|\widehat{\lambda}_k - \lambda_k| 
\leq \epsilon_\lambda = O(\sqrt{K}T^{-1/3}\log T)$
and 
$ \left(\epsilon_{\lambda} \left|\frac{ \widehat{\gamma}_k }{1-\widehat{\gamma}_k} \right|
    +\epsilon_{\gamma}
    \left|\frac{\overline{\lambda} }{(1-\widehat{\gamma}_k)(1-\gamma_k)}\right| \right)
\leq {O}(\sqrt{K}T^{-1/3}\log T)$
since $\widehat{\gamma}_k \leq \gamma_k + \frac{C_{\gamma_k}}{\sqrt{T_0^{2/3}/K}} < 1 $
and $\gamma_k \leq \overline{\gamma} < 1$.

For any pull sequence 
$u_{k, 0:t-1}$, using our obtained estimated parameters 
$\{\widehat{\gamma}_k, \widehat{\lambda}_k, \widehat{b}_k\}_{k=1}^K$,
we define the estimated reward function: for $t \geq 2$,  
$\widehat{{\mu}}_{k,t} (u_{k, 0:t-1}) = 
b_k - \widehat{\lambda}_k \left( \sum_{i=1}^{t-1} 
\widehat{\gamma}_k^{t-i} u_{k,i}\right)$, 
and for $t=1$, 
$\widehat{{\mu}}_{k,1}(u_{k,0:1}) = b_k = {\mu}_{k,1}(u_{k,0:1})$,
where we note that $\widehat{b}_k = b_k$ 
since it is the reward of the first pull of arm $k$.
Given $t \geq 2$, we have that 
\begin{align}
    &|{\mu}_{k,t}(u_{k, 0:t-1})
    -  \widehat{{\mu}}_{k,t}(u_{k, 0:t-1})| \nonumber \\
    =& \left|\widehat{\lambda}_k \left( \sum_{i=1}^{t-1} \widehat{\gamma}_k^{t-i} u_{k,i}\right) 
    - {\lambda}_k \left( \sum_{i=1}^{t-1} {\gamma}_k^{t-i} u_{k,i}\right)\right|\nonumber \\
    =& \left|\widehat{\lambda}_k \left( \sum_{i=1}^{t-1} \widehat{\gamma}_k^{t-i} u_{k,i}\right) 
    - {\lambda}_k \left( \sum_{i=1}^{t-1} \widehat{\gamma}_k^{t-i} u_{k,i}\right)
    + {\lambda}_k \left( \sum_{i=1}^{t-1} \widehat{\gamma}_k^{t-i} u_{k,i}\right)
    - {\lambda}_k \left( \sum_{i=1}^{t-1} {\gamma}_k^{t-i} u_{k,i}\right)\right|\nonumber \\
    \leq& 
    |\widehat{\lambda}_k - \lambda_k| 
    \left|\frac{\widehat{\gamma}_k}{1 - \widehat{\gamma}_k} \right|
    + \overline{\lambda}
    \left|\frac{\widehat{\gamma}_k}{1-\widehat{\gamma}_k} - \frac{\gamma_k}{1-\gamma_k} \right|\nonumber \\
    \leq& 
    \epsilon_{\lambda} \left|\frac{ \widehat{\gamma}_k }{1-\widehat{\gamma}_k} \right|
    +\epsilon_{\gamma}
    \left|\frac{\overline{\lambda} }{(1-\widehat{\gamma}_k)(1-\gamma_k)}\right|.
\label{eq:estimated_reward_difference}
\end{align}

\paragraph{Planning Stage}
Given our policy $\pi_{1:T}$ (along with its pull sequence 
$\{u_{k,0:T}\}_{k=1}^K$),
starting from time $\widetilde{T}+1$,
for any episode $i+1 \geq \widetilde{T}/w$, 
we denote the optimal competitor policy to be 
$\pi^*_{t_{i}+1:t_{i+1}} = 
\texttt{Lookahead}(\{\lambda_k, \gamma_k, b_k\}_{k=1}^K, 
\{u_{k, 0:t_{i-1}}\}_{k=1}^K,t_{i}, t_{i+1})$
where $t_{i} = iw$ and $t_{i+1} = \min\{iw+w, T\}$. 
The cumulative expected reward collected by $\pi_{t_i+1:t_{i+1}}^*$ and $\pi_{t_i+1:t_{i+1}}$
has the difference
\begin{align*}
    &\mathbf{M}(\{\lambda_k, \gamma_k, b_k\}_{k=1}^K, 
    \{u_{k, 0:t_{i-1}}\}_{k=1}^K,t_{i}, t_{i+1})
    - 
    \mathbf{M}(\{\widehat{\lambda}_k, \widehat{\gamma}_k, b_k\}_{k=1}^K, 
    \{u_{k, 0:t_{i-1}}\}_{k=1}^K,t_{i}, t_{i+1})\\
    =& \sum_{t=t_{i}+1}^{t_{i+1}} \mu_{\pi^*_t,t}(u^*_{\pi^*_t, 0:t-1})
    - \sum_{t=t_{i}+1}^{t_{i+1}} \mu_{\pi_t,t}(u_{\pi_t, 0:t-1})\\
    =& \sum_{t=t_{i}+1}^{t_{i+1}} \mu_{\pi^*_t,t}(u^*_{\pi^*_t, 0:t-1})
    - \sum_{t=t_{i}+1}^{t_{i+1}} \widehat{\mu}_{\pi^*_t,t}(u^*_{\pi^*_t, 0:t-1})\\
    &\qquad + \sum_{t=t_{i}+1}^{t_{i+1}} \widehat{\mu}_{\pi^*_t,t}(u^*_{\pi^*_t, 0:t-1})
    - \sum_{t=t_{i}+1}^{t_{i+1}} \widehat{\mu}_{\pi_t,t}(u_{\pi_t, 0:t-1})\\
    &\qquad + \sum_{t=t_{i}+1}^{t_{i+1}} \widehat{\mu}_{\pi_t,t}(u_{\pi_t, 0:t-1})
    - \sum_{t=t_{i}+1}^{t_{i+1}} \mu_{\pi_t,t}(u_{\pi_t, 0:t-1})\\
    \leq&\sum_{t=t_{i}+1}^{t_{i+1}} \mu_{\pi^*_t,t}(u^*_{\pi^*_t, 0:t-1})
    - \sum_{t=t_{i}+1}^{t_{i+1}} \widehat{\mu}_{\pi^*_t,t}(u^*_{\pi^*_t, 0:t-1})\\
    &\qquad + \sum_{t=t_{i}+1}^{t_{i+1}} \widehat{\mu}_{\pi_t,t}(u_{\pi_t, 0:t-1})
    - \sum_{t=t_{i}+1}^{t_{i+1}} \mu_{\pi_t,t}(u_{\pi_t, 0:t-1}).
\end{align*}
where $u^*_{\pi^*_t, 0:t-1}$ is the corresponding pull sequence 
of arm $\pi^*_t$ under policy $\pi^*_{1:t} = (\pi_{1:t_i}, \pi^*_{t_i+1:t})$,
and the last inequality holds because 
$\pi_{t_{i}+1:t_{i+1}} = 
\texttt{Lookahead}(\{\widehat{\lambda}_k, \widehat{\gamma}_k, 
\widehat{b}_k\}_{k=1}^K, \{u_{k, 0:t_{i}}\}_{k=1}^K,t_{i}, t_{i+1})$
is the optimal solution under the estimated parameters  $\{\widehat{\lambda}_k, \widehat{\gamma}_k, 
\widehat{b}_k\}_{k=1}^K$
and $\pi$'s previous past pull sequence $\{u_{k, 0:t_{i}}\}_{k=1}^K$. 
Further, 
using~\eqref{eq:estimated_reward_difference} and the fact that 
$t_i - t_{i-1} \leq w$, we obtain that 
\begin{align*}
    &\sum_{t=t_{i}+1}^{t_{i+1}} \mu_{\pi^*_t,t}(u^*_{\pi^*_t, 0:t-1})
    - \sum_{t=t_{i}+1}^{t_{i+1}} \widehat{\mu}_{\pi^*_t,t}(u^*_{\pi^*_t, 0:t-1})\\
    &\qquad + \sum_{t=t_{i}+1}^{t_{i+1}} \widehat{\mu}_{\pi_t,t}(u_{\pi_t, 0:t-1})
    - \sum_{t=t_{i}+1}^{t_{i+1}} \mu_{\pi_t,t}(u_{\pi_t, 0:t-1})\\
    \leq& 2w \max_k \left(\epsilon_{\lambda} \left|\frac{ \widehat{\gamma}_k }{1-\widehat{\gamma}_k} \right|
    +\epsilon_{\gamma}
    \left|\frac{\overline{\lambda} }{(1-\widehat{\gamma}_k)(1-\gamma_k)}\right| \right).
\end{align*}

Finally, putting it altogether, we have obtained that 
for all $T \geq T_0$, 
\begin{align*}
     \text{Reg}^{w}(T) &= \textstyle \sum_{i=0}^{\lceil T/w \rceil-1} 
    \max_{\Tilde{\pi}_{1:w} \in \mathcal{C}^w}\mathbb{E}\left[ \sum_{j=1}^{\min\{w, T-iw\}} r(x_{iw+j}, \Tilde{\pi}_{j}(x_{iw+j})) \Big| x_{iw+1} = x^i%
    \right] \nonumber \\
    &\qquad \textstyle - 
    \mathbb{E}\left[ \sum_{j=1}^{\min\{w,T-iw\}}  r(x_{iw+j}, {\pi}_{iw+j}(x_{iw+j})) \Big| x_{iw+1} = x^i%
    \right]\\
    \leq& O(T^{2/3}) + (1-T^{-1/3})\left(\sum_{i=T/w}^{\lceil T/w \rceil-1} 
    2w {O}(\sqrt{K}T^{-1/3} \log T)\right)
    + T^{-1/3} \left(T \left(\overline{b} - \underline{b} + \frac{\overline{\lambda}\overline{\gamma}}{1-\overline{\gamma}}\right)\right)\\
    \leq& O(T^{2/3}) + (T - T^{2/3}){O}(\sqrt{K}T^{-1/3} \log T) +O(T^{2/3})\\
    \leq& {O}(\sqrt{K}T^{2/3} \log T),
\end{align*}
which we notice that with probability $\delta = T^{-1/3}$,
the cumulative expected reward from time $\widetilde{T}$ to $T$
between the optimal competitor policy
and our policy $\pi$ is at most $T \left(\overline{b} - \underline{b} + \frac{\overline{\lambda}\overline{\gamma}}{1-\overline{\gamma}}\right)$. 
This completes the proof.
\end{proof}

\subsection{Exploration Strategies}
\label{appendix:other-exploration-strategy}

In the exploration phase of Algorithm~\ref{algo:rebounding_etc} (from 
time $1$ to $\widetilde{T}$),
in addition to playing each arm repeatedly for $\widetilde{T}/K$ times,
in general, 
we could explore by playing each arm at a fixed interval, 
i.e., the time interval between two consecutive pulls of arm $k$
should be a constant $m_k$. 
For example, 
this includes 
playing the arms cyclically with the cylce being $1, 2, \ldots, K$
or playing the first two arms in an alternating fashion from time $1$ to $2\widetilde{T}/K$,
then the next two arms, etc. 
As shown in Theorem~\ref{lemma:a_b_estimation} and Corollary~\ref{thm:estimation-single},
using the datasets (of size $n$) collected by these exploration strategies, 
we can obtain estimators $\widehat{\gamma}_k$ and $\widehat{\lambda}_k$
with the estimation error being on the order of $O(1/\sqrt{n})$.
Using these results 
(in replacement of Theorem~\ref{thm:a_b_estimation_m=1} 
and Corollary~\ref{cor:estimation-single_m=1}
in the estimation stage of the proof of Theorem~\ref{thm:regret-upper}), 
we can obtain that there exists $T_0$ such that 
for all $T \geq T_0$, 
the regret upper bound 
of EEP under these exploration strategies
are of order $O(\sqrt{K} T^{2/3} \log T)$.

\newpage
\section{Additional Proofs of Appendix~\ref{appendix:mdp-related}}
\label{appendix:proof-different-preference}

\subsection{Proof of Corollary~\ref{thm:estimation}}
\label{appendix:multiple-trajectories}

\begin{proof}
Fix $\delta \in (0,1)$. 
By Theorem~\ref{thm:multiple_a_b},
for all $k \in [K]$, with probability $1 - \frac{\delta}{2K}$, we have the following:  
When $m=1$, then $|\widehat{\gamma}_k - \gamma_k| = ||\widehat{a}_k| - a_k| \leq  |\widehat{a}_k - a_k| \leq \epsilon_{a}(n, \frac{\delta}{2K}, k)$. 
When $m \geq 2$, %
we have that 
\begin{align*}
    |\widehat{\gamma}_k - \gamma_k| 
    &= \left|\frac{|\widehat{a}_k| - a_k}{|\widehat{a}_k|^{(m-1)/m} + |\widehat{a}_k|^{(m-2)/m} \gamma_k + \ldots + \gamma_k^{m-1}}\right| \leq \frac{|\widehat{a}_k - a_k|}{\gamma_k^{m-1}}. 
\end{align*}
On the other hand, given that $|\widehat{a}_k - a_k| \leq \epsilon_a(n, \frac{\delta}{2K}, k)$, 
we have that with probability 
$1 - \frac{\delta}{2K}$, 
\begin{align*}
    |\widehat{\lambda}_k - \lambda_k|
    &= \left| \left| \frac{\widehat{d}_k}{\widehat{a}_k}\right| - \frac{d_k}{a_k} \right|
    \leq 
    \left|\frac{\widehat{d}_k}{\widehat{a}_k} - \frac{{d}_k}{\widehat{a}_k} 
    +\frac{{d}_k}{\widehat{a}_k}- \frac{d_k}{a_k}\right| 
    \leq 
     \frac{\epsilon_d(n, \frac{\delta}{2K}, k)}{\widehat{a}_k} +
    \frac{\lambda_k \epsilon_a(n, \frac{\delta}{2K}, k)}{\widehat{a}_k} 
    \leq O\left(\frac{1}{\sqrt{n}}\right). 
\end{align*}
The proof completes as follows:
\begin{align*}
     &\mathbb{P}\left(\forall k , |\widehat{\gamma}_k - \gamma_k| \leq \frac{|\widehat{a}_k - a_k|}{\gamma_k^{m-1}}, |\widehat{\lambda}_k - \lambda_k| \leq  \frac{\epsilon_d(n, \frac{\delta}{2K}, k)}{\widehat{a}_k} +
    \frac{\lambda_k \epsilon_a(n, \frac{\delta}{2K}, k)}{\widehat{a}_k}  \right)
    \geq \prod_{k=1}^K \left(1 - \frac{\delta}{2K}\right)^2
    \geq 1-\delta,
\end{align*}
where the last inequality follows from Bernoulli's inequality. 
\end{proof}

\subsection{Proof of Lemma~\ref{lemma:deterministic-stochastic-transfer}}
\begin{proof}%
Let $\pi_{1:T}$ denote the sequence that policy $\pi$ will take from time $1$ to $T$. 
By the definition of the value function, we have that 
\begin{align*}
    V_{1,\mathcal{M}}^{\pi}({x}_\text{init}) 
    = b_{\pi_1}
    + \sum_{t=2}^T \mathbb{E}_{x_2, \ldots, x_t}\left[ {r}(x_t, \pi_t)\right],
\end{align*}
where $x_t \sim p_\mathcal{M}(\cdot| x_{t-1}, \pi_{t-1}, t-1)$ 
is a state vector drawn from the transition distribution defined in Section~\ref{sec:mdp-setup}. 
Let $k$ denote $\pi_t$ and $u_{k, 0:t-1}$ denote the past pull sequence for arm $k$ under policy $\pi$. 
If $k$ has not been pulled before time $t$, 
then $\mathbb{E}_{x_2, \ldots, x_t}[{r}(x_t, \pi_t)] = b_{\pi_t} = {\mu}_{\pi_t, t}(u_{\pi_t, 0:t-1})$. 
If $k$ has been pulled before, then 
let $t_{1}, \ldots, t_n$ denote the time steps that arm $k$ 
has been pulled before time $t$.  
We have that 
\begin{align*}
    \mathbb{E}_{x_2, \ldots, x_t}\left[ {r}(x_t, k)\right]
    &=  b_k - \left( \mathbb{E}_{x_{2}, \ldots, x_{t-1}}
    \left[\mathbb{E}_{x_t \sim p_\mathcal{M}(\cdot| x_{t-1}, k, t -1)} \left[
    \gamma_k^{n_{k,t}} x_{k, t} + \lambda_k \gamma_k^{n_{k,t}}
    \right] \right] \right)\\ 
    &= b_k - \left(\mathbb{E}_{x_{2}, \ldots, x_{t_n }}
    \left[\mathbb{E}_{x_{t_n+1} \sim p_\mathcal{M}(\cdot| x_{t_n}, k, t_n)} \left[ \gamma_k^{n_{k,t}} x_{k,t_n+1}
    + \lambda_k \gamma_k^{n_{k,t}}\right] \right] \right)\\
    &= b_k - \left(\mathbb{E}_{x_{2}, \ldots, x_{t_n }}
    \left[ \gamma_k^{n_{k,t}} 
    \left( \gamma_k^{n_{k, t_n}} x_{k, t_n} 
    + \lambda_k \gamma_k^{n_{k, t_n}}
    \right) 
    + \lambda_k \gamma_k^{n_{k,t}} \right] \right)\\
    &= \ldots = b_k - \lambda_k \left(\gamma_k^{n_{k,t}} 
    + \gamma_k^{n_{k,t} + n_{k,t_n}} + \ldots 
    + \gamma_k^{n_{k,t} + n_{k,t_n} + \ldots n_{k, t_{1}}}
    \right)\\
    &= {\mu}_{k, t}(u_{k, 0:t-1}),
\end{align*}
where we note that the second equality is true because when arm $k$ is not pulled for example at time $t-1$, 
the state for arm $k$ at time $t-1$ will satisfy that  
$x_{k, t} = x_{k,t-1}$ and $n_{k,t} = n_{k,t-1}+1$ with probability $1$. 
In this case, we have that 
$\mathbb{E}_{x_t \sim p_\mathcal{M}(\cdot| x_{t-1}, k, t -1)} \left[
\gamma_k^{n_{k,t}} x_{k, t} + \lambda_k \gamma_k^{n_{k,t}}
\right]
= 
\gamma_k^{n_{k,t-1} + 1} x_{k, t-1} + \lambda_k \gamma_k^{n_{k,t-1}+1}
= 
\gamma_k^{n_{k,t}} x_{k, t-1} + \lambda_k \gamma_k^{n_{k,t}}
$. 
The third equality is true since 
when arm $k$ is pulled for example at time $t-1$, 
then we have that
$\mathbb{E}_{x_t \sim p_\mathcal{M}(\cdot| x_{t-1}, k, t -1)} \left[
\gamma_k^{n_{k,t}} x_{k, t} + \lambda_k \gamma_k^{n_{k,t}}
\right] = \gamma_k^{n_{k,t}} \left( \gamma_k^{n_{k, t-1}} x_{k, t-1} + \lambda_k \gamma_k^{n_{k, t-1}} \right)+ \lambda_k \gamma_k^{n_{k,t}}$. 
The proof completes by summing over $\mathbb{E}_{x_2, \ldots, x_t}\left[ r(x_t, \pi_t)\right]$ for all $t \geq 2$.
\end{proof}

\subsection{Proof of Proposition~\ref{proposition:mdp-invariant-policy-gaurantee}}
\begin{proof}%
Fix $\delta \in (0,1)$. 
Let $E_1$ be the event that   
\begin{align*}
    \forall k \in [K], 
    ~|\widehat{\gamma}_k - \gamma_k| = \epsilon_{\gamma_k} \leq O\left(\frac{1}{\sqrt{n}}\right), 
    ~|\widehat{\lambda}_k - \lambda_k| = \epsilon_{\lambda_k} \leq O\left(1/\sqrt{n} \right). 
\end{align*}
From Corollary~\ref{thm:estimation}, we have that 
 $\mathbb{P}(E_1) \geq 1-\delta$. 
Let $\pi_{1:T}$ denote the sequence that policy $\pi$ will take from time $1$ to $T$. 
From Lemma~\ref{lemma:deterministic-stochastic-transfer},
we have that 
\begin{align*}
    |V_{1,\mathcal{M}}^{\pi}({x}_\text{init}) - V_{1, \widehat{\mathcal{M}}}^{\pi} ({x}_\text{init}) |
    = \left|\sum_{t=1}^T {\mu}_{\pi_t, t}(u_{\pi_t, 0:t-1})
    -  \widehat{{\mu}}_{\pi_t, t}(u_{\pi_t, 0:t-1}) \right|,
\end{align*}
where $u_{\pi_t, 0:t-1}$ is the past pull sequence for arm $\pi_t$ under policy $\pi$ before time $t$  
and $\widehat{{\mu}}_{k,t} (u_{k, 0:t-1}) = 
b_k - \widehat{\lambda}_k \left( \sum_{i=1}^{t-1} \widehat{\gamma}_k^{t-i} u_{k,i}\right)$ for $t \geq 2$ 
and $\widehat{{\mu}}_{k,1}(u_{k,0:1}) = b_k = {\mu}_{k,1}(u_{k,0:1})$. 
Given $t \geq 2$, let $k$ denote $\pi_t$, we have that 
\begin{align*}
    &|{\mu}_{k,t}(u_{k, 0:t-1})
    -  \widehat{{\mu}}_{k,t}(u_{k, 0:t-1})|\\
    =& \left|\widehat{\lambda}_k \left( \sum_{i=1}^{t-1} \widehat{\gamma}_k^{t-i} u_{k,i}\right) 
    - {\lambda}_k \left( \sum_{i=1}^{t-1} {\gamma}_k^{t-i} u_{k,i}\right)\right|\\
    =& \left|\widehat{\lambda}_k \left( \sum_{i=1}^{t-1} \widehat{\gamma}_k^{t-i} u_{k,i}\right) 
    - {\lambda}_k \left( \sum_{i=1}^{t-1} \widehat{\gamma}_k^{t-i} u_{k,i}\right)
    + {\lambda}_k \left( \sum_{i=1}^{t-1} \widehat{\gamma}_k^{t-i} u_{k,i}\right)
    - {\lambda}_k \left( \sum_{i=1}^{t-1} {\gamma}_k^{t-i} u_{k,i}\right)\right|\\
    \leq& 
    |\widehat{\lambda}_k - \lambda_k| 
    \left|\frac{\widehat{\gamma}_k}{1 - \widehat{\gamma}_k} \right|
    + \overline{\lambda}
    \left|\frac{\widehat{\gamma}_k}{1-\widehat{\gamma}_k} - \frac{\gamma_k}{1-\gamma_k} \right|\\
    \leq& 
    \frac{ \widehat{\gamma}_k \epsilon_{\lambda_k}}{1-\widehat{\gamma}_k} 
    +
    \frac{\overline{\lambda} \epsilon_{\gamma_k}}{(1-\widehat{\gamma}_k)(1-\gamma_k)}
\end{align*}
Since $\widehat{\gamma}_k < 1$ ($\widehat{a}_k \in (\underline{a}, \overline{a})$) almost surely 
and with probability $1-\delta$, for all $k \in [K]$, 
$\epsilon_{\gamma_k} \leq O\left(1/\sqrt{n} \right)$ and
$\epsilon_{\lambda_k} \leq O\left(1/\sqrt{n} \right)$.
We have that with probability $1-\delta$, 
$$
    \left|\sum_{t=1}^T {\mu}_{\pi_t, t}(u_{\pi_t, 0:t-1})
    -  \widehat{{\mu}}_{\pi_t,t}(u_{\pi_t, 0:t-1}) \right|
    \leq \sum_{t=1}^T \left| {\mu}_{\pi_t, t}(u_{\pi_t, 0:t-1})
    -  \widehat{{\mu}}_{\pi_t, t}(u_{\pi_t, 0:t-1}) \right| 
    \leq \left(\frac{T}{\sqrt{n}}\right). 
$$
\end{proof}

\subsection{Proof of Proposition~\ref{proposition:mdp-gaurantee}}
\begin{proof}%
Fix $\delta \in (0,1)$. 
Let $E_1$ be the event that   
\begin{align*}
    \forall k \in [K], 
    ~|\widehat{\gamma}_k - \gamma_k| = \epsilon_{\gamma_k} \leq  O\left(\frac{1}{\sqrt{n}}\right), 
    ~|\widehat{\lambda}_k - \lambda_k| = \epsilon_{\lambda_k} \leq O\left(1/\sqrt{n} \right). 
\end{align*}
From Corollary~\ref{thm:estimation-single}, we have that 
 $\mathbb{P}(E_1) \geq 1-\delta/2$.
Let $\epsilon_\lambda := \max_k \epsilon_{\lambda_k}$. 
Let $E_2$ denote the event that $\forall t \in [T], k \in [K]$, $|x_{k,t}| \leq B(\delta/2)$~\eqref{eq:x_bound}. 
We know that $\mathbb{P}(E_2) \geq 1-\delta/2$. 
When $E_1$ and $E_2$ happen, we first observe that 
for all positive integer $n$ and $k \in [K]$, 
\begin{align*}
    |\widehat{\gamma}_k^n - \gamma_k^n| 
    \leq |\widehat{\gamma}_k - \gamma_k| 
    \left(n \max(\gamma_k^{n-1}, \widehat{\gamma}_k^{n-1}) \right)
    \leq 
    \frac{|\widehat{\gamma}_k - \gamma_k|}{
    \max(\gamma_k, \widehat{\gamma}_k)
    \ln \left({1}/\max(\gamma_k, \widehat{\gamma}_k)\right) } = O(1/\sqrt{n}),
\end{align*}
where%
and the second inequality uses the assumption that $\widehat{a}_k, \gamma_k$ are bounded away from $0$ and $1$.

To continue, we first bound the distance between the transition function in $\widehat{\mathcal{M}}$ 
and $\mathcal{M}$. 
At any any time $t$ and state $x_t = (x_{1,t}, n_{1,t}, \ldots, x_{K,t}, n_{K,t})$, when we pull arm $\pi_t = k$, the next state $x_{t+1}$ is updated by: 
(i) for arm $k$, $n_{k,t+1}=1$  and 
(ii) for all other arms $k' \neq k$, $n_{k,t+1} = n_{k,t} +1$ if $n_k \neq 0$, 
$n_{k,t+1}=0$ if $n_{k,t} =0$, 
and $x_{k',t+1} = x_{k',t}$.
Then, by~\cite[Theorem 1.3]{devroye2018total}, we have that when $n_{\pi_t,t} \neq 0$, 
\begin{align*}
\allowdisplaybreaks
    &\|p_{\widehat{\mathcal{M}}}\left(x_{t+1}|x_t, \pi_t, t\right) -  p_{\mathcal{M}} \left( x_{t+1}|x_t, \pi_t, t \right) \|_1 \\
    \stackrel{(*)}{\leq}& \frac{3 |\widehat{\lambda}_k^2 \sum_{i=0}^{n_{k,t} -1} \widehat{\gamma}_k^{2i}- \lambda_k^2\sum_{i=0}^{n_{k,t} -1} \gamma_k^{2i}|}{ \lambda_k^2\sum_{i=0}^{n_k -1} \gamma_k^{2i}} + \frac{|\gamma_k^{n_{k,t}}x_{k,t} + \lambda_k \gamma_k^{n_{k,t}} -\widehat{\gamma}_k^{n_{k,t}}x_{k,t}  - \widehat{\lambda}_k \widehat{\gamma}_k^{n_{k,t}} |}
    {\lambda_k \sqrt{\sum_{i=0}^{n_{k,t} -1} \gamma_k^{2i}}}\\
    = & 
    \frac{3 |\widehat{\lambda}_k^2 \left(\sum_{i=0}^{n_{k,t} -1} \widehat{\gamma}_k^{2i} -  \sum_{i=0}^{n_{k,t} -1} \gamma_k^{2i}\right)
    + (\widehat{\lambda}_k^2-\lambda_k^2) \sum_{i=0}^{n_{k,t} -1} \gamma_k^{2i}|}{ \lambda_k^2\sum_{i=0}^{n_k -1} \gamma_k^{2i}} \\
    &\qquad + 
    \frac{|\widehat{\gamma}_k^{n_{k,t}}
    - \gamma_k^{n_{k,t}}| B(\delta/2)
    + |\lambda_k \gamma_k^{n_{k,t}} -  \widehat{\lambda}_k \widehat{\gamma}_k^{n_{k,t}} |}
    {\lambda_k \sqrt{\sum_{i=0}^{n_{k,t} -1} \gamma_k^{2i}}} \\
    \stackrel{(**)}{\leq} & 
    3
    \left|\widehat{\lambda}_k^2 \left(\sum_{i=0}^{n_{k,t} -1} \widehat{\gamma}_k^{2i} -  \sum_{i=0}^{n_{k,t} -1} \gamma_k^{2i}\right)
    + (\widehat{\lambda}_k^2-\lambda_k^2) \sum_{i=0}^{n_{k,t} -1} \gamma_k^{2i} \right|
    + |\widehat{\gamma}_k^{n_{k,t}}
    - \gamma_k^{n_{k,t}}| \left(B(\delta/2) + \lambda_k \right) \\
    &\qquad + |\lambda_k \widehat{\gamma}_k^{n_{k,t}} -  \widehat{\lambda}_k \widehat{\gamma}_k^{n_{k,t}} |\\
    \leq& 
    3 
    \left( 
    (\lambda_k + \epsilon_\lambda)^2 
    \left|\frac{1}{1 - \widehat{\gamma}_k^2} - \frac{1}{1 - \gamma_k^2} \right|
    + \frac{|\widehat{\lambda}_k - \lambda_k| (2 \lambda_k + \epsilon_\lambda)}{1 - \gamma_k^2}
    \right)
    + |\widehat{\gamma}_k^{n_{k,t}}
    - \gamma_k^{n_{k,t}}| \left(B(\delta/2) + \lambda_k \right)  + |\lambda_k -  \widehat{\lambda}_k |\\
    =& 
    3 
    \left( 
    \frac{(\lambda_k + \epsilon_\lambda)^2|\widehat{\gamma}_k^2 - \gamma_k^2| }
    {\left(1 - \widehat{\gamma}^2_k \right) (1 - \gamma_k^2)} 
    + \frac{|\widehat{\lambda}_k - \lambda_k| (2 \lambda_k + \epsilon_\lambda)}{1 - \gamma_k^2}
    \right)  + |\widehat{\gamma}_k^{n_{k,t}}
    - \gamma_k^{n_{k,t}}| \left(B(\delta/2) + \lambda_k \right)
    + |\lambda_k -  \widehat{\lambda}_k |
    \\
    =:&\; \epsilon_P = O\left(\frac{1}{\sqrt{n}}\right), 
\end{align*}
where $(*)$ holds since $p_{\mathcal{M}} \left( x_{t+1}|x_t, \pi_t, t \right)$
is a Gaussian density with mean $\gamma_k^{n_{k,t}}x_{k,t} + \lambda_k \gamma_k^{n_{k,t}}$
and variance $\lambda_k^2\sum_{i=0}^{n_k -1} \gamma_k^{2i}$
and $(**)$ uses the fact that $\lambda_k^2 \sum_{i=0}^{n_k-1}\gamma_k^{2i} \geq \lambda_k^2 \geq 1$. 
When $n_{\pi_t,t} = 0$ and condition (i) and (ii) are fulfilled, 
we have that $\|p_{\widehat{\mathcal{M}}}\left(x_{t+1}|x_t, \pi_t, t\right) -  p_{\mathcal{M}} \left( x_{t+1}|x_t, \pi_t, t \right) \|_1 = 0$. 
Otherwise, that is, if condition (i) or (ii) is not satisfied, 
we also have that   
$\|p_{\widehat{\mathcal{M}}}\left(x_{t+1}|x_t, \pi_t, t\right) -  p_{\mathcal{M}} \left( x_{t+1}|x_t, \pi_t, t \right) \|_1 = 0$
since $p_{\widehat{\mathcal{M}}}\left(x_{t+1}|x_t, \pi_t, t\right) =  p_{\mathcal{M}} \left( x_{t+1}|x_t, \pi_t, t \right) = 0$. 
Next, we examine the difference of the expected reward obtained by pulling arm $k$ at state $x_t$ at time $t$ in MDP $\mathcal{M}$ and $\widehat{\mathcal{M}}$; when $n_{k,t} \neq 0$, this is given by
\begin{align*}
    \left|\widehat{{r}}(x_t, k)] - {r}(x_t, k)] \right|
    &= |\gamma_k^{n_{k,t}}x_{k,t} + \lambda_k \gamma_k^{n_{k,t}} -\widehat{\gamma}_k^{n_{k,t}}x_{k,t}  - \widehat{\lambda}_k \widehat{\gamma}_k^{n_{k,t}} | \\
    &\leq |x_{k,t}| \cdot |\gamma^{n_{k,t}} - \widehat{\gamma}^{n_{k,t}}| + 
    |\lambda_k \gamma_k^{n_{k,t}} - {\lambda}_k \widehat{\gamma}_k^{n_{k,t}} + {\lambda}_k \widehat{\gamma}_k^{n_{k,t}} - \widehat{\lambda}_k \widehat{\gamma}_k^{n_{k,t}} | \\
    &\leq \left(B(\delta/2) + \lambda_k \right) 
    |\widehat{\gamma}^{n_{k,t}}_k - {\gamma}^{n_{k,t}}_k| 
    + |\widehat{\lambda}_k - \lambda_k|
    =: \epsilon_R  = O\left(\frac{1}{\sqrt{n}}\right),
\end{align*}
where $\widehat{{r}}(x_t, k)$ is the expected reward of pulling arm $k$ at state $x_t$ in MDP $\widehat{\mathcal{M}}$.  
Putting it altogether, we have that for any deterministic policy $\pi$, 
\allowdisplaybreaks
\begin{align*}
    V_{1,\mathcal{M}}^\pi({x}_\text{init}) - V_{1, \widehat{\mathcal{M}}}^{\pi}({x}_\text{init}) 
    &={r}({x}_\text{init}, \pi_1({x}_\text{init}))
    -\widehat{{r}}({x}_\text{init}, \pi_1({x}_\text{init}))
    + \mathbb{E}_{x_{2} \sim p_{\mathcal{M}}(\cdot|x_1, \pi, 1)} [V_{2,\mathcal{M}}^{\pi}(x_2)] \\
    &\qquad - \mathbb{E}_{x_{2} \sim p_{\widehat{\mathcal{M}}}(\cdot|x_1, \pi, 1)} [V_{2,\widehat{\mathcal{M}}}^{\pi}(x_2)] \\
    &\leq \epsilon_R 
    +  \mathbb{E}_{x_{2} \sim p_{\mathcal{M}}(\cdot|x_1, \pi, 1)} [V_{2,\mathcal{M}}^{\pi}(x_2)]
    - \mathbb{E}_{x_{2} \sim p_{\widehat{\mathcal{M}}}(\cdot|x_1, \pi, 1)} [V_{2,\mathcal{M}}^{\pi}(x_2)] \\
    &\qquad + \mathbb{E}_{x_{2} \sim p_{\widehat{\mathcal{M}}}(\cdot|x_1, \pi, 1)} [V_{2,\mathcal{M}}^{\pi}(x_2)]
    - \mathbb{E}_{x_{2} \sim p_{\widehat{\mathcal{M}}}(\cdot|x_1, \pi, 1)} [V_{2,\widehat{\mathcal{M}}}^{\pi}(x_2)] \\
    &\leq T \epsilon_R +\sum_{t=1}^T
    \mathbb{E}_{\widehat{\mathcal{M}}, \pi} \Big\{ 
    \mathbb{E}_{x_{t+1} \sim p_{\mathcal{M}}(\cdot|x_t, \pi, t)} [V_{t+1,\mathcal{M}}^{\pi}(x_{t+1})]\\
    &\qquad - \mathbb{E}_{x_{t+1} \sim p_{\widehat{\mathcal{M}}}(\cdot|x_t, \pi, t)} [V_{t+1,\mathcal{M}}^{\pi}(x_{t+1})] \Big\} \\
    &\leq T \epsilon_R + T^2 \epsilon_P \max_k{{b}_k},
\end{align*}
where 
$p_{\mathcal{M}}(\cdot|x_t, \pi, t)$ denotes 
$p(\cdot|x_t, \pi_t(x_t), t)$ in MDP ${\mathcal{M}}$ 
and 
the last inequality uses the fact that 
$\langle p_{\mathcal{M}}(\cdot|x_t, \pi, t) - p_{\widehat{\mathcal{M}}}(\cdot|x_t, \pi, t),~ V_{t+1, \mathcal{M}}^\pi \rangle \leq \|p_{\mathcal{M}}(\cdot|x_t, \pi, t) - p_{\widehat{\mathcal{M}}}(\cdot|x_t, \pi, t) \|_1 \|V_{t+1, \mathcal{M}}^\pi\|_{\infty} \leq \epsilon_P T \max_k {b}_k$. 
Finally, we have that 
\begin{align*}
    V_{1,\mathcal{M}}^{*}({x}_\text{init}) - V_{1,\mathcal{M}}^{\pi_{\widehat{\mathcal{M}}}^*}({x}_\text{init}) &=  V_{1,\mathcal{M}}^{\pi_{\mathcal{M}}^*}({x}_\text{init}) - V_{1,\widehat{\mathcal{M}}}^{\pi_{\mathcal{M}}^*}({x}_\text{init})  
    + V_{1,\widehat{\mathcal{M}}}^{\pi_{\mathcal{M}}^*}({x}_\text{init})
    - V_{1,\widehat{\mathcal{M}}}^{\pi_{\widehat{\mathcal{M}}}^*}({x}_\text{init}) \\
    &\qquad + V_{1,\widehat{\mathcal{M}}}^{\pi_{\widehat{\mathcal{M}}}^*}({x}_\text{init})
    - V_{1,\mathcal{M}}^{\pi_{\widehat{\mathcal{M}}}^*}({x}_\text{init}) 
    \leq 2T \epsilon_R + 2T^2 \epsilon_P \max_k {b}_k, %
\end{align*} 
where the equation follows from the fact that $V_{1,\mathcal{M}}^{*}({x}_\text{init})=V_{1,\mathcal{M}}^{\pi_{\mathcal{M}}^*}({x}_\text{init})$ and rearranging the terms, and the inequality follows from applying the bound of $V_{1,\mathcal{M}}^{\pi}({x}_\text{init}) - V_{1,\widehat{\mathcal{M}}}^{\pi}({x}_\text{init})\leq T \epsilon_R + T^2 \epsilon_P \max_k{{b}_k}$ that was derived above for $\pi=\pi_{\mathcal{M}}^*$ and $\pi=\pi_{\widehat{\mathcal{M}}}^*$ and using the fact that the policy  $\pi_{\widehat{\mathcal{M}}}^*$ is optimal for MDP $\widehat{\mathcal{M}}$. 
Let $E_3$ denote the event that $V_{1,\mathcal{M}}^{*}({x}_\text{init}) - V_{1,\mathcal{M}}^{\pi_{\widehat{\mathcal{M}}}^*}({x}_\text{init})  \leq O(T^2/\sqrt{n})$. 
Putting it altogether, we have that 
$\mathbb{P}(E_3) \geq %
\mathbb{P}(E_2, E_1) = 1 - \mathbb{P}(E_2^c \cup E_1^c) \geq 1 - \delta$. 
\end{proof}

\newpage
\section{Additional Experimental Details and Results}
\label{appendix:experiment}
In this appendix, we present additional experimental details and results. 

\paragraph{$w$-lookahead Performance}
When evaluating the performance of $w$-lookahead policies, 
in addition to the case where $T=30$ (Figure~\ref{fig:T_30_lookahead}), 
we have also run the experiments with $T=100$ (Figure~\ref{fig:eep_T_100}).
When solving for the $100$-lookahead policy,
we have increased the number of threads to $50$ 
to solve for~\eqref{eq:obj} and 
stopped the program at a time limit of $24$ hours. 
In such settings, 
we obtain an upper bound on the absolute optimality gap of $64.0$ (percentage optimality gap of $13.0\%$). 
When solved for $w$-lookahead policies with $w$ in between $1$ and $15$  
using $10$ threads, 
Gurobi ends up solving~\eqref{eq:lookahead} within $40$s 
for all different $w$ values. 
Thus, 
despite using significantly lower computational time, 
$w$-lookahead policies achieve a similar cumulative expected reward 
to the $T$-lookahead policies (see Figures~\ref{fig:T_30_lookahead} and~\ref{fig:eep_T_100}).

\begin{figure*}
    \centering
      \begin{subfigure}[b]{0.45\linewidth}
        \centering
        \includegraphics[width=\linewidth]{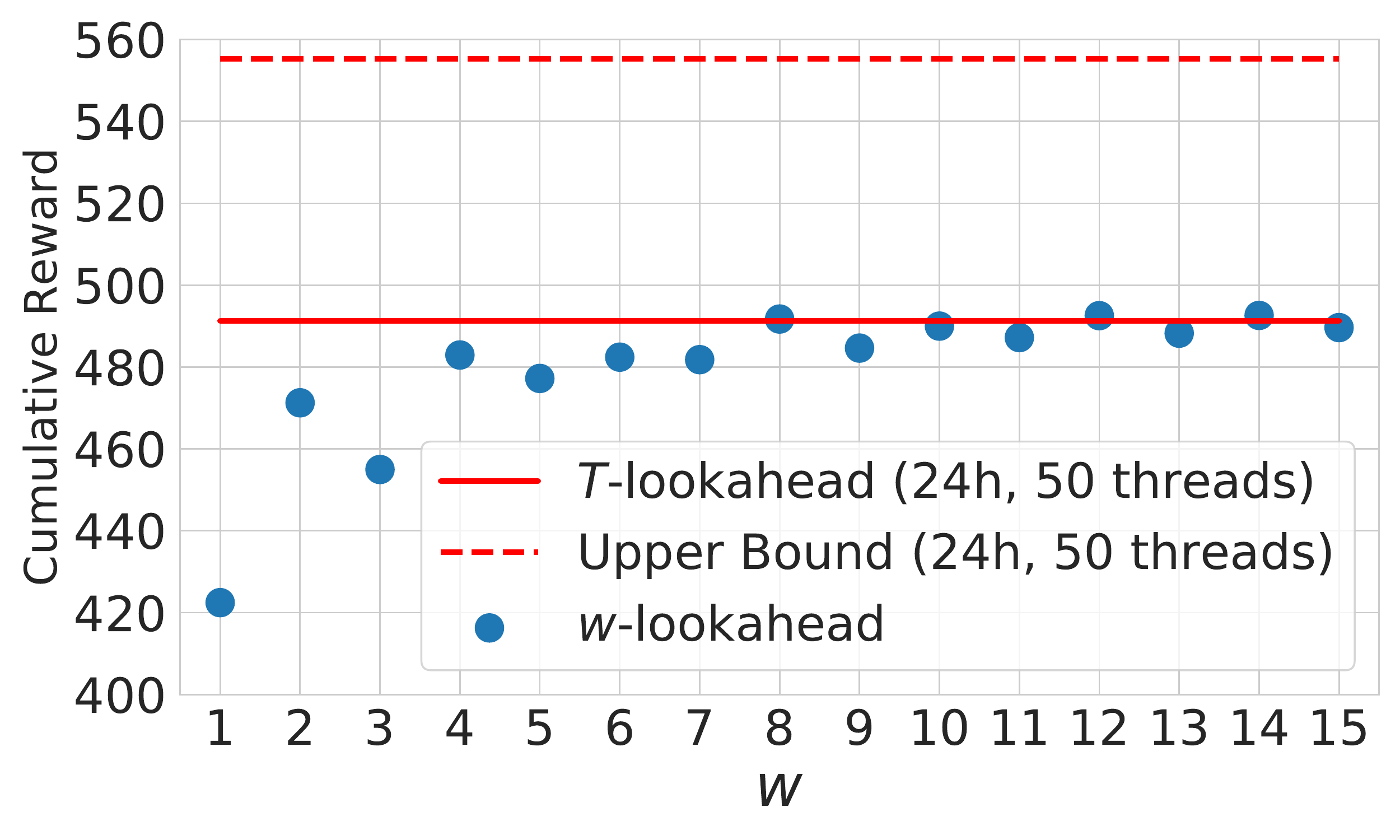}
        \caption{$T=100$}
        \label{fig:eep_T_100}
    \end{subfigure}
    \hfill
    \begin{subfigure}[b]{0.45\linewidth}
        \centering
        \includegraphics[width=\linewidth]{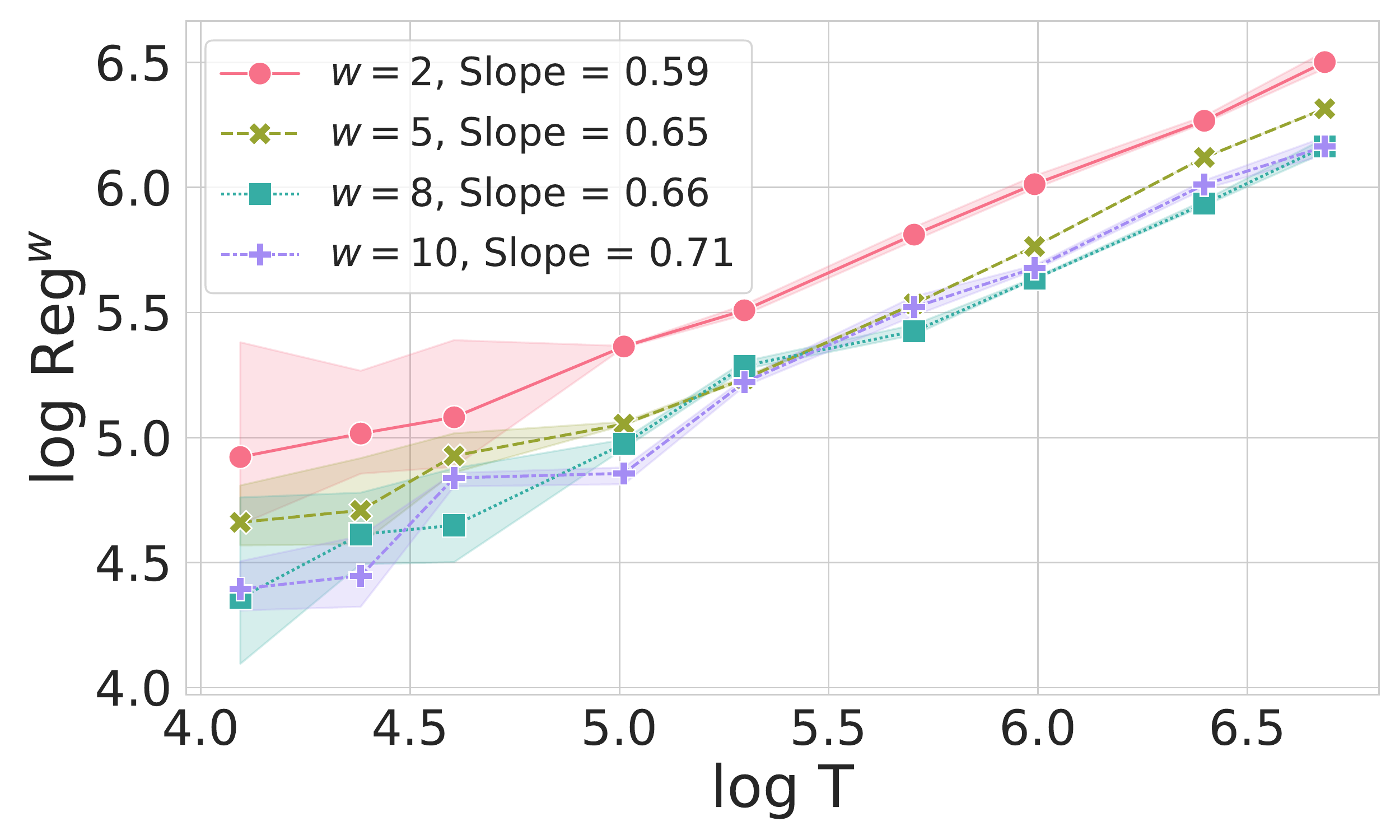}
        \caption{$\log \text{Reg}^w$ v.s. $\log T$}
        \label{fig:eep_800}
    \end{subfigure}
       \caption{ Figure
       \ref{fig:eep_T_100} shows the cumulative expected  reward 
       collected by 
       and $w$-lookahead policy (blue dots) when $T=100$. 
       When solving for the $T$-lookahead policy 
       (\eqref{eq:obj} with $T=100$),
       after $24$ hours,
       Gurobi 9.1 obtains an objective value of $491.3$
       (red solid line)
       with an upper bound $555.3$ (red dotted line)
       and  
       an absolute optimality gap $64.0$ ($13.0\%$). 
       The true cumulative expected reward for $T$-lookahead policy for this problem 
       lies in between 
       the solid and dotted red lines. 
       Figure~\ref{fig:eep_800} 
       shows the log-log plot of the 
       $w$-step lookahead regret 
       of $w$-lookahead EEP 
       (averaged over $5$ random runs)
       under different $T$.
       }
       \label{fig:eep_appendix}
\end{figure*}

\paragraph{EEP Performance}
Figure~\ref{fig:eep_performance}
is the log-log plot of the $w$-step lookahead regret of $w$-lookahead
EEP against the horizon $T$ 
when $T= 60, 80, 100, 150, 200, 300, 400$
(averaged over $20$ random runs)
and Figure~\ref{fig:eep_800} is the log-log plot 
when $T= 60, 80, 100, 150, 200, 300, 400, 600, 800$
(averaged over $5$ random runs),
under the experimental setup provided in \S~\ref{sec:experiment}.

\end{document}